\documentclass{article} 
\usepackage{iclr2027_conference,times}

\usepackage{amsmath,amsfonts,bm}

\def\eqref#1{equation~\ref{#1}}

\def\1{\bm{1}}

\DeclareMathAlphabet{\mathsfit}{\encodingdefault}{\sfdefault}{m}{sl}
\SetMathAlphabet{\mathsfit}{bold}{\encodingdefault}{\sfdefault}{bx}{n}

\usepackage{hyperref}
\usepackage{url}

\title{On Calibration of Large Language Models: \\ From Response To Capability}

\author{
\textbf{Sin-Han Yang}\textsuperscript{1}\thanks{Equal contribution; junior author listed earlier.}
\quad
\textbf{Cheng-Kuang Wu}\textsuperscript{1}\footnotemark[1]
\quad
\textbf{Chieh-Yen Lin}\textsuperscript{1}
\quad
\textbf{Yun-Nung Chen}\textsuperscript{2}
\\[0.6ex]
\hspace{0.1pt}
\textbf{Hung-yi Lee}\textsuperscript{2}
\quad
\textbf{Shao-Hua Sun}\textsuperscript{1,2}
\\[1.0ex]
\textsuperscript{1}Appier AI Research
\quad
\textsuperscript{2}National Taiwan University
\\
\texttt{\{sinhan.yang,brian.wu\}@appier.com}
}

\usepackage[utf8]{inputenc} 
\usepackage[T1]{fontenc}    
\usepackage{hyperref}       
\usepackage{url}            
\usepackage{booktabs}       
\usepackage{amsfonts}       
\usepackage{nicefrac}       
\usepackage{microtype}      

\usepackage{tabularx}
\usepackage[table]{xcolor}
\newcolumntype{Y}{>{\centering\arraybackslash}X}
\usepackage{makecell}
\usepackage{multirow}
\usepackage{wrapfig}
\usepackage{caption}
\usepackage{enumitem}

\usepackage{etoc}

\usepackage[nottoc]{tocbibind}
\settocbibname{References}
\usepackage{minitoc}
\usepackage{bbm}

\definecolor{LOrange}{HTML}{E69F00}
\definecolor{LBlue}{HTML}{56B4E9}
\definecolor{LGreen}{HTML}{009E73}

\usepackage[most]{tcolorbox}
\tcbset{  
  aibox/.style={
    width=\textwidth,
    top=10pt,
    colback=LBlue!6!white,
    colframe=black,
    colbacktitle=black,
    enhanced,
    center,
    attach boxed title to top left={yshift=-0.1in,xshift=0.15in},
    boxed title style={boxrule=0pt,colframe=white,},
  }
}
\newtcolorbox{AIbox}[2][]{aibox,title=#2,#1}

\usepackage{graphicx}
\usepackage{lipsum}
\usepackage{multirow}
\usepackage{subcaption}

\usepackage{amsmath}
\usepackage{amssymb}
\usepackage{mathtools}
\usepackage{amsthm}

\usepackage[capitalize,noabbrev]{cleveref}

\theoremstyle{plain}
\newtheorem{theorem}{Theorem}

\theoremstyle{definition}

\theoremstyle{remark}

\iclrfinalcopy 
\begin{document}

\doparttoc
\faketableofcontents

\maketitle

\begin{abstract}
Accurate confidence estimation is critical for reliable use of large language models (LLMs).
Prior work on LLM calibration largely focuses on response-level confidence, which estimates the correctness of a single generated output. However, this formulation is misaligned with many practical settings where the central question is how likely a model is to solve a query overall.
We show that this mismatch results from the stochastic nature of modern LLM decoding, under which single-response correctness fails to reflect underlying model capability.
To address this issue, we introduce \textbf{capability calibration, a new evaluation framework} for measuring how well query-level confidence aligns with a model’s expected accuracy on individual queries.
We formally distinguish capability calibration (CC) from response calibration (RC) and show that the two differ both theoretically and empirically.
We further show that CC is better suited than RC to applications like pass@$k$ prediction and inference budget allocation.
Finally, we evaluate common confidence estimation methods to understand the practical feasibility of CC. Source code: \url{https://github.com/appier-research/llm-calibration}.

\end{abstract}

\section{Introduction}

Large language models (LLMs) have reshaped human-AI interaction by enabling users to pose queries in natural language and receive informative responses~\cite{ouyang2022training}.
This intuitive interface has driven their rapid adoption across a wide range of applications.
However, despite their fluency, LLMs produce incorrect or misleading outputs without explicitly signaling uncertainty.
This limitation makes accurate confidence estimation a critical component of reliable LLM deployment. 
Well-calibrated confidence scores enable users to better judge when to trust LLMs~\citep{huang2024trustllm,aljohani2025comprehensive}, allow systems to selectively refuse or defer to human experts~\citep{wu2024need}, and support performance prediction for downstream tasks.

Given the important role of confidence estimation, a natural question is how to evaluate its quality.
Calibration \citep{niculescu2005predicting,guo2017calibration} provides a principled framework by evaluating how well estimated confidence aligns with correctness probability.
Existing work on LLM calibration \citep{geng2024survey} adopts a response-level formulation: given a query $x$ and a generated response $\hat{y}$, a confidence estimator produces a score $s$ as the probability that $\hat{y}$ is correct with respect to $x$.
Under this formulation, calibration is evaluated independently for each generated response.
We refer to this setting as response calibration (\cref{fig: fig1}a).

\begin{figure}[t!]
  \centering
  \includegraphics[width=\linewidth]{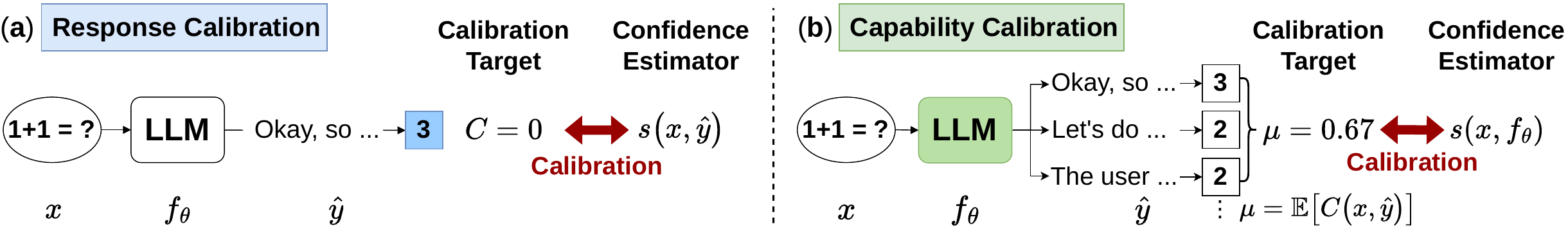}
    \caption{
      \textbf{(a) Response calibration.}
      Given an input $x$, a model $f_{\theta}$, and its \textit{single} sampled output $\hat{y}$, response calibration calibrates confidence $s(x,\hat{y})$ against correctness $\mathcal{C}$ of $\hat{y}$. 
      \textbf{(b) Capability calibration.}
      We propose capability calibration, which calibrates confidence $s(x,f_{\theta})$ against expected accuracy $\mu$ of $f_{\theta}$'s output distribution. 
      While response calibration is limited by LLM stochastic decoding, capability calibration directly targets expected accuracy, providing a more reliable and efficient metric for tasks like forecasting test-time scaling performance and resource allocation.
    }
    \label{fig: fig1}
\end{figure}

In many practical settings, however, what matters is not whether a particular response $\hat{y}$ is correct, but how likely the LLM is to solve a given query overall.
This question naturally arises in applications such as allocating computational resources across queries \citep{chen2023frugalgpt,ong2024routellm} or predicting model performance in downstream pipelines.
We refer to this quantity, \textit{``how likely can the LLM answer this query correctly?''}, as query-level confidence.
While response-level confidence and query-level confidence are equivalent for deterministic models (e.g., standard image classification where prediction is fixed), the two are fundamentally misaligned in LLMs due to the stochastic nature of text generation.
Modern LLMs typically achieve better performance with stochastic decoding~\citep{holtzman2019curious,shi2024thorough}, such as non-zero temperature sampling \citep{renze2024effect,li2025exploring}, which can produce different responses given the same query across inference calls.
As a result, the correctness of any single sampled response cannot accurately reflect the LLM's underlying capability on that query.
This mismatch between single-response correctness and query-level performance lies beyond what response calibration can capture.

Motivated by these observations, we introduce capability calibration (\cref{fig: fig1}b), which targets query-level expected accuracy: the probability that a response sampled from the model’s output distribution conditioned on $x$ would be correct.
This formulation shifts the focus from whether a particular sampled response happens to be correct to the model's expected capability of solving the query.
We formally distinguish capability calibration from response calibration and show that they differ both theoretically (\S\ref{sec:rc-cc-diff}) and empirically (\S\ref{sec: exp evaluation}).
Notably, capability calibration is not merely the expectation of response calibration; the two quantities differ precisely by the variance of response correctness under the model’s output distribution.
Their key differences are summarized in~\cref{tab:diff_comparison}.

We show that capability calibration is better suited than response calibration to practical applications that require query-level estimates of model capability (\S\ref{sec:applications}).
We study two such applications:
(1) pass@$k$ estimation \citep{schaeffer2025large,kazdan2025efficient}, where confidence estimates are used to predict the pass@$k$ success rate of individual queries without costly repeated sampling during inference, and
(2) inference budget allocation \citep{snell2024scaling,damani2024learning}, where confidence estimates guide the allocation of computational resources across queries.
Empirically, capability-calibrated confidence yields stronger oracle performance than response-calibrated confidence, and practical capability-calibrated estimators consistently outperform standard baselines.
We also discuss other applications where capability calibration may be beneficial.

After demonstrating the practical value of capability calibration, we study how to evaluate and estimate it.
We benchmark commonly used training-free and training-based confidence estimation methods after adapting each to the capability calibration objective.
Experimental results indicate that training linear probes on LLM activations offers a favorable tradeoff between online computational cost and confidence estimation performance (§\ref{sec:results}). We also conduct extensive experiments using our evaluation framework to stress-test existing confidence estimators, motivating future work to develop stronger capability-calibrated confidence estimators.
By formally defining capability calibration, demonstrating its practical utility, and establishing it as \textbf{an evaluation framework}, our work offers a new perspective on LLM calibration that directly captures query-level model capability.

\section{Related works}

\subsection{LLM confidence estimation and calibration}
Confidence estimation focuses on estimating the probability that predictions are correct.
Previous works \citep{niculescu2005predicting,guo2017calibration}  define calibration as the agreement between confidence and the correctness of an output. We call this definition  \textbf{response calibration}. 
Denote $x$ as the input, $\hat{y}$ as the model's output, estimated confidence as $s(x,\hat{y})$, and
$\mathcal{C}(x,\hat{y}) \in \{0,1\}$ as the correctness function. Formally, 
$\mathcal{C}(x,\hat{y})
= \mathbf{1}\!\left[ \hat{y} \text{ is correct for } x \right].$ Perfect response calibration is defined as $
    \mathbb{P} \left[ \mathcal{C}(x,\hat{y}) = 1\mid s(x,\hat{y}) = p \right] = p.$

A commonly used evaluation metric is Brier score \citep{glenn1950verification}. 
Let $\mathcal{D}=\{x_i\}_{i=1}^N$ be a dataset of inputs, $\mathbf{s} = [s_1,s_2,...s_i,...s_N]$ be the estimated confidence for each query, and $\hat{y}_i$ be the sampled response for input $x_i$. The Brier score used in response calibration is
\begin{equation} \label{equ: classical brier score}
    \mathcal{L}_{\mathrm{Brier}}^{\mathrm{response}}(\mathbf{s})\triangleq
\frac{1}{N}\sum_{i=1}^N \left(s_i - \mathcal{C}(x_i,\hat{y}_i) \right)^2.
\end{equation}

To achieve response calibration for LLMs, training-free~\citep{lin2022teaching,tian2023just,kadavath2022language, manakul2023selfcheckgpt} and training-based~\citep{zhang2025reasoning,damani2025beyond,wu2025mitigating,li2025conftuner} confidence estimation methods have been proposed.
These methods typically operate post-hoc, estimating confidence only after the output is generated.
In contrast, a body of work focuses on anticipating the performance of a single response, aiming to estimate the response's correctness before it is generated \citep{zhou2022reject, cencerrado2025no,schellaert2025analysing}.
Detailed descriptions of these methods are in Appendix \ref{sec:detailed_related_works}.

Despite the methodological differences in estimating confidence, their calibration target remains the same: they aim to estimate the correctness probability of a single response.
This target is natural for deterministic models, such as standard image classifiers, where an \texttt{argmax} prediction is produced for each input.
As a result, the correctness of a single prediction coincides with the model’s expected accuracy, making response calibration and capability calibration equivalent.
However, LLMs mark a shift from deterministic prediction to stochastic generation, where the same query yields different responses, some correct and some incorrect.
Consequently, the correctness of any one sampled response is only a noisy realization of the model’s underlying ability to solve the query.
This shift makes response calibration insufficient for characterizing query-level model capability, and motivates \textbf{capability calibration} as \textbf{a new evaluation framework} for calibration of stochastic models.

\begin{table}[t!]
    \centering 
    \caption{\textbf{Comparison of calibration definitions.} Unlike existing response calibration that assesses whether the confidence estimate $s$ aligns with the correctness of one decoded answer $\hat{y}$, our proposed capability calibration evaluates whether $s$ aligns with the model $f_{\theta}$'s capability to answer query $x$.}
    \label{tab:diff_comparison}
    \begin{tabularx}{\textwidth}{@{} l >{\hsize=0.85\hsize}X >{\hsize=0.85\hsize}X >{\hsize=1.3\hsize}X @{}}
    \toprule
    \textbf{Definition} & 
    \textbf{Calibration Target} & 
    \textbf{Interpretation} & 
    \textbf{Dependence on LM $f_{\theta}$} \\
    \midrule
    
    \makecell[tl]{\textbf{Response} \\ \textbf{Calibration}} & 
    Accuracy of $\hat{y}$ given $x$ & 
    How likely is $\hat{y}$ correct? & 
    \textbf{No.} Since $\hat{y}$ is already decoded, the estimation of $s(x, \hat{y})$ is decoupled from the generating model $f_{\theta}$. \\
  \addlinespace[7pt] 
    
    \makecell[tl]{\textbf{Capability} \\ \textbf{Calibration}} & 
    Expected accuracy of $f_{\theta}$ given $x$ & 
    How likely is $f_{\theta}$ to answer $x$ correctly? & 
    \textbf{Yes.} The expected accuracy for $x$ depends on $f_{\theta}$'s capability. \\
    
    \bottomrule
\end{tabularx}
\end{table}

\subsection{LLM uncertainty quantification}

LLM uncertainty quantification (UQ) studies methods for estimating how uncertain a model is on an input. 
Unlike calibration, which evaluates how well confidence scores align with correctness probabilities, UQ is often evaluated by the usefulness of uncertainty scores in downstream decision-making \citep{huang2024survey}, such as distinguishing correct from incorrect predictions or deciding when to abstain \citep{kamath2020selective,chen2023frugalgpt,kang2025uncertainty}. 
Therefore, common evaluation metrics are AUROC \citep{hendrycks2016baseline} and the Risk-Coverage curve \citep{geifman2017selective}.
Details of common LLM UQ methods are in Appendix \ref{sec:detailed_related_works}.


Capability calibration and LLM UQ are related but have distinct goals: while UQ produces uncertainty scores,
capability calibration evaluates if these score represent true probabilities of query-level success.
Consequently, capability-calibrated scores can make UQ outputs more interpretable for applications.

\section{Capability calibration} \label{sec: section 3}
Capability calibration measures whether the estimated confidence on a query reflects the expected accuracy of responses sampled from an output distribution.
We formalize this framework below.


\subsection{Definition} \label{sec: cc definition}
For a given input $x$, we define the model’s \textbf{expected accuracy} as

\begin{equation} \label{equ:new confidence definition}
    \mu(x, f_{\theta})
\triangleq
\mathbb{P}_{\hat{y} \sim f_\theta(\cdot \mid x)}[\mathcal{C}(x,\hat{y})=1]
=
\mathbb{E}_{\hat{y} \sim f_\theta(\cdot \mid x)}[\mathcal{C}(x,\hat{y})],
\end{equation}

which is equivalent to $\mu(x, f_{\theta})
=
\lim_{N\to\infty}
\frac{1}{N}\sum_{i=1}^N \mathcal{C}(x,\hat{y}_i),
\quad \hat{y}_i \sim f_\theta(\cdot \mid x).$

\cref{equ:new confidence definition} defines the target of the capability calibration. As illustrated in \cref{fig: fig1}, the target expected accuracy $\mu(x, f_{\theta})$ is defined as the frequency of correct outputs when the language model is sampled infinitely many times on the same input $x$.
Notably, the concept of "expected accuracy" exists in different applications \citep{kazdan2025efficient,damani2024learning}. However, they do not formalize it as confidence calibration target, nor use confidence estimators to estimate the quantity.

\begin{wrapfigure}{r}{0.5\textwidth}
\vspace{-12pt}
\centering
\includegraphics[width=0.5\textwidth]{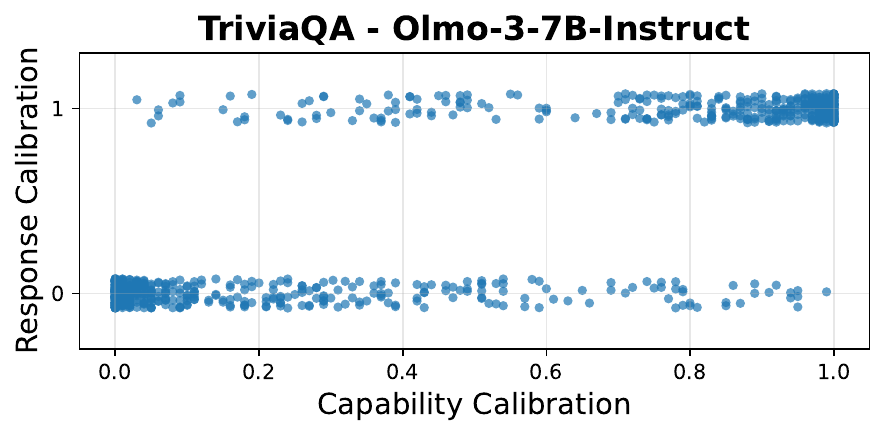}
    \caption{
      \textbf{Divergence of calibration targets.} We estimate the Response Calibration (RC) targets $\mathcal{C}(x, \hat{y})$ and Capability Calibration (CC) targets $\mathbb{E}_{\hat{y} \sim f_\theta(\cdot \mid x)}[\mathcal{C}(x,\hat{y})]$. Since response stochasticity does not guarantee correctness stochasticity, the experiment result reveals divergence between the two targets: queries where the RC label is 0 exhibit CC values spanning the full [0, 1] range. The same observation for queries with an RC label of 1. This confirms that response-level outcomes do not reflect the model's true ability to answer a query. 
    }
    \label{fig: metric_diff main}
\vspace{-50pt}
\end{wrapfigure}

In capability calibration, an estimated confidence $s$ is well calibrated if it is aligned with the expected accuracy $\mu(x, f_{\theta})$. The perfect calibration for capability calibration is $s^* =  \mu(x, f_{\theta})$.

We adopt Brier score \citep{glenn1950verification}, one of the most common metrics used to evaluate query-level calibration, as the primary metric.
Specifically, let $\mu_i$ be the expected accuracy $\mu(x_i,f_\theta)$, the capability calibration Brier score is 
\begin{equation} \label{equ: new brier score}
    \mathcal{L}_{\mathrm{Brier}}^{\mathrm{capability}}(\mathbf{s})
  \triangleq  \frac{1}{N}\sum_{i=1}^N (s_i - \mu_i)^2.
\end{equation}

This metric directly evaluates capability calibration by measuring how closely the estimated confidence matches the expected accuracy over a model's output distribution.
In contrast, response calibration evaluates confidence against the correctness of a single sampled response.
We further discuss their distinction below.

\subsection{Difference between response calibration and capability calibration}
\label{sec:rc-cc-diff}
We argue that these two evaluations diverge in three key aspects:

\begin{itemize}[leftmargin=0.5cm]
    \item \textbf{Evaluation targets}: Response calibration targets correctness of a model's output $\mathcal{C}(x,\hat{y})$, whereas capability calibration targets the expected accuracy over all possible outputs $\mathbb{E}_{\hat{y} \sim f_\theta(\cdot \mid x)}[\mathcal{C}(x,\hat{y})]$.  
    \item  \textbf{Conditional sets}: Response calibration is conditioned on both input and output $(x, \hat{y})$, estimating $\mathbb{P}(\text{Correct} \mid x, \hat{y})$. Capability calibration is conditioned on the input and the model's sampling distribution $(x, f_{\theta})$, estimating $\mathbb{P}(\text{Correct} \mid x, f_{\theta})$.
    \item \textbf{Optimal confidence}: Consequently, their optimal confidence scores differ, as proved in Appendix \ref{sec: proof of metrics differ}. Optimal response-calibrated confidence is a binary indicator of a response's correctness, while optimal capability-calibrated confidence is a continuous probability representing capability. This theoretical divergence is empirically confirmed in \cref{fig: metric_diff main} and \cref{sec: exp evaluation}.
\end{itemize}

Having established that these objectives are distinct, we now formalize the connection between the response calibration loss function and the capability calibration loss function:

\begin{theorem} \label{the: metrics connection}
\textbf{(Decomposition of calibration losses).} Given a set of estimated confidence $\mathbf{s}$, define the expectation of response calibration loss $\mathcal{L}^{\mathrm{response}}_{\mathrm{Brier}}$ on the model's output distribution as 
\[
\mathbb{E}\big[\mathcal{L}^{\mathrm{response}}_{\mathrm{Brier}}\big]
\;\triangleq\;
\frac{1}{N}\sum_{i=1}^N
\mathbb{E}_{\hat{y} \sim f_\theta(\cdot \mid x_i)}
\left[
\big(s_i - \mathcal{C}(x_i,\hat{y})\big)^2
\right].
\]

Decoupling output correctness variance from response calibration, we get capability calibration:
\begin{equation} \label{equ: metrics connection}
    \mathcal{L}_{\mathrm{Brier}}^{\mathrm{capability}} = 
    \underbrace{\mathbb{E}[\mathcal{L}^{\mathrm{response}}_{\mathrm{Brier}}] \vphantom{\sum^N_{i=1}}}_{\text{response calibration}} - 
    \underbrace{\frac{1}{N}\sum^N_{i=1}Var(\mathcal{C}(x_i,\hat{y}))}_{\text{output correctness variance}},
\end{equation}
where
\begin{equation} \label{equ: correctness variance}
    Var(\mathcal{C}(x_i,\hat{y})) =  \mathbb{E}_{\hat{y} \sim f_\theta(\cdot \mid x_i)}[\mathcal{C}(x_i,\hat{y})^2]  
     - \mathbb{E}_{\hat{y} \sim f_\theta(\cdot \mid x_i)}[\mathcal{C}(x_i,\hat{y})]^2.
\end{equation}
\end{theorem}

See Appendix \ref{sec: proof of metrics connection} for the proof. \cref{the: metrics connection} demonstrates that when evaluating a set confidence estimation, capability calibration decouples the \textit{model's output variance} from response calibration. While response calibration penalizes stochasticity of generated outputs, capability calibration targets the model's correctness probability. For strictly convex and differentiable losses, the difference $\mathbb{E}[\mathcal{L}^{response}]-\mathcal{L}^{capability}$ generalizes to Bregman information \citep{banerjee2005clustering}, which quantifies the gap \citep{ gruber2022uncertainty} caused by output randomness (details in Appendix \ref{sec: bregman info}).

\section{Applications of capability calibration}
\label{sec:applications}

We demonstrate that capability calibration benefits several downstream applications where response calibration is not optimal, either because output variance degrades its performance or because it requires full decoding before confidence can be estimated, creating an efficiency bottleneck.

\subsection{Pass@\texorpdfstring{$k$}{k} estimation} \label{sec: passk simulation}

Test-time scaling via repeated sampling has been shown to enhance LLM capabilities \citep{brown2024large}, while simultaneously increasing computational burden \citep{schaeffer2025large, kazdan2025efficient}. Given these trade-offs, the ability to estimate resampling performance at a low inference cost is critical \citep{stroebl2024inference}. A common approach to this problem, as proposed by \citet{kazdan2025efficient}, is to predict pass@$k$ performance by sampling only a small subset of outputs. Their method assumes that the expected accuracy of each query in a dataset follows a beta distribution.

We show that capability-calibrated confidence estimation can simulate the pass@$k$ performance of each query without (1) sampling multiple outputs and (2) assuming a prior distribution over the dataset. Furthermore, by computing the pass@$k$ success rate for each query, we can estimate the pass@$k$ curve for the entire dataset. We discuss the simulation process in Appendix \ref{sec: pass@k sim process}.

We evaluate three confidence estimators: (1) Oracle Response-Calibrated (Oracle-RC), where confidence is the RC ground truth, (2) Oracle Capability-Calibrated (Oracle-CC), where confidence is the CC ground truth, and (3) Probe \citep{li2021implicit}, a common confidence estimator trained with CC target, which we will discuss in \cref{sec: conf estimators}.
We compute the real pass@$k$ via the unbiased estimator of \citet{chen2021evaluating} and use MSE to measure the query-level discrepancy between simulated and actual pass@$k$.
\cref{tab:pass_at_k_math} reports results on MATH-500; additional results are in Appendix \ref{sec: more passk simulation}.

The results confirm that CC is a more suitable target than RC in pass@$k$ estimation.
Oracle-CC is expected accuracy (see \cref{equ:new confidence definition}), so it simulates ground-truth performance perfectly.
In contrast, Oracle-RC focuses on single-response correctness, which is a noisy estimate of expected accuracy, causing MSE to increase at higher $k$. 
This underscores the fundamental limitations of response calibration: even a perfectly calibrated response-level estimator cannot capture pass@$k$ performance.
Notably, \textbf{non-oracle, CC target trained Probes could even outperform Oracle-RC} in most cases.

\subsection{Inference budget allocation} \label{sec: budget alloc}

Inference-time scaling has been shown to improve language model performance \citep{zhang2024scaling,snell2024scaling}.
In real-world scenarios, however, LLM providers often face budget allocation problems \citep{ong2024routellm}: maximizing total accuracy across a stream of queries while constrained by limited GPU hours or API costs. \citet{damani2024learning} turns this challenge into the proxy:
\[
\textit{In the best-of-$k$ setting, how to solve as many problems as possible under a fixed sampling budget}? 
\]
To tackle this problem, \citet{damani2024learning} proposed an approach that involves distributing the total computational budget across a dataset of queries prior to generating answers.
This strategy outperforms uniform allocation by learning a reward model to estimate the marginal ``gain'': the expected improvement in success rate per unit of additional compute. By prioritizing queries where this gain is highest, the system can strategically distribute resources to maximize overall performance.

\begin{figure}[t] 
    \centering
    
    \begin{minipage}[t]{0.49\textwidth}
    \vspace{0pt} 
    \centering
    \captionof{table}{\textbf{Pass@$k$ simulation error (MSE) on the MATH-500 dataset.} We evaluate the ability of different confidence estimations to simulate empirical pass@$k$ performance. Perfectly capability-calibrated confidence (Oracle-CC) achieves near-perfect simulation, whereas the error of perfectly response-calibrated confidence (Oracle-RC) increases as $k$ scales. Notably, the trained estimator (Probe) outperforms the Oracle-RC baseline by approximating the model's expected accuracy.}
    \label{tab:pass_at_k_math}
    \resizebox{\linewidth}{!}{
        \begin{tabular}{p{2.8cm}cccc}
        \toprule
        \textbf{Method} & \textbf{pass@1} & \textbf{pass@4} & \textbf{pass@16} & \textbf{pass@64} \\ 
        \midrule
        \multicolumn{5}{l}{\textbf{Olmo-3-7B-Instruct}} \\
        \hspace{3mm}Oracle-CC & 0.0000 & 0.0000 & 0.0000 & 0.0002 \\
        \hspace{3mm}Oracle-RC & 0.0349 & 0.0497 & 0.0661 & 0.0804 \\
        \hspace{3mm}Probe & 0.0373 & 0.0336 & 0.0219 & 0.0148 \\ 
        \midrule
        \multicolumn{5}{l}{\textbf{Qwen3-8B}} \\
        \hspace{3mm}Oracle-CC & 0.0000 & 0.0000 & 0.0000 & 0.0005 \\
        \hspace{3mm}Oracle-RC & 0.0542 & 0.0884 & 0.1251 & 0.1536 \\
        \hspace{3mm}Probe & 0.0471 & 0.0417 & 0.0265 & 0.0134 \\ 
        \midrule
        \multicolumn{5}{l}{\textbf{gpt-oss-20b}} \\
        \hspace{3mm}Oracle-CC & 0.0000 & 0.0000 & 0.0000 & 0.0002 \\
        \hspace{3mm}Oracle-RC & 0.0243 & 0.0337 & 0.0400 & 0.0440 \\
        \hspace{3mm}Probe & 0.0216 & 0.0102 & 0.0068 & 0.0028  \\ 
        \bottomrule
        \end{tabular}
    }
\end{minipage}%
    \hfill 
    \begin{minipage}[t]{0.49\textwidth}
        \vspace{0pt} 
        \centering
        \includegraphics[width=\linewidth]{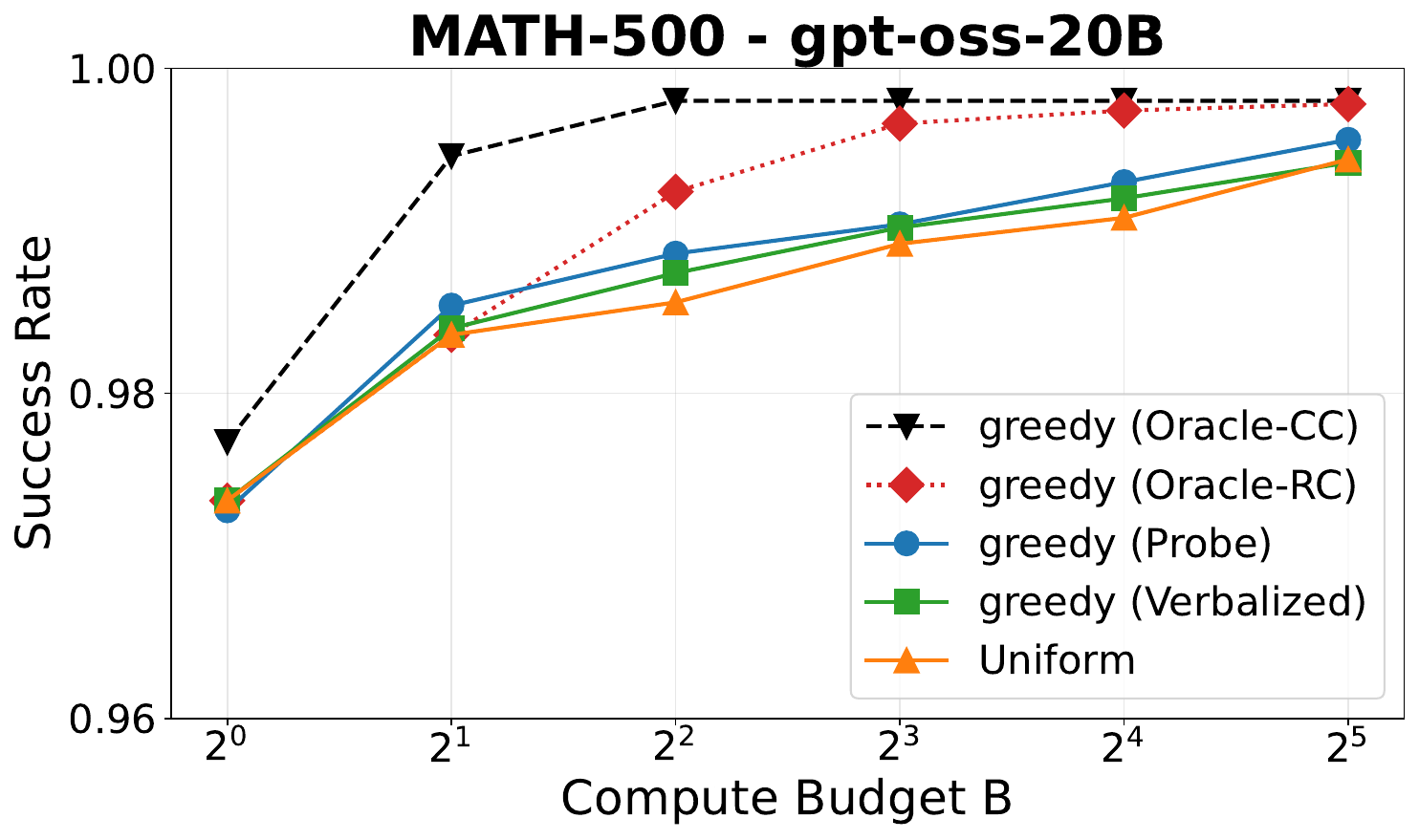}
        \captionof{figure}{
          \textbf{Inference budget allocation performance of confidence estimations.} Given $N$ questions, we evaluate the performance (success rate) of different methods under the fixed inference budget $N\times B$. The Oracle-CC achieves the best performance and consistently outperforms Oracle-RC. Meanwhile, confidence estimators (verbalized and Probe) both outperform the Uniform allocation in various budgets.
        }
        \label{fig: budget_allo_gpt_math}
    \end{minipage}
    
\end{figure}

The ``gain'' metric defined by \citet{damani2024learning} relies directly on expected accuracy in \cref{equ:new confidence definition}. Consequently, capability-calibrated confidence allows us to analytically estimate this gain and apply the greedy allocation algorithm detailed in Appendix \ref{sec: mit greedy algorithm} for inference budget allocation. Besides the confidence estimators we compared in \cref{sec: passk simulation}, we also test Verbalized Confidence (Verbalized), an estimator that is applicable to black-box models, which we discuss in \cref{sec: conf estimators}.   

Experimental results validate the effectiveness of capability-calibrated confidence for inference budget allocation.
Notably, Oracle-RC is fundamentally ill-suited for this task: since RC targets the confidence of a decoded response, obtaining an RC score consumes the very budget the system is trying to allocate.
\cref{fig: budget_allo_gpt_math} shows results for gpt-oss-20b on MATH-500; additional results for other models and datasets are provided in Appendix \ref{sec: more budget_alloc results}.
Oracle-CC yields the best performance across all compute budgets, and Probe consistently outperforms uniform allocation.
Oracle-RC performs worse than Oracle-CC, confirming that hill-climbing CC is the better objective. 
Furthermore, verbalized confidence achieves results comparable to Probe without requiring access to internal model states, suggesting that the benefits of capability-calibrated confidence extend to API-based models.

\subsection{Discussion on other applications} \label{sec: other applications}
As an initial investigation, we prioritize two applications (\S\ref{sec: passk simulation}, \S\ref{sec: budget alloc}) where performance ties directly to expected accuracy.
Beyond these, capability-calibrated confidence could potentially support selective prediction \citep{kamath2020selective}, cost estimation \citep{wu2024inference}, curriculum learning \citep{zhang2025beyond}, and label-free benchmarking \citep{guha2024smoothie}.
Capability calibration is also more efficient than response calibration for applications such as model routing \citep{chen2023frugalgpt,ong2024routellm} and active query refinement \citep{wu2024need}.
Appendix \ref{sec: conceptual discussion with other appli.} provides a conceptual discussion of these domains and capability calibration's advantages over response calibration.


\noindent \textbf{Implication: choosing the right calibration metric.}
Capability calibration is preferable when the goal is to assess a model's expected accuracy, or to forecast model capability efficiently before generation.
Response calibration, in contrast, is more appropriate for assessing the correctness of a single generated response.
Different applications call for different metrics, and capability calibration fills the gap that response calibration cannot address for stochastic models like LLMs. Appendix \ref{sec: estimator selection} further quantifies the results of using the wrong calibration framework in downstream applications.

\section{Towards confidence estimation and evaluation with capability calibration}

Given capability calibration's utility across diverse applications, we now ask: \textit{how can capability-calibrated confidence be achieved?}
We establish an empirical evaluation framework (\cref{sec: exp evaluation}), introduce confidence estimators commonly used for response calibration and adapt them for capability calibration (\cref{sec: conf estimators}), and benchmark their performance (\cref{sec:how achievable cc is}).

\subsection{Evaluation framework} \label{sec: exp evaluation}
For a given query $x$, an LLM's theoretical expected accuracy $\mu$ defined in \cref{equ:new confidence definition} is not directly accessible, so one has to estimate $\mu$ empirically.
For a given $x$, we estimate the LLM's expected accuracy by sampling $k_\mathrm{eval}$ responses $\{\hat{y}_1,...,\hat{y}_{k_\mathrm{eval}}\}$ from $f_\theta(\cdot \mid x)$.
Let $c$ denote the number of correct responses.
The estimated expected accuracy is
$\hat{\mu} = c/ k_\mathrm{eval}.$

\noindent \textbf{Target differs from single-response correctness.} Since response stochasticity does not guarantee correctness stochasticity, we need to check whether estimated expected accuracy or single-response correctness are empirically different calibration targets.
In \cref{fig: metric_diff main}, we compare these two targets using Olmo-3-7B-Instruct on TriviaQA (see \cref{sec: model and dataset details} for setup).
Additional results are in Appendix \ref{app: full metric diff}.
Consistent with findings in \citet{zhang2025beyondsingular}, the results show that LLM output correctness is rarely binary.
This variance confirms that CC targets an objective different from RC.

\textbf{Evaluation metrics:}
\textit{Absolute calibration metric.}
Since capability calibration targets query-level performance, we require a metric that preserves per-query granularity.
Following the discussion in \cref{sec: cc definition}, we use Brier score to measure calibration quality.
Given a dataset of $N$ queries, we compute the empirical capability calibration Brier score following \cref{equ: new brier score} as the mean squared error between the confidence estimate $s_i$ and the empirically estimated expected accuracy $\hat{\mu}_i$ for each query $x_i$.
Lower Brier scores indicate better calibration.

\noindent \textit{Ranking metric.}
Beyond absolute calibration quality, ranking performance is also important for applications that require query-level discrimination.
We report ranking performance in \cref{sec:spearman}.

\subsection{Methods for confidence estimation} \label{sec: conf estimators} 
We next examine the practical feasibility of capability calibration.
We first establish a baseline, then adapt confidence estimators commonly used for response calibration to the capability calibration setting and benchmark their performance.
To make the training-free methods ``verbalized confidence'' and ``P(True)'' more competitive, we also augment them with isotonic regression on the training sets.

\noindent \textbf{Uniform random baseline.}
To assess whether a method delivers meaningful performance, we establish a baseline that uses no information about the query.
For each query $x_i$, we sample a confidence score from a uniform distribution $s_i \sim U(0,1)$.
This baseline admits an analytic expected loss as follows, where $\bar{\mu} = \frac{1}{N}\sum_i \hat{\mu}_i$ is the dataset's mean accuracy (see Appendix \ref{sec: implmentation of confidence estimators} for derivation):
\begin{equation} \label{equ: random baseline formula}
\mathbb{E}_{\mathbf{s}}\!\left[\mathcal{L}_{\mathrm{Brier}}^{\mathrm{capability}}(\mathbf{s})\right]
= \frac{1}{N}\sum_{i=1}^N
\left(\tfrac{1}{3} - \hat{\mu}_i + \hat{\mu}_i^2\right) = \tfrac{1}{3} - \bar{\mu} + \bar{\mu}^2 + \sigma^2,
\end{equation}
Any useful confidence estimation method should outperform this baseline.

\noindent \textbf{Response consistency} \citep{wang2022self}.
A straightforward way to estimate confidence is by measuring the consistency across multiple sampled responses.
We sample $k_c$ responses and compute the fraction that agree with the majority prediction.
For example, if $k_c=10$ and $7$ responses are equivalent, the confidence estimate would be $0.7$.
Note that this method incurs a higher computational cost than other methods, as it requires sampling $k_c$ full responses per query.

\noindent \textbf{Verbalized confidence \citep{lin2022teaching}.} 
We instruct the LLM to report a probability in $[0,1]$ in natural language.
Unlike prior work~\cite{lin2022teaching,tian2023just}, which asks for confidence in the \emph{response} given the query, we ask for confidence in the \emph{query} itself to measure query-level capability.
See Appendix \ref{sec: implmentation of confidence estimators} for the prompt.

\noindent \textbf{P(True) \citep{kadavath2022language}.}
We ask the model if it can answer the \emph{query} correctly by instructing it to respond with only ``Yes'' or ``No''.
We extract the logprobs of these two tokens and use the softmax probability of ``Yes'' as the confidence estimate.
Unlike prior work~\citep{kadavath2022language}, we provide only the \emph{query}, not with the \emph{response}.
We provide the prompt in Appendix \ref{sec: implmentation of confidence estimators}.

\noindent \textbf{Probing LLMs' hidden states \citep{li2021implicit}.}
We train linear probes \cite{li2021implicit} on LLMs' internal representations to predict query-level confidence.
We mean-pool activations from the last input token across transformer blocks to output a confidence score.
This approach incurs low overhead, with an inference cost less than decoding a single token.
See Appendix~\ref{app:probing} for details.


\noindent \textbf{\textit{Cost.}}
When measuring the cost of confidence estimators, we focus on \emph{online inference cost}.
In LLM deployments, preprocessing or training is typically a one-time offline expense, whereas inference cost and latency accumulate continuously during real-world use.
Thus, inference efficiency is the primary practical consideration.
We provide a discussion of confidence estimation cost in Appendix~\ref{sec: cost of conf estimators}.

\begin{table}[t!]
\centering
\small
\caption{\textbf{Capability calibration performance of existing confidence estimators with three LLMs on seven datasets.} We use \textbf{bold} to denote the \textbf{best} calibrated method, and \underline{underline} to denote the \underline{second best}. While the calibration of verbalized confidence and $P(\text{True})$ varies across models, applying isotonic regression yields improvements. Verbalized confidence achieves the best performance for gpt-oss-20b, while Probe is generally the most effective for the other two models.
}
\label{tab:main-results}
\resizebox{\textwidth}{!}{
\begin{tabular}{l c | c c c c c c c}
\toprule
\textbf{Brier score (↓)} & \textbf{Domain} & \multicolumn{2}{c}{\textbf{Factual knowledge}} & \multicolumn{3}{c}{\textbf{Mathematical reasoning}} & \multicolumn{2}{c}{\textbf{General exams}} \\
\cmidrule(lr){3-4} \cmidrule(lr){5-7} \cmidrule(lr){8-9}
\textbf{Method} & \textbf{Cost} & \textbf{TriviaQA} & \textbf{SimpleQA} & \textbf{GSM8K} & \textbf{\mbox{MATH}} & \textbf{AIME} & \textbf{MMLU} & \textbf{GPQA} \\
\midrule
\multicolumn{2}{l|}{\textbf{Olmo-3-7B-Instruct}} & \multicolumn{7}{c}{} \\
\hspace{3mm}Uniform random baseline & N/A & 0.2745 & 0.3133 & 0.3119 & 0.2940 & 0.2462 & 0.2565 & 0.2125 \\
\hspace{3mm}Verbalized confidence & $L$ & 0.2717 & 0.2741 & 0.0444 & 0.0632	& 0.2126	& 0.1556	& 0.2672 \\
\hspace{6mm}+ isotonic regression & $L$ & 0.1840 & \underline{0.0138} & 0.0397 & \underline{0.0501}	& \underline{0.1544} &	\underline{0.1200} &	\underline{0.1205} \\
\hspace{3mm}P(True) & 1 & 0.1698 & 0.0416 & 0.1266 & 0.1392 & 0.1791 & 0.2446 & 0.1520 \\
\hspace{6mm}+ isotonic regression & 1 & \underline{0.1666} & 0.0149 & \underline{0.0389} & 0.0523 & 0.1761 & 0.1228 & 0.1222 \\
\hspace{3mm}Probe & $<1$ & \textbf{0.1101} & \textbf{0.0130} & \textbf{0.0364} & \textbf{0.0373} & \textbf{0.1092} & \textbf{0.1004} & \textbf{0.1072} \\
\midrule
\multicolumn{2}{l|}{\textbf{Qwen3-8B}} & \multicolumn{7}{c}{} \\
\hspace{3mm}Uniform random baseline & N/A & 0.2865 & 0.3109 & 0.3144 & 0.2781 & 0.2800 & 0.2868 & 0.2113 \\
\hspace{3mm}Verbalized confidence & $L$ & 0.2459 & 0.4673 & 0.0455 & 0.0977 & 0.4093 & 0.1263 & 0.2719 \\
\hspace{6mm}+ isotonic regression & $L$ & 0.1661 & \textbf{0.0271} & 0.0405 & 0.0822 & 0.1327 & 0.1085 & \underline{0.1351} \\
\hspace{3mm}P(True) & 1 & 0.3253 & 0.6010 & 0.0483 & 0.1124 & 0.4728 & 0.1548 & 0.3422 \\
\hspace{6mm}+ isotonic regression & 1 & \underline{0.1476} & \underline{0.0277} & \underline{0.0391} & \underline{0.0739} & \underline{0.0904} & \underline{0.1033} & \textbf{0.1243} \\
\hspace{3mm}Probe & $<1$ & \textbf{0.1163} & 0.0306 & \textbf{0.0377} & \textbf{0.0471} & \textbf{0.0578} & \textbf{0.0987} & 0.1571 \\
\midrule
\multicolumn{2}{l|}{\textbf{gpt-oss-20b}} & \multicolumn{7}{c}{} \\
\hspace{3mm}Uniform random baseline & N/A & 0.2639 & 0.3010 & 0.3195 & 0.3063 & 0.2369 & 0.3018 & 0.2388 \\
\hspace{3mm}Verbalized confidence & $L$ & 0.1204 & 0.1885 & \underline{0.0247} & \textbf{0.0186} & \underline{0.0581} & \underline{0.0802} & \underline{0.1066} \\
\hspace{6mm}+ isotonic regression & $L$ & \textbf{0.0813} & \textbf{0.0216} & \textbf{0.0245} & 0.0216 & \textbf{0.0545} & \textbf{0.0752} & \textbf{0.0898} \\
\hspace{3mm}P(True) & $L$ & 0.2435 & 0.5998 & 0.0299 & 0.0219 & 0.1446 & 0.1220 & 0.1933 \\
\hspace{6mm}+ isotonic regression & $L$ & 0.1417 & \underline{0.0238} & 0.0259 & \underline{0.0196} & 0.0852 & 0.0895 & 0.1183 \\
\hspace{3mm}Probe & $<1$ & \underline{0.1027} & 0.0264 & 0.0291 & 0.0216 & 0.0966 & 0.0878 & 0.1462 \\
\bottomrule
\end{tabular}
}
\end{table}

\subsection{Experiments} \label{sec:how achievable cc is}

\subsubsection{Setup} \label{sec: model and dataset details}
\noindent \textbf{Choice of $k_\mathrm{eval}$.}
The estimated expected accuracy $\hat{\mu}$ is binomial with variance $\mu(1-\mu) / k_\mathrm{eval}$, which decreases as $k_\mathrm{eval}$ increases.
We investigate the effect of $k_\mathrm{eval}$ on evaluation reliability in Appendix~\ref{sec: the choice of k} and \ref{sec:k=200 results}. We set $k_\mathrm{eval}$ to $100$ to balance cost and reliability. We also report $k_\mathrm{eval} > 100$ results and error bars in Appendix~\ref {sec:k=200 results}.

\noindent \textbf{Models.}
We use Olmo-3-7B-Instruct~\cite{olmo2025olmo}, Qwen3-8B~\cite{yang2025qwen3}, and gpt-oss-20b~\cite{agarwal2025gpt} to ensure diversity.
We adopt developer-recommended sampling hyperparameters (Appendix~\ref{sec: sampling hyperparameters}).
We also evaluate different decoding temperatures in Appendix~\ref{sec:temperature-transfer}.

\noindent \textbf{Datasets.}
We select datasets from three task domains:
(1) \emph{factual knowledge}, which tests parametric knowledge;
(2) \emph{mathematical reasoning}, where errors compound across multiple intermediate steps; and
(3) \emph{general exams}, which test both knowledge and reasoning in multiple subjects.
For each type, we include datasets of different difficulty levels based on accuracy: \emph{\textbf{Factual knowledge:}} We use TriviaQA~\cite{joshi2017triviaqa} as the easy and SimpleQA verified~\cite{haas2025simpleqa} as the hard one. \emph{\textbf{Mathematical reasoning:}} We adopt GSM8K~\cite{cobbe2021training} as the easiest dataset, MATH-500~\cite{lightman2023let} as the intermediate one, and AIME (questions from 2024 to 2026) as the hardest. \emph{\textbf{General exams:}} We use MMLU~\cite{hendrycks2020measuring}, which spans 57 subjects. 
We use GPQA~\cite{rein2023gpqa} as the harder one.
We report train-test split details in Appendix~\ref{app:probing}.

\subsubsection{Results and discussion}
\label{sec:results}

\noindent \textbf{Probe has the best online cost-performance tradeoff.}
A practically useful method should satisfy two properties:
(1) \emph{Acceptable inference cost}: no higher than decoding the response itself, otherwise the overhead would limit the method's practical utility (see \S\ref{sec:applications}).
(2) \emph{Good calibration performance}: at least better than the uniform random baseline, with lower Brier scores preferred.
Figure~\ref{fig:cost-performance-subplot} shows a representative example, with full results shown in Figure~\ref{fig:cost-performance-fullplot}.
Among evaluated methods, Probe has the lowest inference cost while consistently outperforming the random baseline.
We further analyze the effects of training data for probes in Appendix \ref{sec: probe in ood}.

\begin{wrapfigure}{r}{0.5\textwidth}
\vspace{-15pt}
\centering
\includegraphics[width=0.5\textwidth]{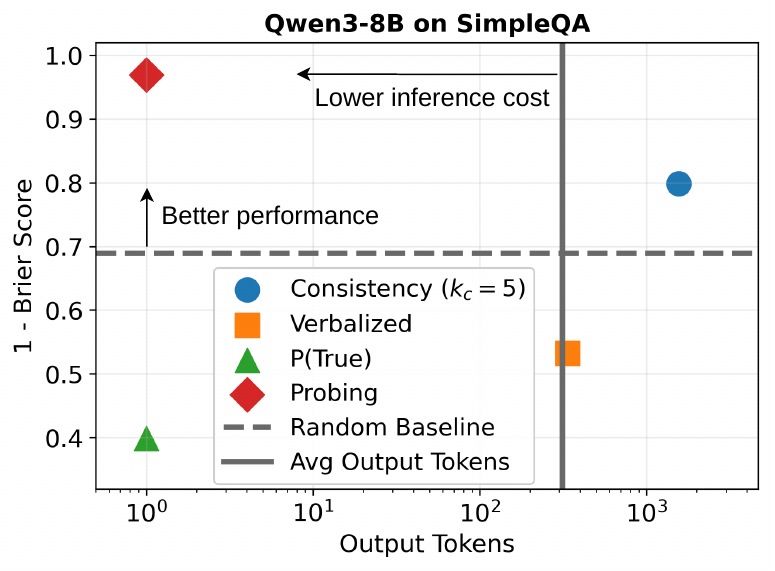}
    \caption{
        \textbf{Cost-performance tradeoff of different methods.}
        We compare inference cost (x-axis, log-scale) against calibration performance (y-axis, 1 - Brier score), where the upper-left corner is the ideal region.
        See Figure~\ref{fig:cost-performance-fullplot} for full results.
    }
    \label{fig:cost-performance-subplot}
\vspace{-20pt}
\end{wrapfigure}

\begin{figure*}[t!]
\centering
\includegraphics[width=\textwidth]{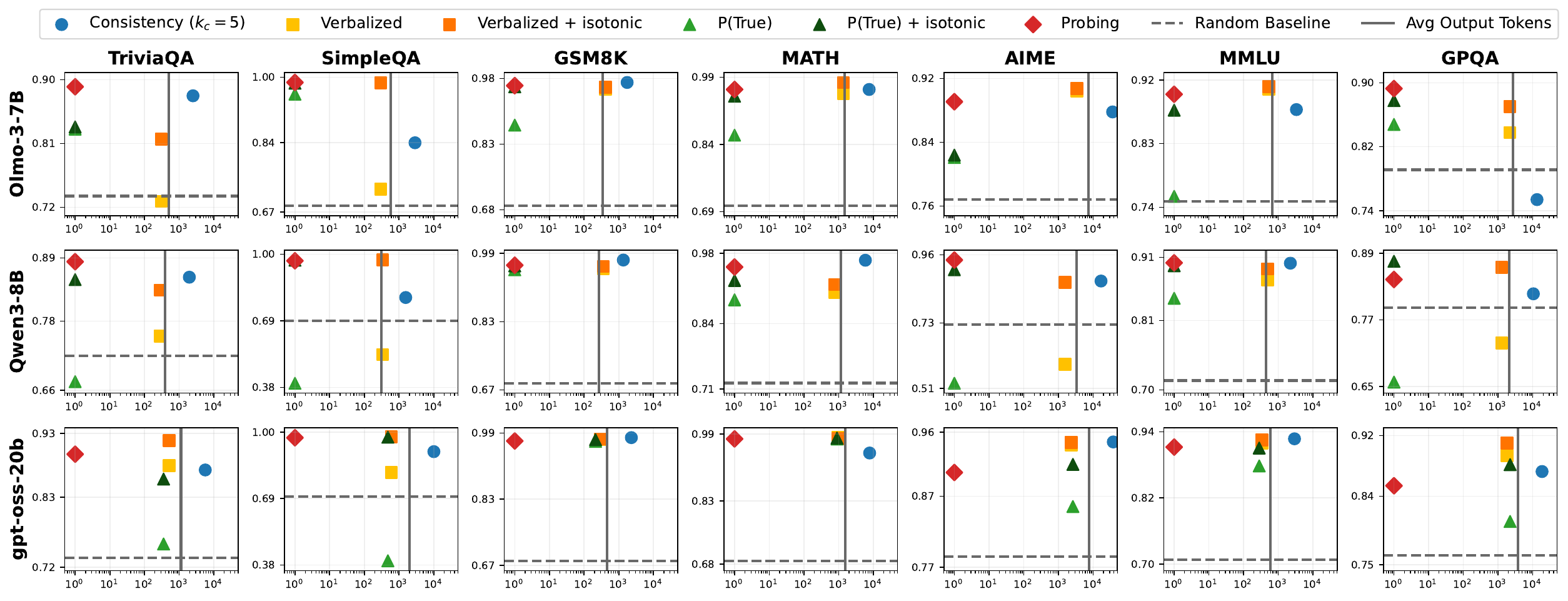}
    \caption{
        \textbf{Cost-performance tradeoff of different confidence estimation methods with three LLMs on seven datasets.}
        Following Figure~\ref{fig:cost-performance-subplot}, we compare inference cost (x-axis, average response tokens) against calibration performance (y-axis).
        Probes consistently outperform the random baseline with the lowest cost, while response consistency incurs a cost higher than decoding responses.
    }
    \label{fig:cost-performance-fullplot}
\end{figure*}

\noindent \textbf{Performance of verbalized confidence and P(True) differs across LLMs.}
As shown in Table~\ref{tab:main-results}, gpt-oss-20b performs strikingly well with verbalized confidence, achieving the best or second-best performance across datasets.
In contrast, Qwen3-8B does not even consistently outperform the random uniform baseline with verbalized confidence.
Moreover, verbalized confidence outperforms P(True) for Qwen3-8B and gpt-oss-20b, but not for Olmo-3-7B-Instruct.
These results suggest that the effectiveness of these training-free methods is model-dependent. Isotonic regression enhances the performance of two methods by recalibrating confidence estimates based on training labels, which makes these methods become training-based. We also explored enhancing verbalized confidence through few-shot prompting, as detailed in Appendix \ref{sec: improving verbal and ptrue}.

\noindent \textbf{Response consistency costs more than responding.}
This method costs more than decoding the response itself
(Figures~\ref{fig:cost-performance-subplot} and \ref{fig:cost-performance-fullplot}), so it is impractical when confidence must be estimated efficiently, such as in inference budget
allocation (\S\ref{sec: budget alloc}).
We report its performance in Appendix~\ref{sec: response consistency additional results}.

\section{Discussion} \label{sec: discussion}
This work formalizes capability calibration and shows that it differs from existing response calibration due to the stochastic nature of LLM outputs. We demonstrate its superiority in downstream applications, including pass@$k$ prediction and inference budget allocation. We also confirm that capability-calibrated confidence estimations are practically achievable.

Our study establishes a new calibration evaluation framework for confidence estimation rather than proposing new estimation methods;
we envision future work developing confidence estimators that push the frontier of capability calibration, and extending the framework to applications such as model routing, human--AI collaboration, and trustworthy AI.


\clearpage


\subsection*{Ethics statement}
This paper presents work whose goal is to advance the field
of Machine Learning by improving the reliability and predictability of LLMs. As LLMs are increasingly deployed
in real-world applications, ensuring they are trustworthy is
paramount. Our framework for capability calibration enables models to more accurately assess their own limitations,
allowing systems to abstain or seek human oversight when
the model is unlikely to succeed. We believe this contributes to safer AI deployment by mitigating the risks associated with overconfidence and hallucination. 

\subsection*{Reproducibility statement}

Experiment details are clearly described in the main paper and Appendix \ref{sec:experiment_setup} to reproduce the empirical results. We also have provided the source code at \url{https://github.com/appier-research/llm-calibration}. The datasets used are all publicly available.

\section*{Acknowledgements}
We would like to thank Appier AI Research team members, Hsuan-Tien Lin (National Taiwan University) and Wei-Lin Chen (University of Virginia), for their feedback on this work.
This work was supported in part by the National Science and Technology Council, Taiwan, under the Grant 114-2628-E-002-021-, and the Taiwan Centers of Excellence. Shao-Hua Sun was supported by the Yushan Fellow Program of the Ministry of Education, Taiwan.

\clearpage

\appendix
\section*{Appendix}

\vspace{-1.8cm}
\begingroup
\hypersetup{colorlinks=false, linkcolor=black}
\hypersetup{pdfborder={0 0 0}}
\part{} 
\parttoc 
\endgroup




\section{Details of Capability Calibration} \label{sec:details_of_CC}
\subsection{Optimal confidence differs under two calibration definitions} \label{sec: proof of metrics differ}
In this section, we show that the optimal confidence of response calibration is different from the optimal confidence capability calibration.
\begin{theorem} \label{the: optimal confidence differs}
\textbf{(Divergence of targets and optima).}
Let $x$ be an input and $\hat{y} \sim f_\theta(\cdot \mid x)$ be a generated response.
Minimizing the Brier scores for response calibration (\cref{equ: classical brier score}) and capability calibration (\cref{equ: new brier score}) yields distinct optimal confidence estimators:
\begin{equation}
\begin{split}
        & s^*_{\mathrm{resp}}(x, \hat{y}) = \mathcal{C}(x, \hat{y}) \in \{0, 1\}, \\
& s^*_{\mathrm{cap}}(x, f_{\theta}) = \mathbb{E}_{\hat{y} \sim f_\theta(\cdot \mid x)}[\mathcal{C}(x,\hat{y})] \in [0, 1].
\end{split}
\end{equation}
Unless the model is deterministic, or its predictions are always correct or always incorrect, the evaluation targets differ, i.e., $\mathcal{C}(x, \hat{y}) \neq \mathbb{E}_{\hat{y} \sim f_\theta(\cdot \mid x)}[\mathcal{C}(x,\hat{y})]$, implying distinct optimal confidence values.
\end{theorem}

\begin{proof}
Let $f_\theta(\cdot \mid x)$ be a generative model.
For each input $x_i$, let $\hat y_i \sim f_\theta(\cdot\mid x_i)$ be a single
sampled output, and let $s_i \in [0,1]$ denote the predicted confidence.

Recall that the response calibration Brier score (\cref{equ: classical brier score}) is defined as:
\[
\mathcal{L}^{\mathrm{response}}_{\mathrm{Brier}}(s_1,\dots,s_N)
\triangleq
\frac{1}{N}\sum_{i=1}^N
\left(
s_i - \mathcal{C}(x_i,\hat y_i)
\right)^2,
\quad
\hat y_i \sim f_\theta(\cdot\mid x_i).
\]

Define the confidence estimations of the dataset as $\mathbf{s} = (s_1, s_2, ..., s_n)$. Since the Brier score is convex, the stationary point is the global minimum, which is
\[
\nabla \mathcal{L}^{\mathrm{response}}_{\mathrm{Brier}}(\mathbf{s}^*) = \mathbf{0},
\]
i.e.,
\[
\frac{\partial \mathcal{L}^{\mathrm{response}}_{\mathrm{Brier}}}{\partial s_i}
=
\frac{2}{N}(s_i - \mathcal{C}(x_i,\hat y_i)) =0, \quad \forall i \in \{1,\dots,N\}.
\]

Thus, the optimal confidence estimation is
\[
s_i^{*,\mathrm{response}} = \mathcal{C}(x_i,\hat y_i),
\quad \forall i \in \{1,\dots,N\}.
\]

Follow the definition of \cref{equ:new confidence definition}
\[
\mu_i \triangleq
\mathbb{E}_{\hat{y}\sim f_\theta(\cdot\mid x_i)}[\mathcal{C}(x_i,\hat{y})],
\]

we restate the capability calibration Brier score from \cref{equ: new brier score}:
\[
\mathcal{L}_{\mathrm{Brier}}^{\mathrm{capability}}(s_1,\dots,s_N)
\triangleq
\frac{1}{N}\sum_{i=1}^N (s_i - \mu_i)^2
=
\frac{1}{N}\sum_{i=1}^N
\left(
s_i -
\mathbb{E}_{\hat{y}\sim f_\theta(\cdot\mid x_i)}
[\mathcal{C}(x_i,\hat{y})]
\right)^2.
\]

Following the same proof, the optimal confidence that minimize $\mathcal{L}_{\mathrm{Brier}}^{\mathrm{capability}}$ satisfies
\[
s_i^{*,\mathrm{capability}} = \mu_i,
\quad \forall i \in \{1,\dots,N\}.
\]
For a generative model, $\mathcal{C}(x_i,\hat y_i)\in\{0,1\}$,
while $\mu_i \in [0,1]$.
Unless the model is deterministic or perfectly correct/incorrect on $x_i$,
we have
\[
\mathcal{C}(x_i,\hat y_i) \neq \mu_i \triangleq
\mathbb{E}_{\hat{y}\sim f_\theta(\cdot\mid x_i)}[\mathcal{C}(x_i,\hat{y})].
\]
Therefore, 
\[
(s_1^{*,\mathrm{response}},\dots,s_N^{*,\mathrm{response}})
\neq
(s_1^{*,\mathrm{capability}},\dots,s_N^{*,\mathrm{capability}}),
\]
i.e., the two calibration objectives induce different optimal confidence
predictors over the dataset.
\end{proof}

\subsection{Proof of \texorpdfstring{\cref{the: metrics connection}}{Theorem 2}} \label{sec: proof of metrics connection}
\begin{proof}
First, expand the expectation of the Brier Score loss on the old calibration 
\begin{equation}
\label{equ:old_loss}
\begin{split}
\mathbb{E}\big[\mathcal{L}_{\mathrm{Brier}}\big]
\; & \triangleq\;
\frac{1}{N}\sum_{i=1}^N
\mathbb{E}_{\hat{y} \sim f_\theta(\cdot \mid x_i)}
\left[
\big(s_i - \mathcal{C}(x_i,\hat{y})\big)^2
\right] \\
& = \frac{1}{N}\sum_{i=1}^N \Big( s_i^2 - 2 s_i \cdot \mathbb{E}_{\hat{y} \sim f_\theta(\cdot \mid x_i)}[\mathcal{C}(x_i,\hat{y})] + \mathbb{E}_{\hat{y} \sim f_\theta(\cdot \mid x_i)}[\mathcal{C}(x_i,\hat{y})^2] \Big). 
\end{split}
\end{equation}

Second, expand the Brier Score loss on expected calibration 
\begin{equation}
\label{equ:new_loss}
\mathcal{L}_{\mathrm{Brier}}^{\mathrm{expected}} 
\triangleq \frac{1}{N}\sum_{i=1}^N (s_i - \mu_i)^2
= \frac{1}{N}\sum_{i=1}^N \big( s_i^2 - 2 s_i\, \mu_i + \mu_i^2 \big),
\end{equation}

\noindent
Subtracting \cref{equ:new_loss} - \cref{equ:old_loss}, we obtain
\begin{equation}
\begin{split}
    \mathcal{L}_{\mathrm{Brier}}^{\mathrm{expected}} - \mathbb{E}[\mathcal{L}_{\mathrm{Brier}}]
    & = \frac{1}{N}\sum_{i=1}^N \big( \mu_i^2 -2s_i\cdot (\mu_i -\mathbb{E}_{\hat{y} \sim f_\theta(\cdot \mid x_i)}[\mathcal{C}(x_i,\hat{y})])- \mathbb{E}_{\hat{y} \sim f_\theta(\cdot \mid x_i)}[\mathcal{C}(x_i,\hat{y})^2] \big) \\
& = \frac{1}{N}\sum_{i=1}^N \big( \mu_i^2 - \mathbb{E}_{\hat{y} \sim f_\theta(\cdot \mid x_i)}[\mathcal{C}(x_i,\hat{y})^2] \big) \\
& = \frac{1}{N}\sum_{i=1}^N \big( \mathbb{E}_{\hat{y} \sim f_\theta(\cdot \mid x_i)}[\mathcal{C}(x_i,\hat{y})]^2 - \mathbb{E}_{\hat{y} \sim f_\theta(\cdot \mid x_i)}[\mathcal{C}(x_i,\hat{y})^2] \big) \\
& = - \frac{1}{N}\sum_{i=1}^N \mathrm{Var}[\mathcal{C}(x_i,\hat{y})].
\end{split}
\end{equation}

\end{proof}

\subsection{Connection between response calibration loss and capability calibration loss}
\label{sec: bregman info}
For the same loss function $\mathcal{L}(s,t)$ that evaluates the agreement between confidence $s$ and target $t$, capability calibration and response calibration evaluate $s$ with different $t$. Denote the loss function as $\mathcal{L}_s(t)$ when we discuss the impact of $t$.

Capability calibration takes the expected accuracy as the target. In general, capability calibration loss is  
\[
\mathcal{L}_s(\mathbb{E}_{\hat{y} \sim f(\cdot|x)}[\mathcal{C}(x,\hat{y})]),
\]
we simplify it as $\mathcal{L}_s(\mathbb{E}[\mathcal{C}(x,\hat{y})])$.

Response calibration takes the individual response's correctness as the target. Therefore, the expectation of the response calibration loss is
\[
\mathbb{E}_{\hat{y} \sim f(\cdot|x)}[\mathcal{L}_s(\mathcal{C}(x,\hat{y}))],
\]
we simplify it as $\mathbb{E}[\mathcal{L}_s(\mathcal{C}(x,\hat{y}))]$.

The difference between the capability calibration loss and the response calibration loss
\begin{equation}
    \mathbb{E}[\mathcal{L}_s(\mathcal{C}(x,\hat{y}))] - \mathcal{L}_s(\mathbb{E}[\mathcal{C}(x,\hat{y})]),
\end{equation}

is known as Jensen Gap \citep{reid2011information}. For loss functions that are strictly convex and differentiable, Jensen Gap is equivalent to the Bregman information \citep{banerjee2005clustering}, which measures the diversity of the random variable $\mathcal{C}(x,\hat{y})$ through the lens of the loss function $\mathcal{L}$. For square error loss, the Bregman information is the variance of $\mathcal{C}(x,\hat{y})$ \citep{banerjee2005clustering}. 

\clearpage

\section{Details of experiment setup}\label{sec:experiment_setup}
\subsection{The choice of \texorpdfstring{$k_\mathrm{eval}$}{k eval}} \label{sec: the choice of k}

\paragraph{Theoretical insights.}
Let $f_\theta(\cdot\mid x)$ denote the LLM’s conditional output distribution given a query $x$, and let
$\mathcal{C}(x,\hat{y})\in\{0,1\}$ be a deterministic correctness function that indicates whether a sampled response $\hat{y}$ is correct for $x$.
The \emph{expected accuracy} (our capability-calibration target) is
\[
\mu(x,f_\theta)\triangleq \mathbb{P}_{\hat{y}\sim f_\theta(\cdot\mid x)}[\mathcal{C}(x,\hat{y})=1].
\]
Since $\mu(x,f_\theta)$ is not directly observable, we approximate it by drawing $k_\mathrm{eval}$ independent samples
$\hat{y}_1,\ldots,\hat{y}_{k_\mathrm{eval}}\sim f_\theta(\cdot\mid x)$ and computing the empirical mean
\[
\hat{\mu}\;\triangleq\;\frac{1}{k_\mathrm{eval}}\sum_{j=1}^{k_\mathrm{eval}} \mathcal{C}(x,\hat{y}_j).
\]
For convenience, define the per-sample correctness indicator $Z_j \triangleq \mathcal{C}(x,\hat{y}_j)\in\{0,1\}$ and the number of correct samples
\[
c \;\triangleq\; \sum_{j=1}^{k_\mathrm{eval}} Z_j,
\qquad\text{so that}\qquad
\hat{\mu} = \frac{c}{k_\mathrm{eval}}.
\]

\textbf{Binomial model and estimation variance.}
Under independent sampling randomness across repeated generations and a deterministic evaluator $\mathcal{C}$, we have
\[
Z_j \mid x \sim \mathrm{Bernoulli}(\mu),
\qquad
c \sim \mathrm{Binomial}(k_\mathrm{eval},\mu).
\]
Consequently, $\hat{\mu}$ is an unbiased estimator of $\mu$ with
\begin{equation} 
\mathbb{E}[\hat{\mu}] = \mu,
\qquad
\mathrm{Var}(\hat{\mu}) = \frac{\mu(1-\mu)}{k_\mathrm{eval}}.
\label{eq:mu_hat_var}
\end{equation}
Notably, this binomial structure is induced by the binary correctness indicator and does not assume any particular parametric form for the LLM’s raw text distribution.

\textbf{Sample-size guidance (normal/Wald approximation).}
Equation~\eqref{eq:mu_hat_var} implies that the standard error of $\hat{\mu}$ decays as $O(k_\mathrm{eval}^{-1/2})$.
A common rule-of-thumb for selecting $k_\mathrm{eval}$ is to control the margin of error of $\hat{\mu}$.
Using the asymptotic normal (Wald) approximation,
\[
\hat{\mu}\approx \mathcal{N}\!\left(\mu,\frac{\mu(1-\mu)}{k_\mathrm{eval}}\right),
\]
a two-sided $95\%$ confidence interval has approximate half-width (margin of error, \textsc{MoE})
\[
\mathrm{MoE} \approx z_{0.975}\sqrt{\frac{\mu(1-\mu)}{k_\mathrm{eval}}},\qquad z_{0.975}=1.96.
\]
We interpret $\epsilon$ as a target absolute accuracy for per-query estimation: we aim for the $95\%$ confidence interval of $\hat{\mu}$ to have half-width at most $\epsilon$, i.e., $|\hat{\mu}-\mu|\le \epsilon$ with approximately $95\%$ confidence under the normal approximation.
Thus, achieving $\mathrm{MoE}\le \epsilon$ suggests
\begin{equation}
k_\mathrm{eval}\ \gtrsim\ \frac{z_{0.975}^2\,\mu(1-\mu)}{\epsilon^2}.
\label{eq:wald_ksize}
\end{equation}
If $\mu$ is unknown, the conservative worst case uses $\mu(1-\mu)\le 1/4$ (attained at $\mu=0.5$), yielding
\begin{equation}
k_\mathrm{eval}\ \gtrsim\ \frac{z_{0.975}^2}{4\epsilon^2}.
\label{eq:wald_worstcase}
\end{equation}

\textbf{Practical considerations and our choice of $k_\mathrm{eval}$.}
While Equation \eqref{eq:wald_ksize}--\eqref{eq:wald_worstcase} are useful for intuition, Wald intervals can under-cover when $\mu$ is close to $0$ or $1$ and/or $k_\mathrm{eval}$ is small, and may produce bounds outside $[0,1]$.
We therefore use the Wald analysis only as a guideline for the scaling behavior in Equation \eqref{eq:mu_hat_var}, and complement it with a sensitivity analysis over $k_\mathrm{eval}$.
Prior work \citep{zhang2025beyondsingular} uses $k_\mathrm{eval}=50$ as the ground truth. We find that $k_\mathrm{eval}=100$ provides a stable per-query estimate of $\mu$ while keeping evaluation cost tractable; increasing $k_\mathrm{eval}$ beyond this point yields diminishing returns relative to the additional sampling cost.
For completeness, one may alternatively adopt Wilson score intervals (which typically provide closer-to-nominal coverage under the same Bernoulli sampling assumptions) to select $k_\mathrm{eval}$ via a short numerical search; our empirical validation supports that $k_\mathrm{eval}=100$ is a reliable operating point in our setting.

\paragraph{Empirical results.}
We show representative empirical results of gpt-oss-20b~\cite{agarwal2025gpt} on AIME25 in Figure~\ref{fig:k_eval}.
We found that the standard error of expected accuracies decreases to about 0.0056 when $k_\mathrm{eval}=100$.
Although larger $k_\mathrm{eval}$ further decreases variance, it shows diminishing benefits.
We finally chose $k_\mathrm{eval}=100$ to balance cost and reliability. Another supportive evidence is discussed in \cref{sec:k=200 results}.

\begin{figure*}[t!]
\centering
\includegraphics[width=\textwidth]{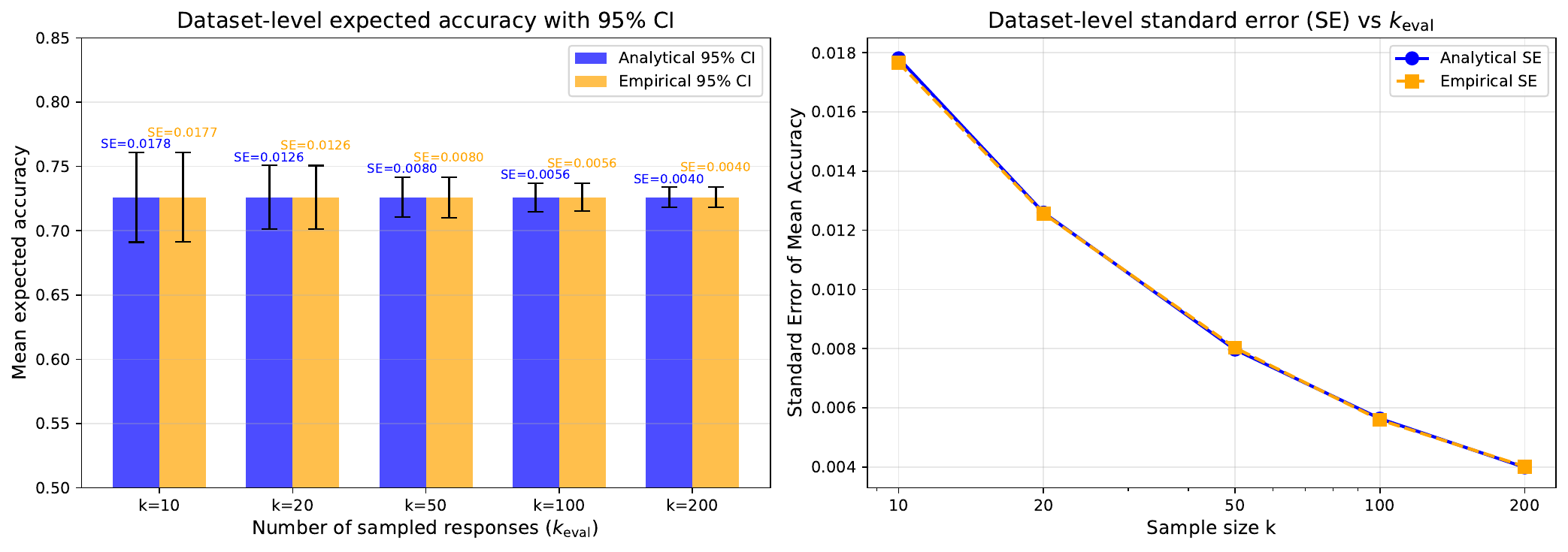}
    \caption{
        \textbf{Representative empirical results showing the standard error of expected accuracies obtained by different $k_\mathrm{eval}$.}
    }
    \label{fig:k_eval}
\end{figure*}

\subsection{Correcting Brier score for the measurement of expected accuracy}

In this section, we show that correcting Brier scores to obtain an unbiased estimator does not change the ranking of confidence estimators.

From Equation \ref{eq:mu_hat_var} in Appendix \ref{sec: the choice of k}, we show that
\[
\mathbb{E}[\hat{\mu}] = \mu,
\qquad
\operatorname{Var}(\hat{\mu})
= \frac{\mu(1-\mu)}{k_{\mathrm{eval}}}.
\]
Denote $k_{\mathrm{eval}} = k$. The expected Brier score over the output distribution is
\[
\mathbb{E}\left[(s_i-\hat{\mu}_i)^2\right]
=
(s_i-\mu_i)^2
+
\frac{\mu_i(1-\mu_i)}{k},
\]
which is a biased estimator of the true Brier loss. To obtain an unbiased estimator, we correct for the measurement noise in the empirical $\hat{\mu}_i$ by subtracting
\[
\frac{\hat{\mu}_i(1-\hat{\mu}_i)}{k-1}.
\]
Consequently,
\[
\begin{aligned}
\mathbb{E}\left[
(s_i-\hat{\mu}_i)^2
-
\frac{\hat{\mu}_i(1-\hat{\mu}_i)}{k-1}
\right]
&=
(s_i-\mu_i)^2
+
\frac{\mu_i(1-\mu_i)}{k}
-
\frac{\mu_i(1-\mu_i)}{k} \\
&=
(s_i-\mu_i)^2.
\end{aligned}
\]

Importantly, the correction term
$\frac{\hat{\mu}_i(1-\hat{\mu}_i)}{k-1}$
depends only on the label. Since all methods are evaluated on the same queries and labels, the same term is subtracted from every method. Therefore, the correction changes the absolute Brier scores but leaves method rankings within each model--dataset setting unchanged.

\subsection{Stability of experiment findings under different \texorpdfstring{$k_\mathrm{eval}$}{k eval}} \label{sec:k=200 results}

We conduct additional experiments to examine how $k_{eval}$ affects the stability of the calibration metrics. Using $k_{eval} =200$ as the reference, we subsample to $k_{eval} =\{10,20,...,190\}$, perform 1000 trials for each value, and measure both Brier score variability and the relative ranking stability of confidence estimation methods. In the resulting plots, solid lines show the mean and shaded regions show the 5th–95th percentile range. Full results are provided in \cref{fig:k_eval-200}.

We find that across most model-dataset pairs, both Brier scores and rankings of methods are already stable at  $k_{eval} =100$
, and in many cases even at smaller values. The experiment results show that $k_{eval} =100$ is sufficient for the scenarios in this paper.
 
\begin{figure}[h]
\centering
\includegraphics[width=\textwidth]{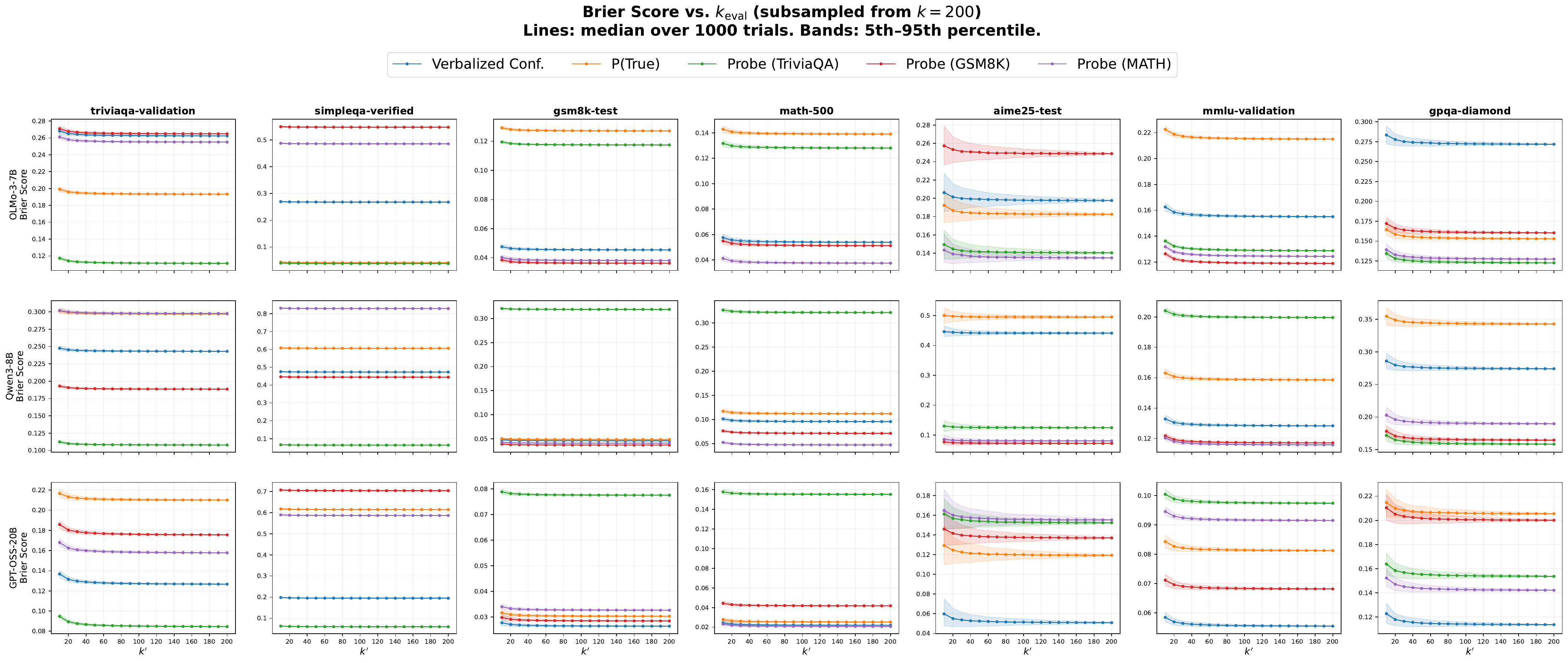}
    \caption{
        \textbf{Brier score of confidence estimators under different $k_{eval}$.} Probes are trained by training dataset of TriviaQA \citep{joshi2017triviaqa}, GSM8K \citep{cobbe2021training}, and MATH \citep{hendrycksmath2021}. Experimental results show that the Brier score is stable at $k_{eval} =100$ for all scenarios.
    }
    \label{fig:k_eval-200}
\end{figure}

\subsection{Sampling hyperparameters for each LLM.} \label{sec: sampling hyperparameters}
For each LLM used in the experiments, we use the sampling hyperparameters recommended by its model developers.
For Olmo-3-7B-Instruct~\cite{olmo2025olmo}, we use temperature=0.6 and top-p=0.95.
For Qwen3-8B~\cite{yang2025qwen3}, we use chat template of non-reasoning mode to save inference cost due to limited computational budget, and adopt temperature=0.7 and top-p=0.8.
For gpt-oss-20b~\cite{agarwal2025gpt}, we use temperature=1.0 and top-p=1.0.

\subsection{Transferability of capability calibration across decoding temperatures}
\label{sec:temperature-transfer}

We now analyze the transferability of Capability Calibration across decoding . Across 3 datasets, 3 models, and 3 decoding temperatures, we find that Capability calibration's target, expected accuracy, is stable at the dataset level and broadly robust at the instance level. Meanwhile, The Brier score and the ranking of confidence estimators are likewise stable.

\paragraph{Experiment setup.}
For three models we benchmarked in the paper, we test two temperatures in addition to each model's recommended temperature. We test one representative dataset per task: TriviaQA, GSM8K, and MMLU.

\paragraph{The stability of capability calibration's target.}
As capability calibration is an evaluation framework for confidence estimations, we first check how the evaluation target (expected accuracy) varies across different decoding temperatures.

\cref{tab:temperature-mean} shows that at the dataset level, the average expected accuracy is stable across different decoding temperatures.
Meanwhile, \cref{tab:temperature-instance} shows strong instance-level rank correlations (Spearman's $\rho$) for TriviaQA and MMLU. The correlations are lower for GSM8K, but the MSE remains consistently small across all settings. MSE remains small because many targets are near 0 or 1 (\cref{fig:full_metric_diff}), making aggregate errors more stable.

\begin{table}[ht]
\centering
\small
\caption{Average expected accuracy of three datasets.}
\label{tab:temperature-mean}
\resizebox{\textwidth}{!}{
\begin{tabular}{l c c c c}
\toprule
\textbf{Model} &
\textbf{Temperature} &
\textbf{TriviaQA} &
\textbf{GSM8K} &
\textbf{MMLU} \\
\midrule

Olmo-3-7B-Instruct
& 0.6 / 0.3 / 0.15
& 0.4884 / 0.4933 / 0.4957
& 0.9331 / 0.9340 / 0.9348
& 0.7006 / 0.6989 / 0.6989 \\

Qwen3-8B-non-thinking
& 0.7 / 0.35 / 0.175
& 0.6040 / 0.6122 / 0.6140
& 0.9356 / 0.9355 / 0.9361
& 0.7940 / 0.7979 / 0.7976 \\

gpt-oss-20b
& 1.0 / 0.5 / 0.25
& 0.6507 / 0.6625 / 0.6565
& 0.9576 / 0.9585 / 0.9566
& 0.8444 / 0.8429 / 0.8384 \\

\bottomrule
\end{tabular}
}
\end{table}

\begin{table}[th]
\centering
\small
\caption{Instance-level expected accuracy compared with the recommended temperature (Temp.).}
\label{tab:temperature-instance}
\resizebox{\textwidth}{!}{
\begin{tabular}{l c | cc | cc | cc}
\toprule
&
&
\multicolumn{2}{c|}{\textbf{TriviaQA}} &
\multicolumn{2}{c|}{\textbf{GSM8K}} &
\multicolumn{2}{c}{\textbf{MMLU}} \\
\textbf{Model} &
\textbf{Temp. pair} &
\textbf{Spearman's $\rho$} & \textbf{MSE} &
\textbf{Spearman's $\rho$} & \textbf{MSE} &
\textbf{Spearman's $\rho$} & \textbf{MSE} \\
\midrule

Olmo-3-7B-Instruct
& $0.6 \leftrightarrow 0.3$
& 0.9721 & 0.0048
& 0.8654 & 0.0012
& 0.9763 & 0.0034 \\

& $0.6 \leftrightarrow 0.15$
& 0.9522 & 0.0106
& 0.8443 & 0.0026
& 0.9654 & 0.0063 \\

\midrule

Qwen3-8B-non-thinking
& $0.7 \leftrightarrow 0.35$
& 0.9642 & 0.0052
& 0.8680 & 0.0016
& 0.9596 & 0.0026 \\

& $0.7 \leftrightarrow 0.175$
& 0.9348 & 0.0106
& 0.8129 & 0.0033
& 0.9404 & 0.0059 \\

\midrule

gpt-oss-20b
& $1.0 \leftrightarrow 0.5$
& 0.9449 & 0.0120
& 0.6734 & 0.0011
& 0.8784 & 0.0035 \\

& $1.0 \leftrightarrow 0.25$
& 0.9225 & 0.0208
& 0.6414 & 0.0027
& 0.8530 & 0.0063 \\

\bottomrule
\end{tabular}
}
\end{table}

\paragraph{Confidence estimator evaluation.}
After confirming the stability of the evaluation target, we next ask if the Brier score and the ranking of confidence estimators are stable as well. We test all confidence estimators in \cref{tab:main-results}. \cref{tab:temperature-brier} shows that estimator rankings mostly remain unchanged across temperatures, with Brier scores close to those at the recommended temperature (\cref{tab:main-results}).

\begin{table}[ht]
\centering
\small
\caption{\textbf{Brier score of the confidence estimator at different temperatures.} Iso denotes isotonic regression.}
\label{tab:temperature-brier}
\setlength{\tabcolsep}{5pt}
\begin{tabular}{l | ccc | ccc}
\toprule

\multicolumn{1}{l|}{\textbf{Olmo-3-7B-Instruct}}
& \multicolumn{3}{c|}{\textbf{Temp. 0.3}}
& \multicolumn{3}{c}{\textbf{Temp. 0.15}} \\
\textbf{Method}
& \textbf{TriviaQA} & \textbf{GSM8K} & \textbf{MMLU}
& \textbf{TriviaQA} & \textbf{GSM8K} & \textbf{MMLU} \\
\midrule

Uniform random baseline
& 0.2748 & 0.3143 & 0.2594
& 0.2743 & 0.3205 & 0.2594 \\
Verbalized confidence
& 0.2878 & 0.0483 & 0.1489
& 0.2950 & 0.0480 & 0.1576 \\
Verbalized confidence+Iso
& 0.1684 & \textbf{0.0382} & \underline{0.1276}
& 0.1764 & \textbf{0.0389} & 0.1346 \\
P(True)
& 0.1794 & 0.1306 & 0.2272
& 0.1858 & 0.1317 & 0.2306 \\
P(True)+Iso
& \underline{0.1663} & 0.0410 & 0.1283
& \underline{0.1730} & 0.0416 & \underline{0.1336} \\
Probe
& \textbf{0.1186} & \underline{0.0383} & \textbf{0.1036}
& \textbf{0.1318} & \underline{0.0390} & \textbf{0.1114} \\

\midrule
\multicolumn{1}{l|}{\textbf{Qwen3-8B-non-thinking}}
& \multicolumn{3}{c|}{\textbf{Temp. 0.35}}
& \multicolumn{3}{c}{\textbf{Temp. 0.175}} \\
\textbf{Method}
& \textbf{TriviaQA} & \textbf{GSM8K} & \textbf{MMLU}
& \textbf{TriviaQA} & \textbf{GSM8K} & \textbf{MMLU} \\
\midrule

Uniform random baseline
& 0.3015 & 0.3133 & 0.2909
& 0.3085 & 0.3201 & 0.2930 \\
Verbalized confidence
& 0.2573 & 0.0531 & 0.1210
& 0.2607 & 0.0503 & 0.1260 \\
Verbalized confidence+Iso
& 0.1761 & 0.0451 & 0.1078
& 0.1808 & 0.0460 & 0.1106 \\
P(True)
& 0.3278 & 0.0515 & 0.1557
& 0.3326 & 0.0531 & 0.1585 \\
P(True)+Iso
& \underline{0.1540} & \underline{0.0428} & \underline{0.1051}
& \underline{0.1586} & \underline{0.0447} & \underline{0.1078} \\
Probe
& \textbf{0.1273} & \textbf{0.0412} & \textbf{0.0993}
& \textbf{0.1350} & \textbf{0.0429} & \textbf{0.1044} \\

\midrule
\multicolumn{1}{l|}{\textbf{gpt-oss-20b}}
& \multicolumn{3}{c|}{\textbf{Temp. 0.5}}
& \multicolumn{3}{c}{\textbf{Temp. 0.25}} \\
\textbf{Method}
& \textbf{TriviaQA} & \textbf{GSM8K} & \textbf{MMLU}
& \textbf{TriviaQA} & \textbf{GSM8K} & \textbf{MMLU} \\
\midrule

Uniform random baseline
& 0.2720 & 0.3252 & 0.2917
& 0.2887 & 0.3244 & 0.3064 \\
Verbalized confidence
& 0.1634 & \underline{0.0261} & \underline{0.0901}
& 0.1860 & \underline{0.0283} & \underline{0.0945} \\
Verbalized confidence+Iso
& \textbf{0.0816} & \textbf{0.0259} & \textbf{0.0801}
& \textbf{0.0984} & \textbf{0.0270} & \textbf{0.0847} \\
P(True)
& 0.2653 & 0.0294 & 0.1249
& 0.2713 & 0.0332 & 0.1278 \\
P(True)+Iso
& 0.1493 & 0.0271 & 0.0963
& 0.1611 & 0.0295 & 0.0981 \\
Probe
& \underline{0.1143} & 0.0331 & 0.0946
& \underline{0.1208} & 0.0365 & 0.1013 \\

\bottomrule
\end{tabular}
\end{table}

\paragraph{Conclusion.}
In summary, both the capability calibration target and confidence estimator evaluations are stable across decoding temperatures. These results directly demonstrate the transferability of capability calibration across decoding temperatures.

\subsection{LLMs' mean expected accuracies across datasets}\label{sec: llm mean expected acc}

\cref{tab:expected_acc_of_models} shows the dataset-level expected accuracy of every model-dataset pair in this paper.

\begin{table*}[ht!]
    \centering
    \small
    \caption{Three LLMs' mean expected accuracies in seven datasets.}
    \label{tab:expected_acc_of_models}
    \begin{tabularx}{\textwidth}{@{} l Y Y Y Y Y Y Y @{}}
        \toprule
        \textbf{Model / Dataset} & 
        \textbf{TriviaQA} & 
        \textbf{SimpleQA} & 
        \textbf{GSM8K} &
        \textbf{MATH-500} &
        \textbf{AIME} &
        \textbf{MMLU} &
        \textbf{GPQA} \\
        \midrule
        
        \textbf{Olmo-3-7B-Instruct}     & 48.84\% & 1.92\% & 93.31\% & 90.06\% & 47.29\% & 70.06\% & 43.64\% \\
        
        \textbf{Qwen3-8B} 
        & 60.40\% & 3.10\% & 93.56\% & 82.96\% & 20.97\% & 79.40\% & 49.65\% \\
        
        \textbf{gpt-oss-20b} & 65.07\% & 3.90\% & 95.76\% & 95.66\% & 69.10\% & 84.44\% & 66.97\% \\
        
        \bottomrule
    \end{tabularx}
\end{table*}

\subsection{Compute resources for estimating expected accuracy} \label{app: compute resources}
To estimate the expected accuracy of the models using $k_{eval}=100$, all experiments are conducted on a single GPU (NVIDIA A40 or NVIDIA RTX Pro 6000). Each dataset takes at most 24 hours.

\section{Details of confidence estimation methods}\label{sec:confidence_estimation}

\subsection{Detailed descriptions of existing methods}\label{sec:detailed_related_works}
In this section, we discuss the details of existing response calibration confidence estimators and LLM uncertainty quantification methods.

\paragraph{Response calibration confidence estimators.} Training-free methods: verbalized confidence \citep{lin2022teaching,tian2023just} that prompts the model to state its certainty, and token probability methods \citep{kadavath2022language, manakul2023selfcheckgpt}.
Specifically, $P(\text{True})$ \citep{kadavath2022language} estimates confidence by measuring the probability assigned to confirmation tokens (e.g., "True"). 
These methods estimate confidence after the output is generated.

\paragraph{LLM Uncertainty Quantification  (UQ) methods.} We categorize LLM UQ methods into three kinds: token-based approaches, which derive confidence from token-level likelihoods or log probabilities \citep{kadavath2022language, duan2024shifting}; Sampling-based approaches, which analyze the consistency or entropy of multiple generated outputs \citep{wang2022self, kuhn2023semantic, cecere2025monte}; and methods leveraging the models' internal signals. This includes training models to explicitly output an "I Don't Know" token \citep{cohen2024don}, probing hidden states, or estimating query-level uncertainty via internal self-evaluation \citep{chen2025query}.

\subsection{Implementation details of confidence estimators} \label{sec: implmentation of confidence estimators}
\paragraph{Random baseline.}
The random baseline draws a prediction $s \sim \mathrm{Uniform}(0,1)$
independently of the target. For a single question $i$ with target
$\hat{\mu}_i$, the expected squared error is
\begin{align*}
\mathbb{E}_{s}\!\left[(s - \hat{\mu}_i)^2\right]
&= \int_{0}^{1} \left(s^2 - 2s\hat{\mu}_i + \hat{\mu}_i^2\right) ds \\
&= \left[ \tfrac{1}{3}s^3 - s^2\hat{\mu}_i + s\hat{\mu}_i^2 \right]_0^1 \\
&= \tfrac{1}{3} - \hat{\mu}_i + \hat{\mu}_i^2 .
\end{align*}
Averaging over the $N$ questions in the dataset and writing
$\bar{\mu} = \frac{1}{N}\sum_i \hat{\mu}_i$ and
$\sigma^2 = \frac{1}{N}\sum_i (\hat{\mu}_i - \bar{\mu})^2$
for the mean and variance of the targets, the expected loss of the
random baseline is
\begin{align*}
\frac{1}{N}\sum_{i=1}^{N}\left(\tfrac{1}{3} - \hat{\mu}_i + \hat{\mu}_i^2\right)
= \tfrac{1}{3} - \bar{\mu} + \bar{\mu}^2 + \sigma^2 ,
\end{align*}
using $\frac{1}{N}\sum_i \hat{\mu}_i^2 = \bar{\mu}^2 + \sigma^2$.

\paragraph{Verbalized confidence and P(True).} Prompts are provided in \cref{prompt: verbalized confidence} and \cref{prompt: ptrue}, respectively.

\begin{figure*}[ht]
\centering
\small
\begin{AIbox}{Verbalized confidence prompt}
Question: \{question\} \\

    How likely are you to answer the question correctly? You may refer to the following probabilities P: \\ 
    - 0.0-0.1: "Almost no chance" \\
    - 0.1-0.2: "Highly unlikely" \\
    - 0.2-0.3: "Chances are slight" \\
    - 0.3-0.4: "Unlikely" \\
    - 0.4-0.5: "Less than even" \\
    - 0.5-0.6: "Better than even" \\
    - 0.6-0.7: "Likely" \\
    - 0.7-0.8: "Very good chance" \\
    - 0.8-0.9: "Highly likely" \\
    - 0.9-1.0: "Almost certain" \\ 
    Reason about your uncertainty and confidence, and then provide a probability P between 0.0 and 1.0 in the format of $\backslash$boxed\{P\}.
\end{AIbox}
\caption{The prompt for verbalized confidence.}
\label{prompt: verbalized confidence}
\end{figure*}

\begin{figure*}[ht]
\centering
\small
\begin{AIbox}{P(True) prompt}
Question: \{question\} \\ 

Are you able to answer the question correctly? \\
    Answer with only a single word: Yes or No.
\end{AIbox}
\caption{The prompt for P(True).}
\label{prompt: ptrue}
\end{figure*}

\paragraph{Training linear probes on LLM's hidden states.}
\label{app:probing}

We used layer activations after the initial embedding layer and after each transformer block, so there are $\ell + 1$ layers of activations used, with $\ell$ being the number of transformer blocks in the LLM.
In our preliminary experiments, we trained a linear probe on activations of each layer, as well as on max-pooled or mean-pooled activations.
We found that mean-pooled activations performed the best on the validation sets of each dataset, so we train linear probes on the mean-pooled activations in our main experiments (results listed in \S\ref{sec:results}).
We also searched for other hyperparameters, including number of epochs = $\{100,200,500,1000\}$, batch sizes = $\{32,64,128,256\}$, weight decays = $\{0.1,0.01,0.001\}$, loss functions = $\{\mathrm{BCE}, \mathrm{MSE}\}$, input feature standardization = $\{\mathrm{False},\mathrm{True}\}$, and learning rates = $\{\mathrm{1e-2,5e-3,2e-3,1e-3,5e-4,2e-4,1e-4,5e-5,2e-5,1e-5,5e-6,2e-6,1e-6}\}$ on the validation sets.
The final chosen hyperparameters are listed in Table~\ref{tab:hyperparameters}.
We also tried training 2-layer MLP probes, but found that their performance did not differ from linear probes.

\paragraph{Train-test splits for training-based methods.}
For verbalized confidence + isotonic regression, P(True) + isotonic regression, or linear probe training, we adopt the same train-test splits as follows:
\begin{itemize}
    \item TriviaQA~\cite{joshi2017triviaqa}: We sample 1,000 instances from the validation set as the test split, and sample 7,473 instances from the training set as the train split.
    \item SimpleQA~\cite{wei2024measuring, haas2025simpleqa}: We use the 1,000 instances from SimpleQA verified~\cite{haas2025simpleqa} as the test split, and the non-overlapping instances in the original SimpleQA~\cite{wei2024measuring} as the train split.
    \item GSM8K~\cite{cobbe2021training}: We use the 7,473 instances in the training set as the training split, and the 1,319 instances in the test set as the test split.
    \item MATH~\cite{hendrycksmath2021}: We use the 500 instances in MATH-500~\cite{lightman2023let} as the test split, and the remaining 12,000 instances from the original MATH dataset~\cite{hendrycksmath2021} as the training split.
    \item AIME~\footnote{\url{https://huggingface.co/datasets/gneubig/aime-1983-2024}}~\footnote{\url{https://huggingface.co/datasets/math-ai/aime25}}~\footnote{\url{https://huggingface.co/datasets/MathArena/aime_2026}}: We use the 918 instances from AIME 1983-2023 as the training split, and the 90 instances from AIME 2024-2026 as the test split.
    \item MMLU~\cite{hendrycks2020measuring}: We use all the 1,531 instances in the validation set as the training split, and sample 1,000 instances in the test set as the test split due to budget consideration.
    \item GPQA~\cite{rein2023gpqa}: We use the 198 instances in the diamond subset as the test split, and the non-overlapping 348 instances in the extended subset as the train split. 
\end{itemize}

\begin{table*}[ht!]
    \centering
    \small
    \caption{Hyperparameters or model information for linear probes trained on three LLMs.}
    \label{tab:hyperparameters}
    \begin{tabularx}{\textwidth}{@{} l Y Y Y @{}}
        \toprule
        \textbf{Hyperparameter or Model Info.} & 
        \textbf{Olmo-3-7B-Ins.} & 
        \textbf{Qwen3-8B} & 
        \textbf{gpt-oss-20b} \\
        \midrule
        
        Number of layer activations used & 
        33 & 
        37 & 
        25 \\
        
        Hidden dimension & 
        4096 & 
        4096 & 
        2880 \\
        
        Epochs & 
        100 & 
        100 & 
        100 \\
        
        Batch size & 
        32 & 
        32 & 
        32 \\
        
        Weight decay & 
        0.01 & 
        0.01 & 
        0.01 \\
        
        Pooling method & 
        Mean pooling & 
        Mean pooling & 
        Mean pooling \\
        
        Loss function & 
        BCE loss & 
        BCE loss & 
        BCE loss \\
        
        Feature standardization & 
        False & 
        False & 
        True \\
        
        Learning rate (TriviaQA) & 
        $5 \times 10^{-3}$ & 
        $2 \times 10^{-4}$ & 
        $2 \times 10^{-4}$ \\

        Learning rate (SimpleQA) & 
        $1 \times 10^{-2}$ & 
        $2 \times 10^{-4}$ & 
        $1 \times 10^{-2}$ \\
        
        Learning rate (GSM8K) & 
        $1 \times 10^{-2}$ & 
        $1 \times 10^{-5}$ & 
        $1 \times 10^{-3}$ \\
        
        Learning rate (MATH) & 
        $1 \times 10^{-2}$ & 
        $1 \times 10^{-4}$ & 
        $2 \times 10^{-3}$ \\

        Learning rate (AIME) & 
        $5 \times 10^{-3}$ & 
        $5 \times 10^{-5}$ & 
        $2 \times 10^{-4}$ \\

        Learning rate (MMLU) & 
        $1 \times 10^{-2}$ & 
        $2 \times 10^{-5}$ & 
        $2 \times 10^{-3}$ \\

        Learning rate (GPQA) & 
        $5 \times 10^{-3}$ & 
        $2 \times 10^{-5}$ & 
        $1 \times 10^{-4}$ \\
        
        \bottomrule
    \end{tabularx}
\end{table*}

\clearpage

\subsection{Analysis of \texorpdfstring{$k_c$}{kc} for Response Consistency method} \label{sec: response consistency additional results}
In this section, we analyze how the sample size $k_c$ impacts the \textbf{Response Consistency} confidence estimator. As shown in \cref{tab: consistency abl}, a larger $k_c$ generally leads to improved Brier scores by reducing estimation variance. However, this performance gain trades off with a higher estimation cost. We observe performance saturation because increasing $k_c$ cannot correct for fundamental miscalibration, which is the main limitation of this confidence estimator. If the model's most frequent response is incorrect, the estimator converges to a confidence score for a failure case, preventing further reduction in Brier score.
\begin{table}[h]
\centering
\small
\caption{\textbf{Capability calibration Brier scores of Response Consistency across different numbers of samples} $k_c$. We use \textbf{bold} to denote the \textbf{best} calibrated method, and \underline{underline} to denote the \underline{second best}. Larger $k_c$ generally improves performance at the cost of higher estimation overhead.}
\label{tab: consistency abl}
\resizebox{\textwidth}{!}{
\begin{tabular}{l c | c c c c c c c}
\toprule
& \textbf{Domain} & \multicolumn{2}{c}{\textbf{Factual knowledge}} & \multicolumn{3}{c}{\textbf{Mathematical reasoning}} & \multicolumn{2}{c}{\textbf{General exams}} \\
\cmidrule(lr){3-4} \cmidrule(lr){5-7} \cmidrule(lr){8-9}
\textbf{Method} & \textbf{Cost} & \textbf{TriviaQA} & \textbf{SimpleQA} & \textbf{GSM8K} & \textbf{\mbox{MATH}} & \textbf{AIME} & \textbf{MMLU} & \textbf{GPQA} \\
\midrule
\multicolumn{2}{l|}{\textbf{Olmo-3-7B-Instruct}} & \multicolumn{7}{c}{} \\
\hspace{3mm}Consistency ($k=5$) & 5$L$ & 0.1228 & 0.1605 & 0.0290 & \underline{0.0373} & 0.1220 & 0.1220 & 0.2462 \\
\hspace{3mm}Consistency ($k=10$) & 10$L$ & \underline{0.1121} & \underline{0.1297} & \underline{0.0272} & 0.0383 & \underline{0.1097} & \underline{0.1103} & \underline{0.2246} \\
\hspace{3mm}Consistency ($k=20$) & 20$L$ & \textbf{0.1046} & \textbf{0.1141} & \textbf{0.0264} & \textbf{0.0350} & \textbf{0.1048} & \textbf{0.1030} & \textbf{0.2079} \\
\multicolumn{2}{l|}{\textbf{Qwen3-8B}} & \multicolumn{7}{c}{} \\
\hspace{3mm}Consistency ($k=5$) & 5$L$ & 0.1433 & 0.2016 & 0.0255 & 0.0336 & 0.1285 & 0.0993 & 0.1830 \\
\hspace{3mm}Consistency ($k=10$) & 10$L$ & \underline{0.1301} & \underline{0.1657} & \underline{0.0241} & \underline{0.0304} & \underline{0.1282} & \underline{0.0922} & \underline{0.1656} \\
\hspace{3mm}Consistency ($k=20$) & 20$L$ & \textbf{0.1239} & \textbf{0.1521} & \textbf{0.0237} & \textbf{0.0276} & \textbf{0.1069} & \textbf{0.0892} & \textbf{0.1485} \\
\multicolumn{2}{l|}{\textbf{gpt-oss-20b}} & \multicolumn{7}{c}{} \\
\hspace{3mm}Consistency ($k=5$) & 5$L$ & \textbf{0.1274} & 0.0912 & 0.0210 & \textbf{0.0553} & 0.0536 & 0.0728 & 0.1275 \\
\hspace{3mm}Consistency ($k=10$) & 10$L$ & \underline{0.1382} & \underline{0.0624} & \underline{0.0202} & \underline{0.0554} & \textbf{0.0391} & \underline{0.0673} & \underline{0.1066} \\
\hspace{3mm}Consistency ($k=20$) & 20$L$ & 0.1406 & \textbf{0.0520} & \textbf{0.0188} & 0.0593 & \underline{0.0402} & \textbf{0.0629} & \textbf{0.1007} \\
\bottomrule
\end{tabular}
}
\end{table}

\subsection{Probe performance in out-of-distribution data} \label{sec: probe in ood}
\begin{table}[t!]
\centering
\small
\caption{\textbf{Capability calibration performance (Brier score) of linear probes trained with different training sets.} For probes, we use different colors to indicate \colorbox{LGreen!25}{in-domain in-distribution}, \colorbox{LBlue!25}{in-domain out-of-distribution}, and \colorbox{LOrange!25}{out-domain} performance. We use \textbf{bold} to denote the \textbf{best} calibrated method, and \underline{underline} to denote the \underline{second best}. 
Probe performs best under \colorbox{LGreen!25}{in-domain in-distribution} settings and generalizes reasonably well under \colorbox{LBlue!25}{in-domain out-distribution} settings.}
\label{tab:main-results-old}
\resizebox{\textwidth}{!}{
\begin{tabular}{l c | c c c c c c c}
\toprule
\textbf{Brier score (↓)} & \textbf{Domain} & \multicolumn{2}{c}{\textbf{Factual knowledge}} & \multicolumn{3}{c}{\textbf{Mathematical reasoning}} & \multicolumn{2}{c}{\textbf{General exams}} \\
\cmidrule(lr){3-4} \cmidrule(lr){5-7} \cmidrule(lr){8-9}
\textbf{Method} & \textbf{Cost} & \textbf{TriviaQA} & \textbf{SimpleQA} & \textbf{GSM8K} & \textbf{\mbox{MATH}} & \textbf{AIME} & \textbf{MMLU} & \textbf{GPQA} \\
\midrule
\multicolumn{2}{l|}{\textbf{Olmo-3-7B-Instruct}} & \multicolumn{7}{c}{} \\
\hspace{3mm}Probe (train on TriviaQA) & $<1$ & \cellcolor{LGreen!25}\textbf{0.1101} & \cellcolor{LBlue!25}\underline{0.0397} & \cellcolor{LOrange!25}0.1377 & \cellcolor{LOrange!25}0.1547 & \cellcolor{LOrange!25}0.1315 & \cellcolor{LOrange!25}0.1209 & \cellcolor{LOrange!25}0.1230 \\
\hspace{3mm}Probe (train on SimpleQA) & $<1$ & \cellcolor{LBlue!25}0.3337 & \cellcolor{LGreen!25}\textbf{0.0130} & \cellcolor{LOrange!25}0.5342 & \cellcolor{LOrange!25}0.4510 & \cellcolor{LOrange!25}0.1894 & \cellcolor{LOrange!25}0.3187 & \cellcolor{LOrange!25}0.1557 \\
\hspace{3mm}Probe (train on GSM8K) & $<1$ & \cellcolor{LOrange!25}0.2924 & \cellcolor{LOrange!25}0.6038 & \cellcolor{LGreen!25}\textbf{0.0364} & \cellcolor{LBlue!25}0.0496 & \cellcolor{LBlue!25}0.1622 & \cellcolor{LOrange!25}\underline{0.1142} & \cellcolor{LOrange!25}0.1635 \\
\hspace{3mm}Probe (train on MATH) & $<1$ & \cellcolor{LOrange!25}0.2628 & \cellcolor{LOrange!25}0.4810 & \cellcolor{LBlue!25}\underline{0.0382} & \cellcolor{LGreen!25}\textbf{0.0373} & \cellcolor{LBlue!25}\underline{0.1195} & \cellcolor{LOrange!25}0.1271 & \cellcolor{LOrange!25}0.1255 \\
\hspace{3mm}Probe (train on AIME) & $<1$ & \cellcolor{LOrange!25}0.2015 & \cellcolor{LOrange!25}0.3686 & \cellcolor{LBlue!25}0.0382 & \cellcolor{LBlue!25}\underline{0.0443} & \cellcolor{LGreen!25}\textbf{0.1092} & \cellcolor{LOrange!25}0.1170 & \cellcolor{LOrange!25}0.1167 \\
\hspace{3mm}Probe (train on MMLU) & $<1$ & \cellcolor{LOrange!25}0.2058 & \cellcolor{LOrange!25}0.2349 & \cellcolor{LOrange!25}0.0383 & \cellcolor{LOrange!25}0.0527 & \cellcolor{LOrange!25}0.2755 & \cellcolor{LGreen!25}\textbf{0.1004} & \cellcolor{LBlue!25}\underline{0.1149} \\
\hspace{3mm}Probe (train on GPQA) & $<1$ & \cellcolor{LOrange!25}\underline{0.1935} & \cellcolor{LOrange!25}0.1587 & \cellcolor{LOrange!25}0.3428 & \cellcolor{LOrange!25}0.1726 & \cellcolor{LOrange!25}0.1883 & \cellcolor{LBlue!25}0.2492 & \cellcolor{LGreen!25}\textbf{0.1072} \\
\midrule
\multicolumn{2}{l|}{\textbf{Qwen3-8B}} & \multicolumn{7}{c}{}  \\
\hspace{3mm}Probe (train on TriviaQA) & $<1$ & \cellcolor{LGreen!25}\textbf{0.1163} & \cellcolor{LBlue!25}\underline{0.0591} & \cellcolor{LOrange!25}0.2411 & \cellcolor{LOrange!25}0.1269 & \cellcolor{LOrange!25}0.1075 & \cellcolor{LOrange!25}0.1256 & \cellcolor{LOrange!25}\textbf{0.1366} \\
\hspace{3mm}Probe (train on SimpleQA) & $<1$ & \cellcolor{LBlue!25}0.4255 & \cellcolor{LGreen!25}\textbf{0.0306} & \cellcolor{LOrange!25}0.4370 & \cellcolor{LOrange!25}0.2226 & \cellcolor{LOrange!25}0.4496 & \cellcolor{LOrange!25}0.3955 & \cellcolor{LOrange!25}0.1619 \\
\hspace{3mm}Probe (train on GSM8K) & $<1$ & \cellcolor{LOrange!25}0.2164 & \cellcolor{LOrange!25}0.5177 & \cellcolor{LGreen!25}\textbf{0.0377} & \cellcolor{LBlue!25}\underline{0.0678} & \cellcolor{LBlue!25}0.1570 & \cellcolor{LOrange!25}\underline{0.1121} & \cellcolor{LOrange!25}0.1602 \\
\hspace{3mm}Probe (train on MATH) & $<1$ & \cellcolor{LOrange!25}0.3235 & \cellcolor{LOrange!25}0.8566 & \cellcolor{LBlue!25}0.0411 & \cellcolor{LGreen!25}\textbf{0.0471} & \cellcolor{LBlue!25}\underline{0.0774} & \cellcolor{LOrange!25}0.1358 & \cellcolor{LOrange!25}0.2383 \\
\hspace{3mm}Probe (train on AIME) & $<1$ & \cellcolor{LOrange!25}0.2210 & \cellcolor{LOrange!25}0.4383 & \cellcolor{LBlue!25}0.0993 & \cellcolor{LBlue!25}0.1216 & \cellcolor{LGreen!25}\textbf{0.0578} & \cellcolor{LOrange!25}0.1527 & \cellcolor{LOrange!25}0.1607 \\
\hspace{3mm}Probe (train on MMLU) & $<1$ & \cellcolor{LOrange!25}0.2638 & \cellcolor{LOrange!25}0.6218 & \cellcolor{LOrange!25}\underline{0.0402} & \cellcolor{LOrange!25}0.0697 & \cellcolor{LOrange!25}0.2971 & \cellcolor{LGreen!25}\textbf{0.0987} & \cellcolor{LBlue!25}\underline{0.1570} \\
\hspace{3mm}Probe (train on GPQA) & $<1$ & \cellcolor{LOrange!25}\underline{0.2117} & \cellcolor{LOrange!25}0.2707 & \cellcolor{LOrange!25}0.3009 & \cellcolor{LOrange!25}0.4568 & \cellcolor{LOrange!25}0.1207 & \cellcolor{LBlue!25}0.1746 & \cellcolor{LGreen!25}0.1571 \\
\midrule
\multicolumn{2}{l|}{\textbf{gpt-oss-20b}} & \multicolumn{7}{c}{}   \\
\hspace{3mm}Probe (train on TriviaQA) & $<1$ & \cellcolor{LGreen!25}\textbf{0.1027} & \cellcolor{LBlue!25}\underline{0.0894} & \cellcolor{LOrange!25}\underline{0.0273} & \cellcolor{LOrange!25}0.0284 & \cellcolor{LOrange!25}0.1508 & \cellcolor{LOrange!25}0.1071 & \cellcolor{LOrange!25}0.1514 \\
\hspace{3mm}Probe (train on SimpleQA) & $<1$ & \cellcolor{LBlue!25}0.5694 & \cellcolor{LGreen!25}\textbf{0.0264} & \cellcolor{LOrange!25}0.8525 & \cellcolor{LOrange!25}0.7829 & \cellcolor{LOrange!25}0.3848 & \cellcolor{LOrange!25}0.7277 & \cellcolor{LOrange!25}0.5702 \\
\hspace{3mm}Probe (train on GSM8K) & $<1$ & \cellcolor{LOrange!25}0.2574 & \cellcolor{LOrange!25}0.8625 & \cellcolor{LGreen!25}0.0291 & \cellcolor{LBlue!25}0.0395 & \cellcolor{LBlue!25}\underline{0.1252} & \cellcolor{LOrange!25}0.1329 & \cellcolor{LOrange!25}0.2273 \\
\hspace{3mm}Probe (train on MATH) & $<1$ & \cellcolor{LOrange!25}0.2274 & \cellcolor{LOrange!25}0.8025 & \cellcolor{LBlue!25}0.0317 & \cellcolor{LGreen!25}\textbf{0.0216} & \cellcolor{LBlue!25}0.2530 & \cellcolor{LOrange!25}0.1037 & \cellcolor{LOrange!25}0.1541 \\
\hspace{3mm}Probe (train on AIME) & $<1$ & \cellcolor{LOrange!25}0.2385 & \cellcolor{LOrange!25}0.7793 & \cellcolor{LBlue!25}\textbf{0.0272} & \cellcolor{LBlue!25}0.0319 & \cellcolor{LGreen!25}\textbf{0.0966} & \cellcolor{LOrange!25}0.0985 & \cellcolor{LOrange!25}0.1515 \\
\hspace{3mm}Probe (train on MMLU) & $<1$ & \cellcolor{LOrange!25}\underline{0.1739} & \cellcolor{LOrange!25}0.2729 & \cellcolor{LOrange!25}0.0278 & \cellcolor{LOrange!25}\underline{0.0224} & \cellcolor{LOrange!25}0.1889 & \cellcolor{LGreen!25}\textbf{0.0878} & \cellcolor{LBlue!25}\underline{0.1487} \\
\hspace{3mm}Probe (train on GPQA) & $<1$ & \cellcolor{LOrange!25}0.2218 & \cellcolor{LOrange!25}0.7578 & \cellcolor{LOrange!25}0.0327 & \cellcolor{LOrange!25}0.0378 & \cellcolor{LOrange!25}0.1305 & \cellcolor{LBlue!25}\underline{0.0954} & \cellcolor{LGreen!25}\textbf{0.1462} \\
\bottomrule
\end{tabular}%
}
\end{table}

\cref{tab:main-results-old} provides the performance of Probe in out-of-distribution and out-of-domain scenario. It shows that probing performs well under in-domain, in-distribution settings.
However, some applications may require applying a confidence estimator to
(1) \textit{same-domain but out-of-distribution} queries (e.g., different factual knowledge datasets), or
(2) \textit{out-of-domain} queries (e.g., training on factual knowledge but applying to mathematical reasoning).
Overall, probing generalizes reasonably well under in-domain, out-of-distribution settings, especially for the mathematical reasoning domain.
However, it does not consistently generalize to out-of-domain settings. Existing works \citep{cencerrado2025no,zhang2025reasoning} also find that probes struggles with out-of-domain settings for response calibration. Since this paper is focused on the calibration (evaluation) framework of confidence estimator instead of confidence estimator itself, the out-of-domain generalization of probes is out-of-scope for this paper.
Developing generalizable methods for capability calibration remains an important direction for future work.

\begin{table}[ht!]
\centering
\small
\caption{\textbf{Capability calibration Brier scores of linear probes trained on single and mixed datasets.} Results reported by Brier scores (↓). For linear probes trained with different datasets, we use different colors to indicate \colorbox{LGreen!25}{in-domain in-distribution}, \colorbox{LBlue!25}{in-domain out-of-distribution}, and \colorbox{LOrange!25}{out-domain} performance. We use \textbf{bold} to denote the \textbf{best} calibrated method, and \underline{underline} to denote the \underline{second best}.
Results show that training probes on a mixture of datasets (TriviaQA + GSM8K) generally yields the most robust calibration across both in-distribution and in-domain OOD tasks (e.g., MATH, SimpleQA). However, this benefit is less consistent for out-domain datasets (MMLU, GPQA), where specialized single-dataset probes occasionally maintain an edge.}
\label{tab: mix train data}
\resizebox{\textwidth}{!}{
\begin{tabular}{l c | c c c c c c c}
\toprule
& \textbf{Domain} & \multicolumn{2}{c}{\textbf{Factual knowledge}} & \multicolumn{3}{c}{\textbf{Mathematical reasoning}} & \multicolumn{2}{c}{\textbf{General exams}} \\
\cmidrule(lr){3-4} \cmidrule(lr){5-7} \cmidrule(lr){8-9}
\textbf{Method} & \textbf{Cost} & \textbf{TriviaQA} & \textbf{SimpleQA} & \textbf{GSM8K} & \textbf{\mbox{MATH}} & \textbf{AIME} & \textbf{MMLU} & \textbf{GPQA} \\
\midrule
\multicolumn{2}{l|}{\textbf{Olmo-3-7B-Instruct}} & \multicolumn{7}{c}{} \\
\hspace{3mm}Probe (train on TriviaQA) & $<1$ & \cellcolor{LGreen!25}\textbf{0.1101} & \cellcolor{LBlue!25}\underline{0.0397} & \cellcolor{LOrange!25}0.1377 & \cellcolor{LOrange!25}0.1547 & \cellcolor{LOrange!25}\underline{0.1315} & \cellcolor{LOrange!25}0.1209 & \cellcolor{LOrange!25}\underline{0.1230} \\
\hspace{3mm}Probe (train on GSM8K) & $<1$ & \cellcolor{LOrange!25}0.2924 & \cellcolor{LOrange!25}0.6038 & \cellcolor{LGreen!25}\textbf{0.0364} & \cellcolor{LBlue!25} \textbf{0.0496} & \cellcolor{LBlue!25}0.1622 & \cellcolor{LOrange!25}\textbf{0.1142} & \cellcolor{LOrange!25}0.1635 \\
\hspace{3mm}Probe (TriviaQA + GSM8K) & $<1$ & \cellcolor{LGreen!25}\underline{0.1147} & \cellcolor{LBlue!25}\textbf{0.0396} & \cellcolor{LGreen!25}\underline{0.0371} & \cellcolor{LBlue!25}\underline{0.0545} & \cellcolor{LBlue!25}\textbf{0.1127} & \cellcolor{LOrange!25}\underline{0.1298} & \cellcolor{LOrange!25}\textbf{0.1141} \\
\midrule
\multicolumn{2}{l|}{\textbf{Qwen3-8B}} & \multicolumn{7}{c}{} \\
\hspace{3mm}Probe (train on TriviaQA) & $<1$ & \cellcolor{LGreen!25}\underline{0.1163} & \cellcolor{LBlue!25}\textbf{0.0591} & \cellcolor{LOrange!25}0.2411 & \cellcolor{LOrange!25}\underline{0.1269} & \cellcolor{LOrange!25} \underline{0.1075} & \cellcolor{LOrange!25}\underline{0.1256} & \cellcolor{LOrange!25}\textbf{0.1366} \\
\hspace{3mm}Probe (train on GSM8K) & $<1$ & \cellcolor{LOrange!25}0.2164 & \cellcolor{LOrange!25}0.5177 & \cellcolor{LGreen!25}\textbf{0.0377} & \cellcolor{LBlue!25}\textbf{0.0678} & \cellcolor{LBlue!25}0.1570 & \cellcolor{LOrange!25}\textbf{0.1121} & \cellcolor{LOrange!25}\underline{0.1602} \\
\hspace{3mm}Probe (TriviaQA + GSM8K) & $<1$ & \cellcolor{LGreen!25}\textbf{0.1079} & \cellcolor{LBlue!25}\underline{0.0607} & \cellcolor{LGreen!25}\underline{0.0408} & \cellcolor{LBlue!25}0.1378 & \cellcolor{LBlue!25}\textbf{0.0952} & \cellcolor{LOrange!25}0.1422 & \cellcolor{LOrange!25}0.1683 \\
\midrule
\multicolumn{2}{l|}{\textbf{gpt-oss-20b}} & \multicolumn{7}{c}{} \\
\hspace{3mm}Probe (train on TriviaQA) & $<1$ & \cellcolor{LGreen!25}\underline{0.1027} & \cellcolor{LBlue!25}\underline{0.0894} & \cellcolor{LOrange!25}\textbf{0.0273} & \cellcolor{LOrange!25}\textbf{0.0284} & \cellcolor{LOrange!25}0.1508 & \cellcolor{LOrange!25}\underline{0.1071} & \cellcolor{LOrange!25}\underline{0.1514} \\
\hspace{3mm}Probe (train on GSM8K) & $<1$ & \cellcolor{LOrange!25}0.2574 & \cellcolor{LOrange!25}0.8625 & \cellcolor{LGreen!25}0.0291 & \cellcolor{LBlue!25}\underline{0.0395} & \cellcolor{LBlue!25}\underline{0.1252} & \cellcolor{LOrange!25}0.1329 & \cellcolor{LOrange!25}0.2273 \\
\hspace{3mm}Probe (TriviaQA + GSM8K) & $<1$ & \cellcolor{LGreen!25}\textbf{0.0902} & \cellcolor{LBlue!25}\textbf{0.0707} & \cellcolor{LGreen!25}\underline{0.0277} & \cellcolor{LBlue!25}0.0554 & \cellcolor{LBlue!25}\textbf{0.1181} & \cellcolor{LOrange!25}\textbf{0.0818} & \cellcolor{LOrange!25}\textbf{0.1256} \\
\bottomrule
\end{tabular}
}
\end{table}

\subsubsection{Mixing training dataset from different domains} \label{sec: mix of train data}
In this section, we discuss the effect of mixing training datasets from different domains. The experiment results are in \cref{tab: mix train data}. While single-dataset probes typically perform best on their specific in-distribution tasks, the mixed probe (TriviaQA + GSM8K) often achieves the best or second-best Brier scores across all in-domain out-of-distribution categories, such as MATH and SimpleQA. This suggests that data diversity improves the probe's ability to generalize to in-domain tasks. However, this trend does not hold in out-domain settings like MMLU and GPQA, where the performance of mixed versus single probes varies by model, indicating that mixing training data does not guarantee improved calibration for entirely unrelated domains.

\clearpage

\subsection{Evaluating confidence estimators in ranking metric} \label{sec:spearman}

Ranking metrics are valuable for tasks needing query-level discrimination, which a constant predictor lacks. Since the target of CC is a continuous value in $[0,1]$, a suitable ranking metric would be Spearman’s correlation, which measures whether the confidence scores preserve the ranking of examples by their correctness probability. We evaluated the Spearman’s r  for existing confidence estimation methods in \cref{tab:main-results-spearman}, whose findings are consistent with the Brier score \cref{sec:results}: Verbalized confidence achieves the best performance for gpt-oss-20b, while Probe is generally the most effective for the other two models.

\begin{table}[h!]
\centering
\small
\caption{\textbf{Capability calibration performance (Spearman's $\rho$) of different methods with three LLMs on seven datasets.} We use \textbf{bold} to denote the \textbf{best} calibrated method, and \underline{underline} to denote the \underline{second best}. The findings are consistent with the Brier score result in \cref{tab:main-results}.
}
\label{tab:main-results-spearman}
\resizebox{\textwidth}{!}{
\begin{tabular}{l c | c c c c c c c}
\toprule

\textbf{Spearman's $\rho$ (↑)} & \textbf{Domain} & \multicolumn{2}{c}{\textbf{Factual knowledge}} & \multicolumn{3}{c}{\textbf{Mathematical reasoning}} & \multicolumn{2}{c}{\textbf{General exams}} \\
\cmidrule(lr){3-4} \cmidrule(lr){5-7} \cmidrule(lr){8-9}
\textbf{Method} & \textbf{Cost} & \textbf{TriviaQA} & \textbf{SimpleQA} & \textbf{GSM8K} & \textbf{\mbox{MATH}} & \textbf{AIME} & \textbf{MMLU} & \textbf{GPQA} \\
\midrule
\multicolumn{2}{l|}{\textbf{Olmo-3-7B-Instruct}} & \multicolumn{7}{c}{} \\
\hspace{3mm}Uniform random baseline & N/A & 0.0000 & 0.0000 & 0.0000 & 0.0000 & 0.0000 & 0.0000 & 0.0000 \\
\hspace{3mm}Verbalized confidence & $L$ & 0.3424 & 0.0982 & 0.1293 & 0.1476 & \underline{0.4013} & 0.1709 & 0.1552 \\
\hspace{6mm}+ isotonic regression & $L$ & 0.3443 & 0.0928 & 0.1325 & 0.1576 & 0.3593 & 0.1696 & \underline{0.1694} \\
\hspace{3mm}P(True) & 1 & 0.3944 & \underline{0.1336} & 0.2474 & 0.1862 & 0.1123 & 0.3163 & 0.0211 \\
\hspace{6mm}+ isotonic regression & 1 & \underline{0.3951} & 0.1256 & \underline{0.2499} & \underline{0.1889} & 0.0993 & \underline{0.3178} & 0.0565 \\
\hspace{3mm}Probe & $<1$ & \textbf{0.6418} & \textbf{0.1576} & \textbf{0.4551} & \textbf{0.6493} & \textbf{0.5885} & \textbf{0.5573} & \textbf{0.4001} \\
\midrule
\multicolumn{2}{l|}{\textbf{Qwen3-8B}} & \multicolumn{7}{c}{} \\
\hspace{3mm}Uniform random baseline & N/A & 0.0000 & 0.0000 & 0.0000 & 0.0000 & 0.0000 & 0.0000 & 0.0000 \\
\hspace{3mm}Verbalized confidence & $L$ & 0.3197 & 0.0225 & 0.0606 & 0.1954 & 0.1297 & 0.3075 & 0.0829 \\
\hspace{6mm}+ isotonic regression & $L$ & 0.3192 & \textbf{0.0352} & 0.0636 & 0.1891 & 0.1682 & 0.3069 & 0.0811 \\
\hspace{3mm}P(True) & 1 & 0.4795 & -0.0556 & 0.3043 & 0.3989 & 0.4042 & \underline{0.3863} & \textbf{0.3261} \\
\hspace{6mm}+ isotonic regression & 1 & \underline{0.4876} & -0.0599 & \underline{0.3051} & \underline{0.4019} & \underline{0.4212} & 0.3840 & \underline{0.2949} \\
\hspace{3mm}Probe & $<1$ & \textbf{0.6411} & \underline{0.0320} & \textbf{0.4058} & \textbf{0.7058} & \textbf{0.5711} & \textbf{0.4681} & -0.0077 \\
\midrule
\multicolumn{2}{l|}{\textbf{gpt-oss-20b}} & \multicolumn{7}{c}{} \\
\hspace{3mm}Uniform random baseline & N/A & 0.0000 & 0.0000 & 0.0000 & 0.0000 & 0.0000 & 0.0000 & 0.0000 \\
\hspace{3mm}Verbalized confidence & $L$ & \underline{0.6674} & \textbf{0.3286} & \textbf{0.2526} & \underline{0.3565} & \underline{0.7237} & \underline{0.5059} & \underline{0.5828} \\
\hspace{6mm}+ isotonic regression & $L$ & \textbf{0.6680} & \underline{0.3254} & \underline{0.2506} & 0.3558 & \textbf{0.7254} & \textbf{0.5075} & \textbf{0.5912} \\
\hspace{3mm}P(True) & $L$ & 0.2625 & 0.0901 & 0.1885 & 0.0205 & 0.1975 & 0.1811 & 0.2937 \\
\hspace{6mm}+ isotonic regression & $L$ & 0.2608 & 0.1089 & 0.1859 & 0.0289 & 0.2135 & 0.1890 & 0.2904 \\
\hspace{3mm}Probe & $<1$ & 0.6283 & 0.1607 & 0.2283 & \textbf{0.5372} & 0.4854 & 0.2814 & 0.2088 \\
\bottomrule
\end{tabular}%
}
\end{table}

\clearpage
\subsection{Evaluating confidence estimators in ECE} \label{sec:ece}

In this section, we report the ECE score of the confidence estimators. As shown in \cref{tab:main-results-ece}, verbalized confidence after isotonic regression is competitive with the linear probe. \cref{fig:reliability_diagrams} further visualizes the ECE metric as reliability diagrams. 

\begin{table}[h!]
\centering
\small
\caption{\textbf{Capability calibration performance (ECE) of different methods with three LLMs on seven datasets.}
We use \textbf{bold} to denote the \textbf{best} calibrated method, and \underline{underline} to denote the \underline{second best}. The findings are consistent with the Brier score results in \cref{tab:main-results}.
}
\label{tab:main-results-ece}
\resizebox{\textwidth}{!}{
\begin{tabular}{l c | c c c c c c c}
\toprule

\textbf{ECE ($\downarrow$)} & \textbf{Domain}
& \multicolumn{2}{c}{\textbf{Factual knowledge}}
& \multicolumn{3}{c}{\textbf{Mathematical reasoning}}
& \multicolumn{2}{c}{\textbf{General exams}} \\
\cmidrule(lr){3-4} \cmidrule(lr){5-7} \cmidrule(lr){8-9}
\textbf{Method} & \textbf{Cost}
& \textbf{TriviaQA} & \textbf{SimpleQA}
& \textbf{GSM8K} & \textbf{\mbox{MATH}} & \textbf{AIME}
& \textbf{MMLU} & \textbf{GPQA} \\
\midrule

\multicolumn{2}{l|}{\textbf{Olmo-3-7B-Instruct}} & \multicolumn{7}{c}{} \\
\hspace{3mm}Uniform random baseline & N/A
& 0.2511 & 0.4667 & 0.4365 & 0.4110 & 0.2644 & 0.2905 & 0.2601 \\
\hspace{3mm}Verbalized confidence & $L$
& 0.3222 & 0.3803 & \underline{0.0143} & 0.0309 & 0.2913 & 0.1877 & 0.3714 \\
\hspace{6mm}+ isotonic regression & $L$
& \underline{0.0382} & 0.0141 & 0.0192 & \textbf{0.0066} & \underline{0.1597} & \textbf{0.0153} & \underline{0.0446} \\
\hspace{3mm}P(True) & 1
& 0.0985 & 0.0847 & 0.2483 & 0.2455 & 0.2062 & 0.3126 & 0.1540 \\
\hspace{6mm}+ isotonic regression & 1
& 0.1098 & \underline{0.0137} & 0.0186 & \underline{0.0131} & 0.1829 & 0.0235 & \textbf{0.0423} \\
\hspace{3mm}Probe & $<1$
& \textbf{0.0306} & \textbf{0.0095} & \textbf{0.0085} & 0.0218 & \textbf{0.1116} & \underline{0.0164} & 0.0652 \\

\midrule
\multicolumn{2}{l|}{\textbf{Qwen3-8B}} & \multicolumn{7}{c}{} \\
\hspace{3mm}Uniform random baseline & N/A
& 0.2615 & 0.4496 & 0.4386 & 0.3591 & 0.3440 & 0.3368 & 0.2539 \\
\hspace{3mm}Verbalized confidence & $L$
& 0.2758 & 0.5494 & 0.0181 & 0.0877 & 0.5210 & 0.1112 & 0.3817 \\
\hspace{6mm}+ isotonic regression & $L$
& \textbf{0.0275} & \textbf{0.0211} & 0.0171 & \underline{0.0148} & 0.1747 & \textbf{0.0183} & \textbf{0.0947} \\
\hspace{3mm}P(True) & 1
& 0.3782 & 0.6576 & 0.0663 & 0.1671 & 0.5917 & 0.1993 & 0.4657 \\
\hspace{6mm}+ isotonic regression & 1
& 0.0340 & \underline{0.0240} & \textbf{0.0035} & 0.0268 & \underline{0.1029} & 0.0297 & \underline{0.1038} \\
\hspace{3mm}Probe & $<1$
& \underline{0.0296} & 0.0399 & \underline{0.0080} & \textbf{0.0146} & \textbf{0.0532} & \underline{0.0255} & 0.1463 \\

\midrule
\multicolumn{2}{l|}{\textbf{gpt-oss-20b}} & \multicolumn{7}{c}{} \\
\hspace{3mm}Uniform random baseline & N/A
& 0.2731 & 0.4426 & 0.4579 & 0.4575 & 0.2950 & 0.3683 & 0.2829 \\
\hspace{3mm}Verbalized confidence & $L$
& 0.1691 & 0.3593 & 0.0118 & \textbf{0.0054} & 0.1086 & 0.0726 & 0.1264 \\
\hspace{6mm}+ isotonic regression & $L$
& \underline{0.0241} & \underline{0.0252} & \textbf{0.0052} & 0.0257 & \textbf{0.0455} & \textbf{0.0399} & \textbf{0.0564} \\
\hspace{3mm}P(True) & $L$
& 0.3119 & 0.6334 & 0.0441 & 0.0426 & 0.2620 & 0.1634 & 0.2886 \\
\hspace{6mm}+ isotonic regression & $L$
& \textbf{0.0156} & \textbf{0.0217} & \underline{0.0054} & \underline{0.0200} & \underline{0.0803} & \underline{0.0465} & \underline{0.1053} \\
\hspace{3mm}Probe & $<1$
& 0.0275 & 0.0557 & 0.0322 & 0.0462 & 0.1089 & 0.0498 & 0.1368 \\

\bottomrule
\end{tabular}%
}
\end{table}

\begin{figure}[h]
    \centering

    \begin{subfigure}{\linewidth}
        \centering
        \includegraphics[width=\linewidth]{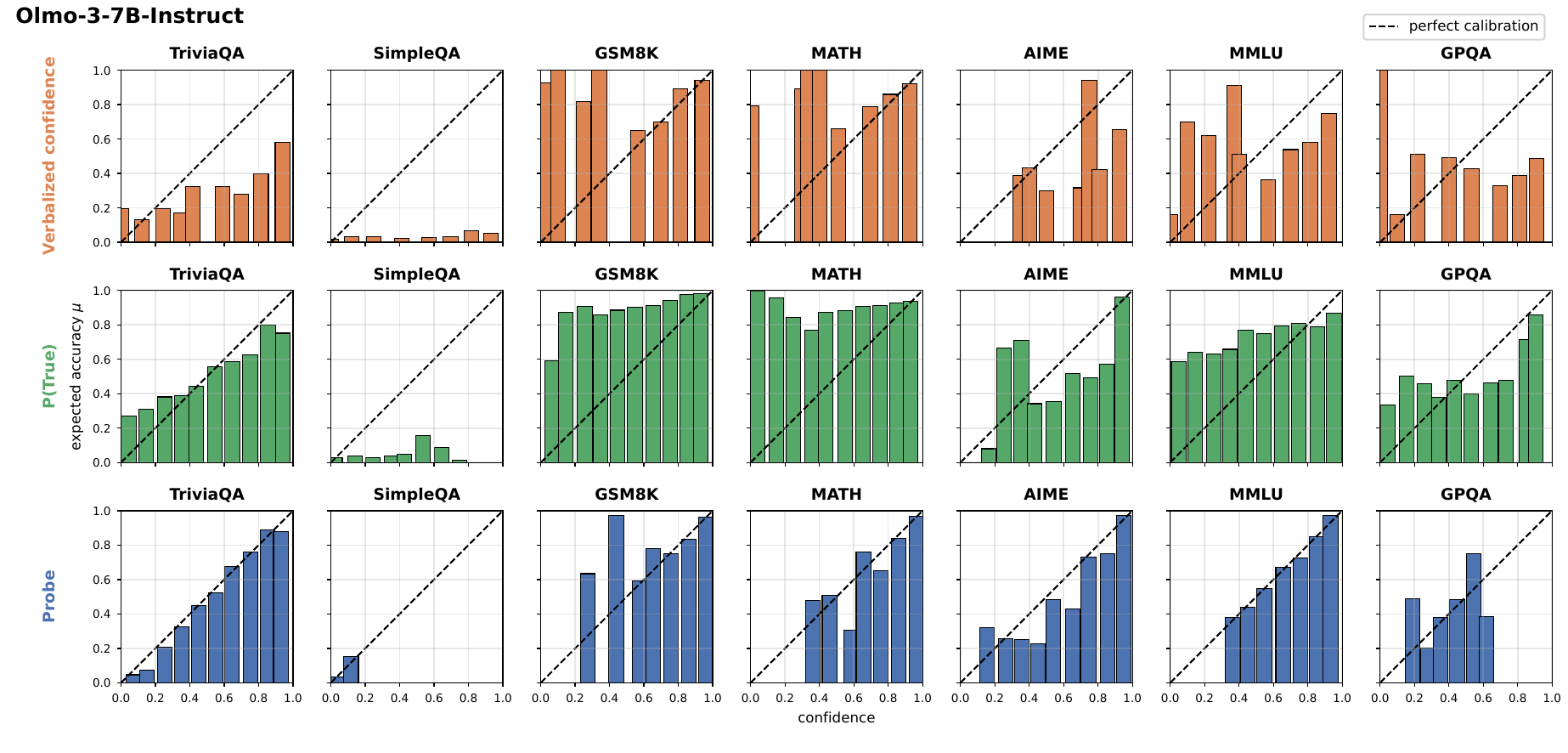}
    \end{subfigure}

    \vspace{2mm}

    \begin{subfigure}{\linewidth}
        \centering
        \includegraphics[width=\linewidth]{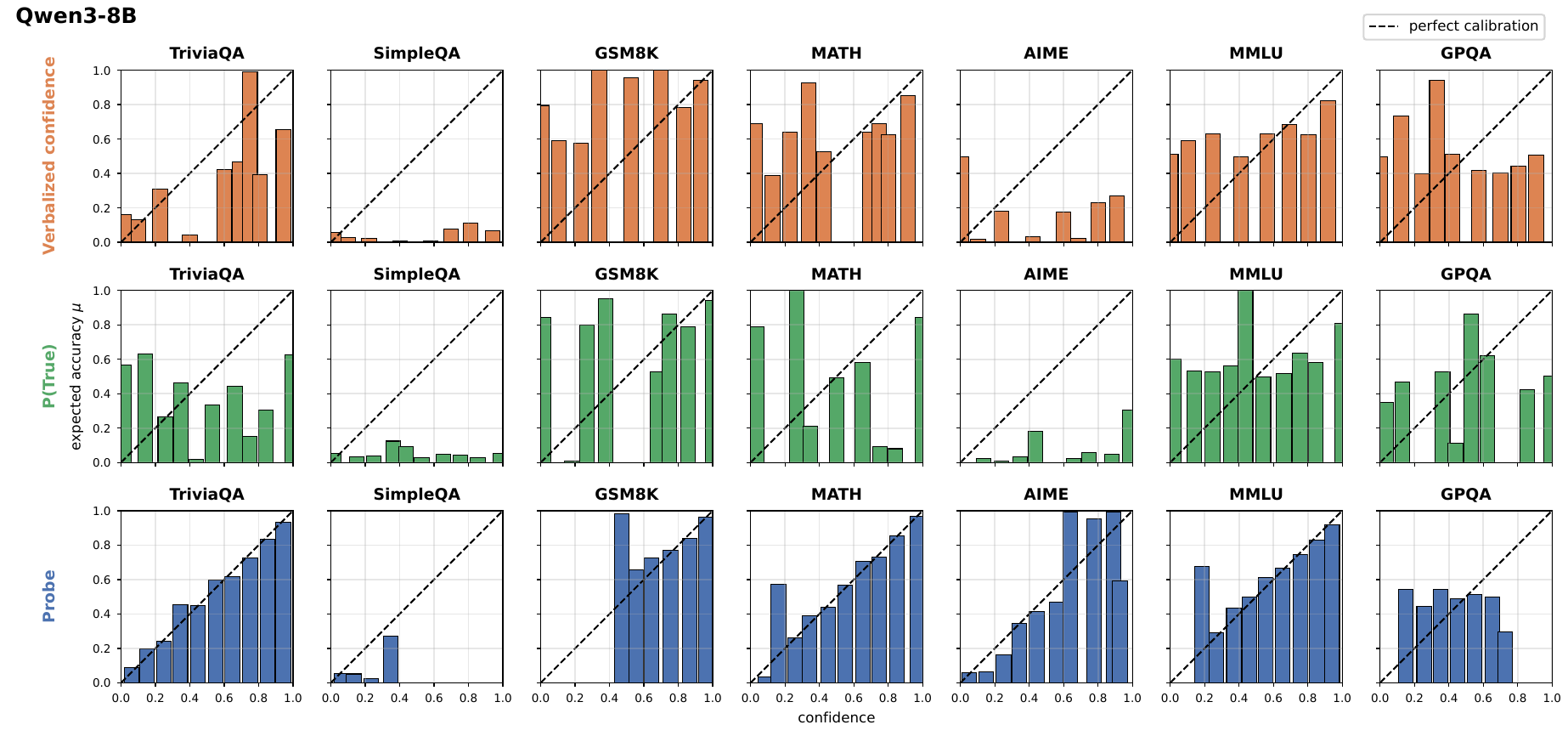}
    \end{subfigure}
    
    \vspace{2mm}    
    
    \begin{subfigure}{\linewidth}
        \centering
        \includegraphics[width=\linewidth]{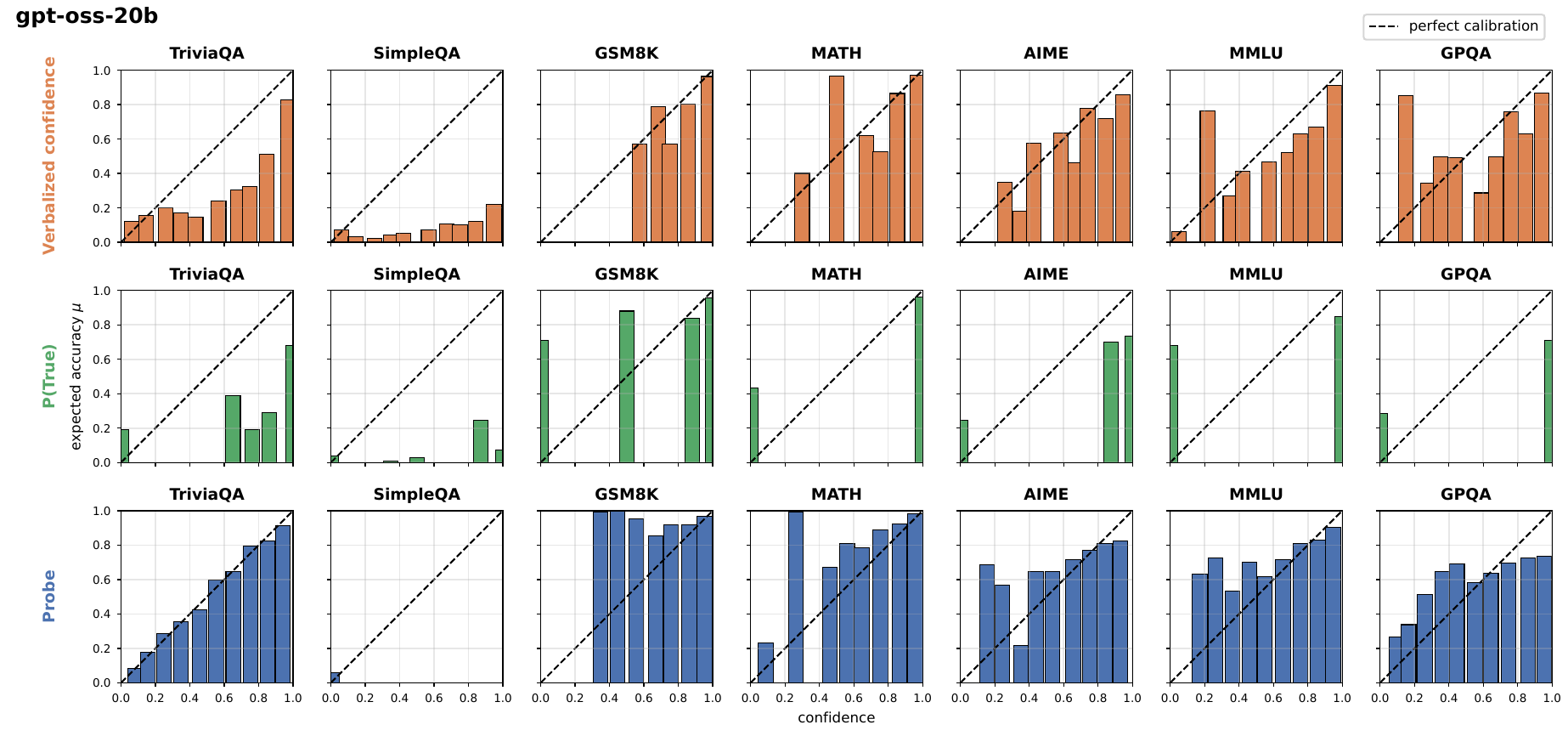}
    \end{subfigure}

    \caption{\textbf{Reliability diagrams for the three models.}}
    \label{fig:reliability_diagrams}
\end{figure}

\clearpage

\subsection{Improving verbalized confidence} \label{sec: improving verbal and ptrue}

In this section, we investigate whether few-shot examples help verbalized confidence. Experiment results are in \cref{tab: few-shot verbalized}. The performance of few-shot verbalized confidence differs across LLMs and datasets; few-shot examples do not always improve the performance.

\begin{table}[htbp]
    \centering
    
    \caption{\textbf{Performance of few-shot verbalized confidence on TriviaQA (left) and GSM8K (right) with reduced offline training costs}. Few-shot examples are from the training set of TriviaQA and GSM8K, respectively. Three few-shot examples are provided. Other confidence estimators are listed for comparison. Results are reported by Brier scores ($\downarrow$). While few-shot verbalized confidence consistently improves upon the zero-shot baseline on GSM8K, this advantage is not observed consistently on TriviaQA.}
    \label{tab: few-shot verbalized}
    
    \begin{minipage}{.48\textwidth}
        \centering
        \resizebox{\textwidth}{!}{%
        \begin{tabular}{lccc}
            \toprule
            Method / Model & Olmo-3-7B-Instruct & Qwen3-8B & gpt-oss-20b \\
            \midrule
            Uniform random baseline & 0.2745 & 0.2865 & 0.2639 \\
            Verbalized confidence & 0.2717 & 0.2459 & 0.1204 \\
            Verbalized confidence (few-shot) & 0.2802 & 0.2724 & 0.1217 \\
            P(True) & 0.1698 & 0.3253 & 0.2435 \\
            Probe & \textbf{0.1101} & \textbf{0.1163} & \textbf{0.1027} \\
            \bottomrule
        \end{tabular}%
        }
    \end{minipage}\hfill
    \begin{minipage}{.48\textwidth}
        \centering
        \resizebox{\textwidth}{!}{%
        \begin{tabular}{lccc}
            \toprule
            Method / Model & Olmo-3-7B-Instruct & Qwen3-8B & gpt-oss-20b \\
            \midrule
            Uniform random baseline & 0.3119 & 0.3144 & 0.3195 \\
            Verbalized confidence & 0.0444 & 0.0455 & 0.0247 \\
            Verbalized confidence (few-shot) & 0.0423 & 0.0429 & \textbf{0.0237} \\
            P(True) & 0.1266 & 0.0483 & 0.0299 \\
            Probe& \textbf{0.0364} & \textbf{0.0377} & 0.0291 \\
            \bottomrule
        \end{tabular}%
        }
    \end{minipage}
    
\end{table}

\subsection{Cost of confidence estimators}
\label{sec: cost of conf estimators}

The cost for confidence estimators stems from \textbf{(1)} the one-time offline preprocessing/training cost and \textbf{(2)} the recurring online inference cost.  We detail two types of costs below:

For \textbf{(1)} offline training, only the linear probe requires offline cost for all confidence estimators we discussed in \cref{sec: conf estimators}. The cost comes from (a) constructing training labels for each query by repeated sampling to estimate the expected accuracy, and (b) extracting hidden states. However, once the hidden states are cached, training the probe itself only requires fitting on frozen hidden states rather than backpropagating through the full LLM. In our case, hidden-state extraction, caching, and probe training were all done on a single 20GB GPU.

For \textbf{(2)} online cost, all methods require one forward pass to process the query. Linear probe only add a single linear-layer computation, which is much smaller and is the main advantage of linear probes. By contrast, methods such as verbalized confidence require autoregressive decoding, often requiring hundreds of forward passes. Therefore, linear probing is substantially cheaper during online deployment, where latency is crucial and inference cost scales with usage.

Meanwhile, Training cost of linear probes can be further reduced: While we use $k_{eval}=100$, probes can be trained with a smaller $k_{train}$. We show this by using $k_{train} = 10$ for building training sets, and find that resulting probes still achieve similar performance while reducing data construction cost to one-tenth of the original setup where $k_{train} = 100$. Experiment results of probes trained on TriviaQA and GSM8K are presented in \cref{tab:reduced_results}.

\begin{table}[htbp]
    \centering
    
    \caption{\textbf{Performance of linear probes on TriviaQA (left) and GSM8K (right) with reduced offline training costs}. Results are reported by Brier scores ($\downarrow$). Other confidence estimators are listed for comparison. Results show that linear probes trained with $k_{train} = 10$ achieve nearly identical performance to the full setup ($k_{train} = 100$) across all evaluated models.}
    \label{tab:reduced_results}
    
    \begin{minipage}{.48\textwidth}
        \centering
        \resizebox{\textwidth}{!}{%
        \begin{tabular}{lccc}
            \toprule
            Method / Model & Olmo-3-7B-Instruct & Qwen3-8B & gpt-oss-20b \\
            \midrule
            Uniform random baseline & 0.2745 & 0.2865 & 0.2639 \\
            Verbalized confidence & 0.2717 & 0.2459 & 0.1204 \\
           P(True) & 0.1698 & 0.3253 & 0.2435 \\
            Probe ($k_{train} = 100$) &  \textbf{0.1101} & \textbf{0.1163} & \textbf{0.1027} \\
            Probe ($k_{train} = 10$) & 0.1117 & 0.1083 & 0.0859 \\
            \bottomrule
        \end{tabular}%
        }
    \end{minipage}\hfill
    \begin{minipage}{.48\textwidth}
        \centering
        \resizebox{\textwidth}{!}{%
        \begin{tabular}{lccc}
            \toprule
            Method / Model & Olmo-3-7B-Instruct & Qwen3-8B & gpt-oss-20b \\
            \midrule
            Uniform random baseline & 0.3119 & 0.3144 & 0.3195 \\
            Verbalized confidence & 0.0444 & 0.0455 & \textbf{0.0247} \\
            P(True) & 0.1266 & 0.0483 & 0.0299 \\
            Probe ($k_{train} = 100$) &  \textbf{0.0364} & \textbf{0.0377} & 0.0291\\
            Probe ($k_{train} = 10$) & 0.0370 & 0.0371 & 0.0290 \\
            \bottomrule
        \end{tabular}%
        }
    \end{minipage}
    
\end{table}

\section{Details of the applications} \label{sec:dec_applications} 

\subsection{Pass@\texorpdfstring{$k$}{k} simulation details}\label{sec: pass@k sim details}
\subsubsection{Process of simulating Pass@\texorpdfstring{$k$}{k} curve }  \label{sec: pass@k sim process}

\paragraph{Pass@$k$ score at each $k$.} Define the capability-calibrated confidence as $p$. For each query $i$ that is sampled $k$ times, the success rate, i.e., pass it or not, follows a Bernoulli distribution. The Bernoulli distribution has mean $p$ and variance $p(1-p)$. Therefore, for each query $i$ that samples $k$ times, the success rate is $P(\text{success@}k)_i = S_{i,k} = 1 - (1 - p_i)^k$, the variance is $S_{i,k}(1 - S_{i,k})$. Generalized to the dataset level. For a dataset $\mathcal{D}=\{x_i\}_{i=1}^N$, the pass@$k$ score's mean and variance are $\mu_k = \frac{1}{N} \sum_{i=1}^{N} S_{i,k}$ and $\text{std}_k^2 = \frac{1}{N^2} \sum_{i=1}^{N} S_{i,k}(1 - S_{i,k})$ respectively. Here, we assume that each query is independent, so the covariance is zero. 

\paragraph{Drawing the simulation curve.} The mean and variance of the pass@$k$ score are estimated for each $k$ along the curve. Based on the Central Limit Theorem, the distribution of pass@$k$ is treated as a normal distribution, allowing for the calculation of a 95\% confidence interval. This interval represents the region where the true curve is expected to be.

\subsubsection{More simulation results} \label{sec: more passk simulation}
In this section, we discuss more pass@$k$ simulation results. First of all, the MSE of AIME simulation results are in \cref{tab:pass_at_k_aime}. The experiment results are consistent with the simulation results on MATH-500.

\cref{fig:passk_curve} presents the pass@$k$ simulation curve at the dataset level. While Oracle Capability-Calibrated Confidence (Oracle-CC) fits the actual pass@$k$ curve near perfectly, Probe does not fit well in most scenarios. The experiment results encourage future work on developing confidence estimators that can capture the model's capability on the dataset.  

\begin{table}[h]
\centering
\caption{\textbf{Pass@$k$ simulation error (MSE) on the AIME dataset.} The experiment findings are consistent with the MATH-500 simulation results at \cref{tab:pass_at_k_math}.}
\label{tab:pass_at_k_aime}
\begin{tabular}{p{2.8cm}cccc}
\toprule
\textbf{Method} & \textbf{pass@1} & \textbf{pass@4} & \textbf{pass@16} & \textbf{pass@64} \\ 
\midrule
\multicolumn{5}{l}{\textbf{Olmo-3-7B-Instruct}} \\
\hspace{3mm}Oracle-CC& 0.0000 & 0.0000 & 0.0001  & 0.0014 \\
\hspace{3mm}Oracle-RC & 0.0967 & 0.1523 & 0.2541  & 0.3395 \\
\hspace{3mm}Probe & 0.1092 & 0.2002 & 0.1928  & 0.1577 \\ 
\midrule
\multicolumn{5}{l}{\textbf{Qwen3-8B}} \\
\hspace{3mm}Oracle-CC& 0.0000 & 0.0000 & 0.0002  & 0.0036 \\
\hspace{3mm}Oracle-RC & 0.0522 & 0.0988 & 0.2307  & 0.4170 \\
\hspace{3mm}Probe & 0.0578 & 0.1386 & 0.2495  & 0.2992 \\ 
\midrule
\multicolumn{5}{l}{\textbf{gpt-oss-20b}} \\
\hspace{3mm}Oracle-CC & 0.0000 & 0.0000 & 0.0000 & 0.0006 \\
\hspace{3mm}Oracle-RC & 0.1068 & 0.1520 & 0.1959  & 0.2223 \\
\hspace{3mm}Probe & 0.0966 & 0.0968 & 0.0744  & 0.0584 \\ 
\bottomrule
\end{tabular}
\end{table}

\begin{figure}[ht!]
    \centering
    \newcommand{\figw}{0.48\textwidth}

    \begin{subfigure}{\figw}
        \includegraphics[width=\linewidth]{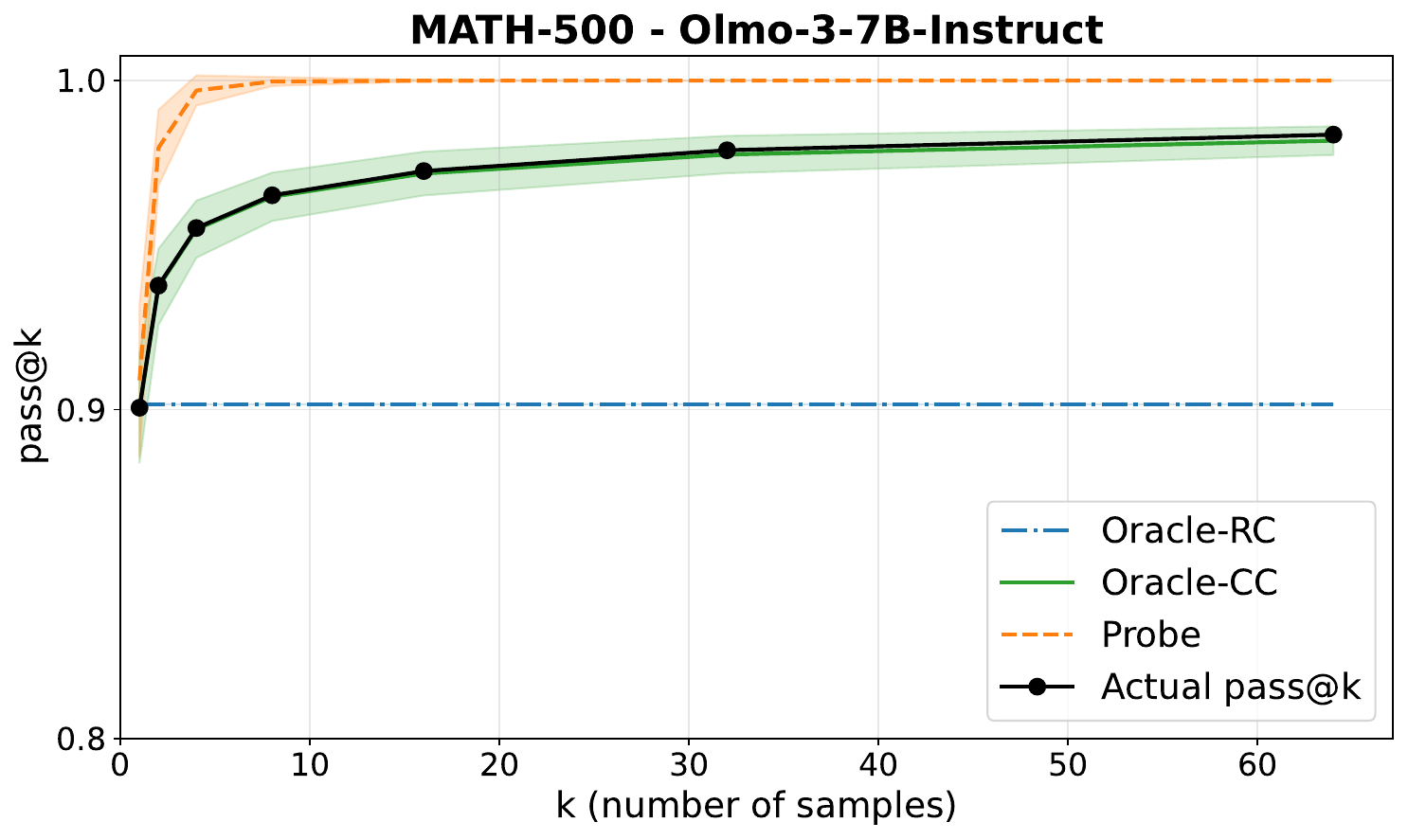}
    \end{subfigure}\hfill
    \begin{subfigure}{\figw}
        \includegraphics[width=\linewidth]{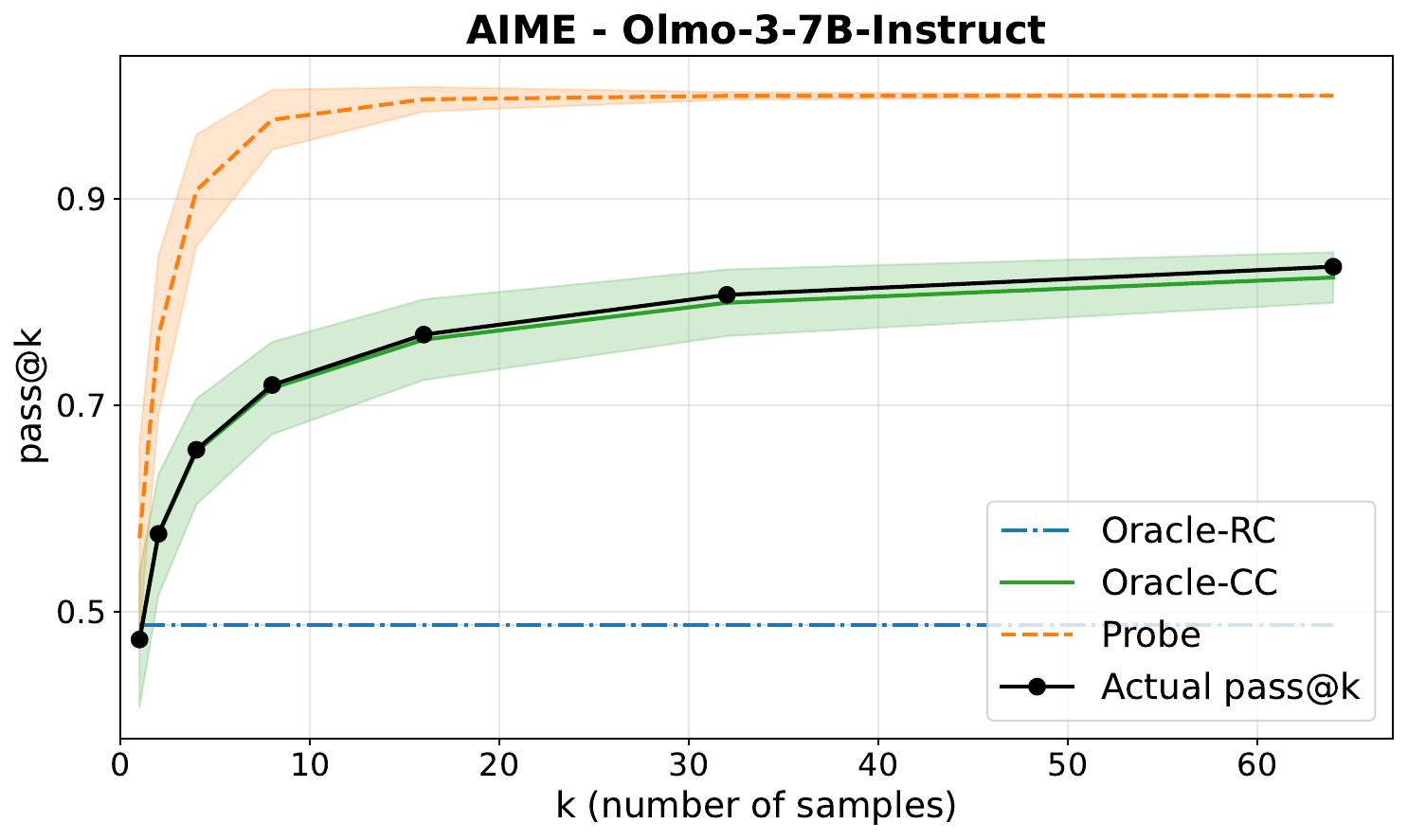}
    \end{subfigure}

    \vspace{0.3cm} 

    \begin{subfigure}{\figw}
        \includegraphics[width=\linewidth]{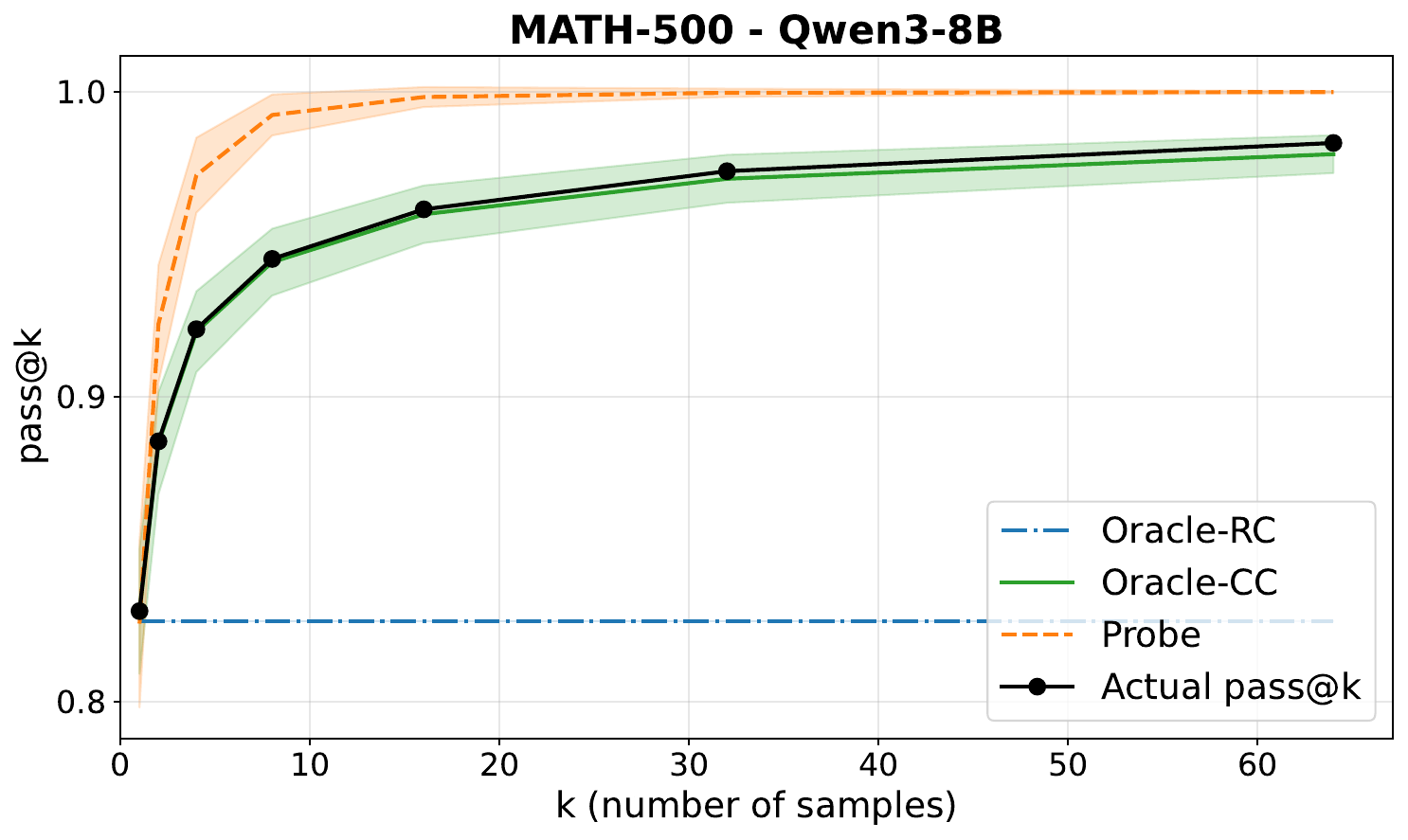}
    \end{subfigure}\hfill
    \begin{subfigure}{\figw}
        \includegraphics[width=\linewidth]{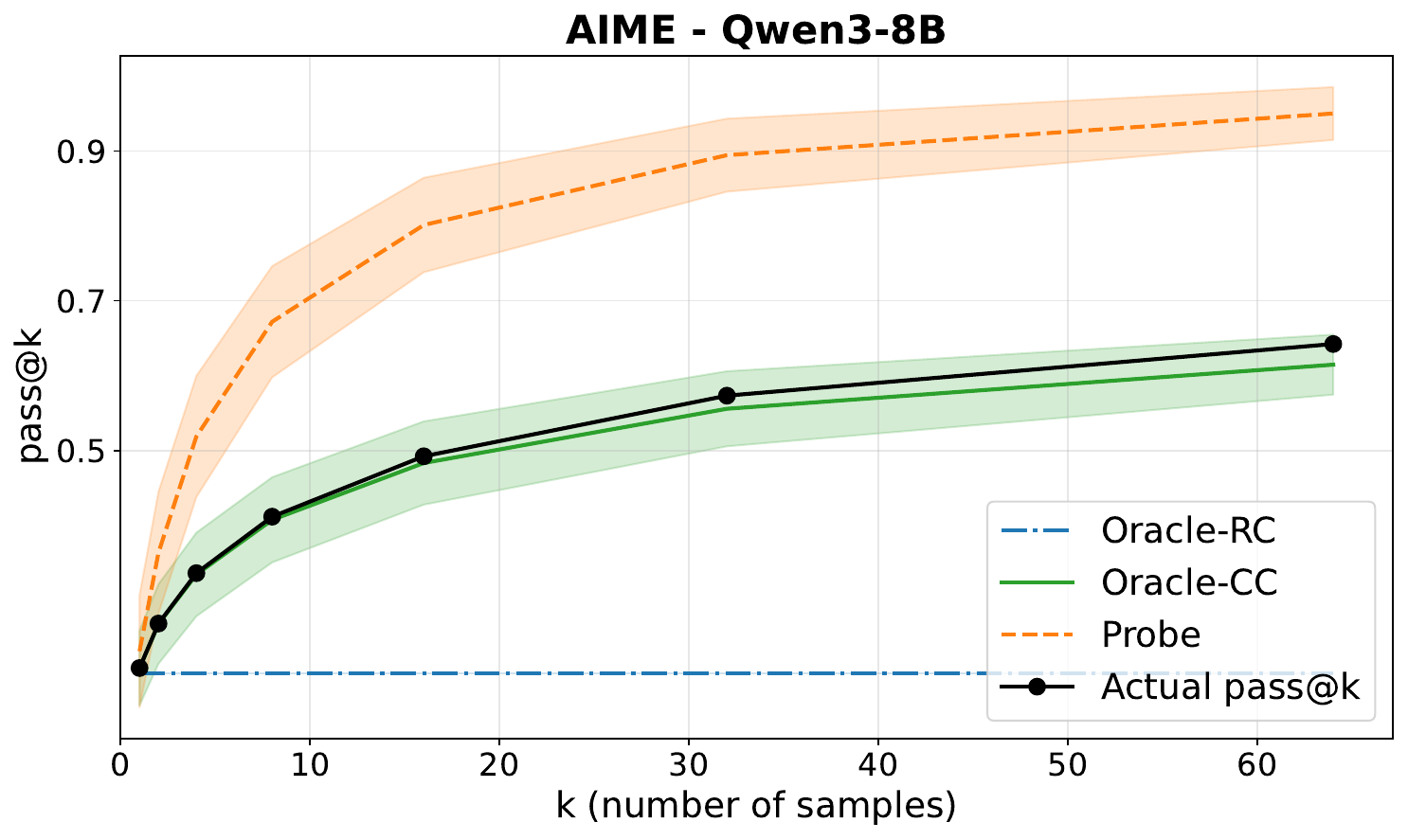}
    \end{subfigure}

    \vspace{0.3cm}

    \begin{subfigure}{\figw}
        \includegraphics[width=\linewidth]{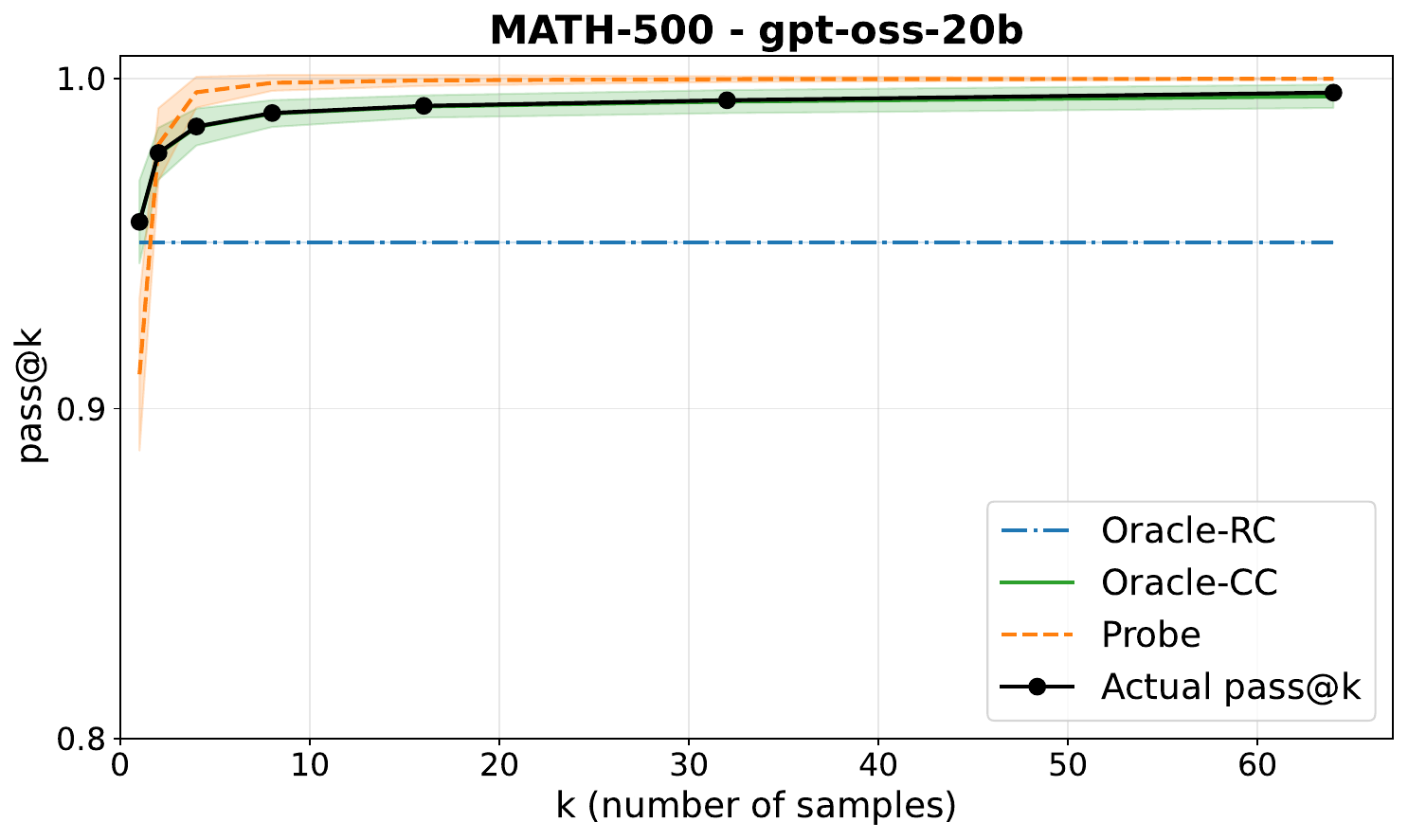}
    \end{subfigure}\hfill
    \begin{subfigure}{\figw}
        \includegraphics[width=\linewidth]{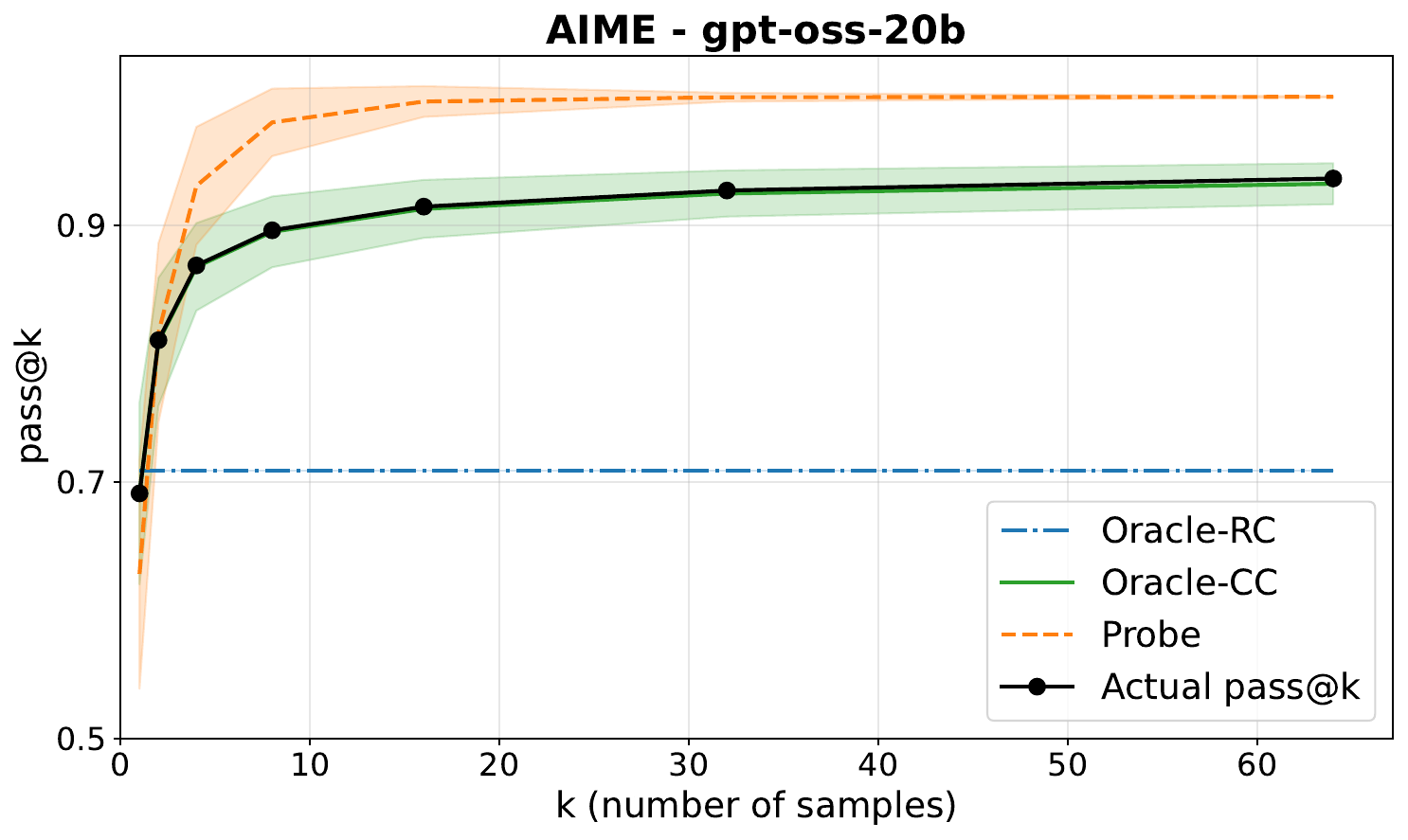}
    \end{subfigure}

    \caption{\textbf{Pass@$k$ simulation results across 3 models  on MATH-500 and AIME datasets.} While Oracle-CC can simulate the actual pass@$k$ near-perfectly, Oracle-RC performs poorly as it is not measuring the models' capabilities. However, Probe does not fit well in many scenarios.}
    \label{fig:passk_curve}
\end{figure}

\clearpage
\subsection{Inference budget allocation details}\label{sec:budget_allocation_details}
\subsubsection{Greedy algorithm for test-time compute allocation \texorpdfstring{\citep{damani2024learning}}{(Damani et al., 2024)}} \label{sec: mit greedy algorithm}

To allocate the inference budget efficiently, we maximize the expected number of solved questions (best-of-$k$). Let $N$ be the number of questions, $N\times B$ be the total budget, and $p_i$ be the capability-calibrated confidence estimation for question $i$.

The total expected score $S$ is the sum of the probabilities that each question is solved at least once:
\begin{equation}
    S = \sum_{i=1}^{N} \left[ 1 - (1 - p_i)^{k_i} \right]
\end{equation}
where $k_i$ is the number of samples allocated to question $i$.

To optimize this, we analyze the marginal improvement of adding a single sample to question $i$, given that it has already been allocated $k_i$ samples:
\begin{equation}
\begin{split}
    \text{Gain}_i &= S(\text{with } k_i+1 \text{ samples}) - S(\text{with } k_i \text{ samples}) \\
    &= \left[ 1 - (1 - p_i)^{k_i+1} \right] - \left[ 1 - (1 - p_i)^{k_i} \right] \\
    &= (1 - p_i)^{k_i} - (1 - p_i)^{k_i+1} \\
    &= (1 - p_i)^{k_i} \left[ 1 - (1 - p_i) \right] \\
    &= p_i (1 - p_i)^{k_i}
\end{split}
\end{equation}
Because the gain function is strictly decreasing with respect to $k_i$, a greedy strategy that iteratively assigns the next budget unit to the question with the highest current $\text{Gain}_i$ results in better allocation.

\subsubsection{More experiment results} \label{sec: more budget_alloc results}
Full experiment results of inference budget allocation are available at \cref{fig:budget_full}. We observe that confidence estimators with lower Brier scores have better inference budget allocation performance.

\begin{figure}[h!]
    \centering
    \newcommand{\figw}{0.48\textwidth}

    \begin{subfigure}{\figw}
        \includegraphics[width=\linewidth]{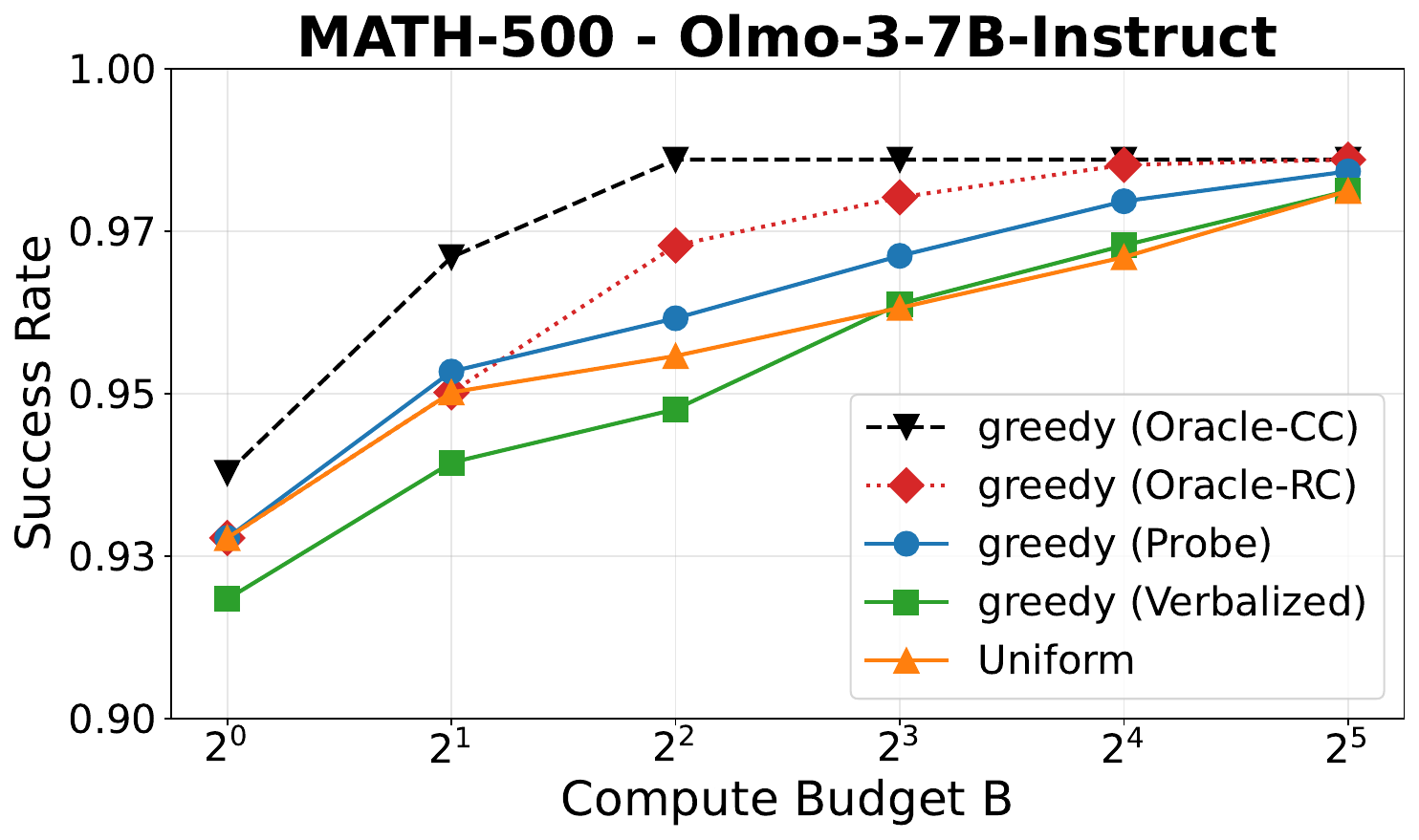}
    \end{subfigure}\hfill
    \begin{subfigure}{\figw}
        \includegraphics[width=\linewidth]{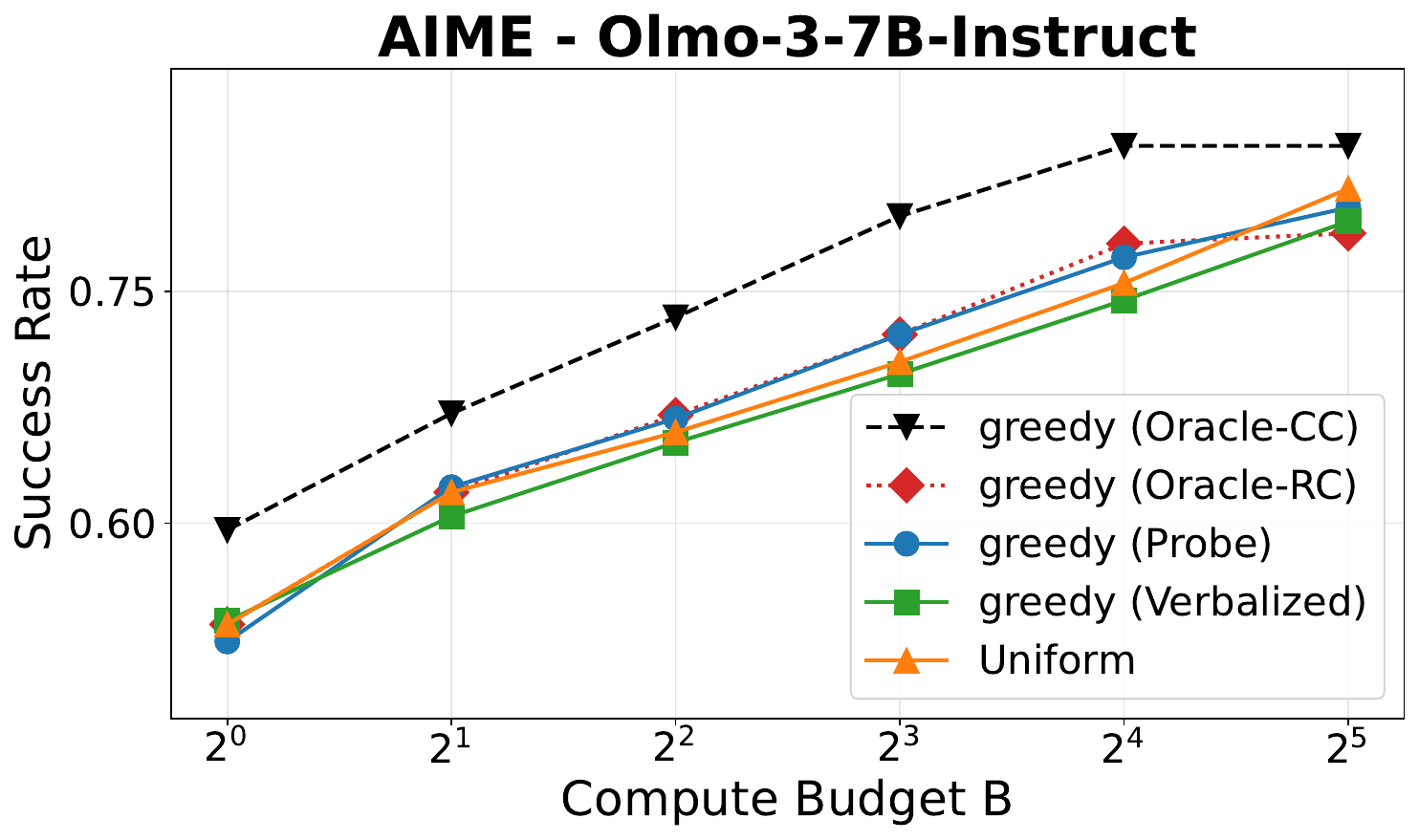}
    \end{subfigure}

    \vspace{0.3cm} 

    \begin{subfigure}{\figw}
        \includegraphics[width=\linewidth]{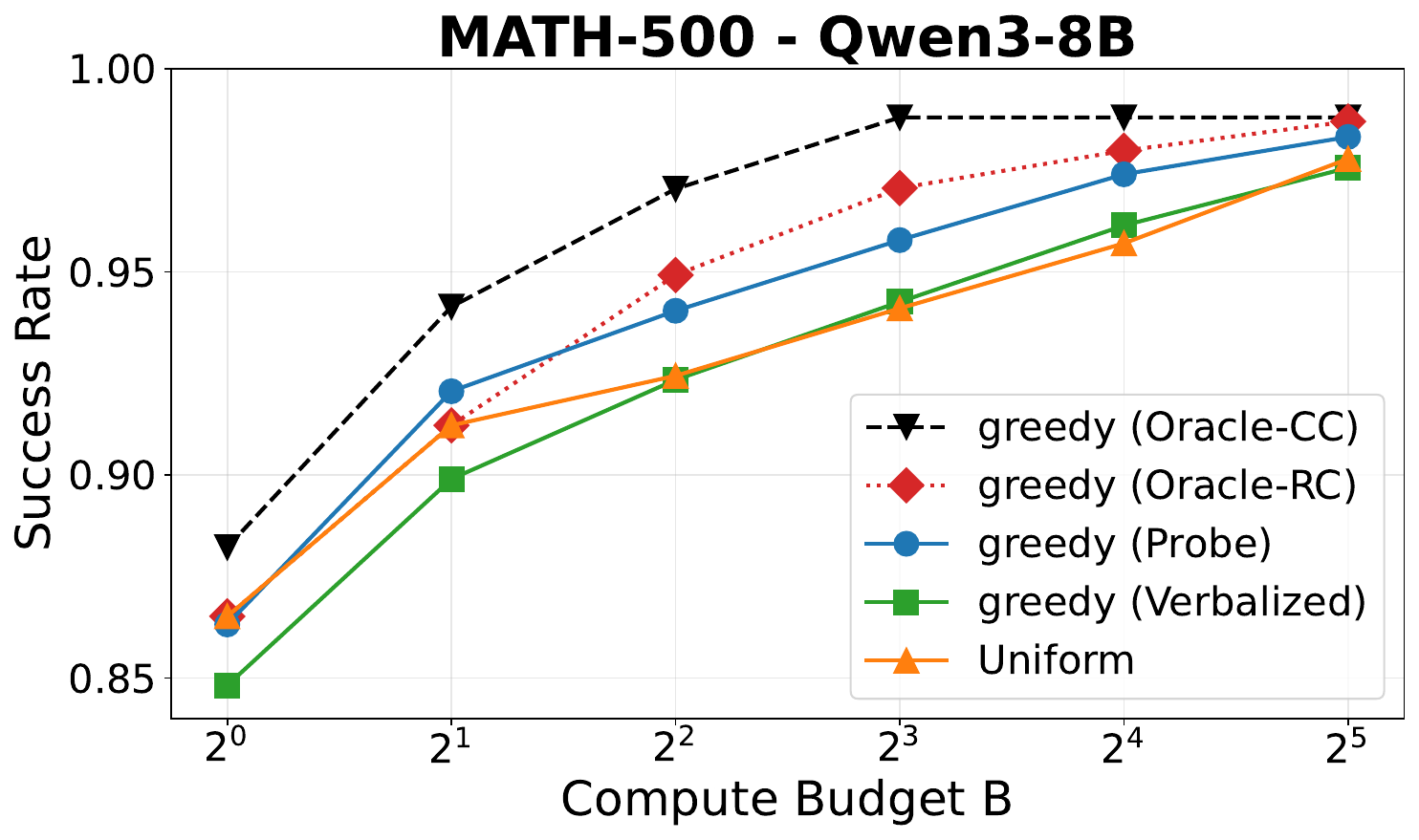}
    \end{subfigure}\hfill
    \begin{subfigure}{\figw}
        \includegraphics[width=\linewidth]{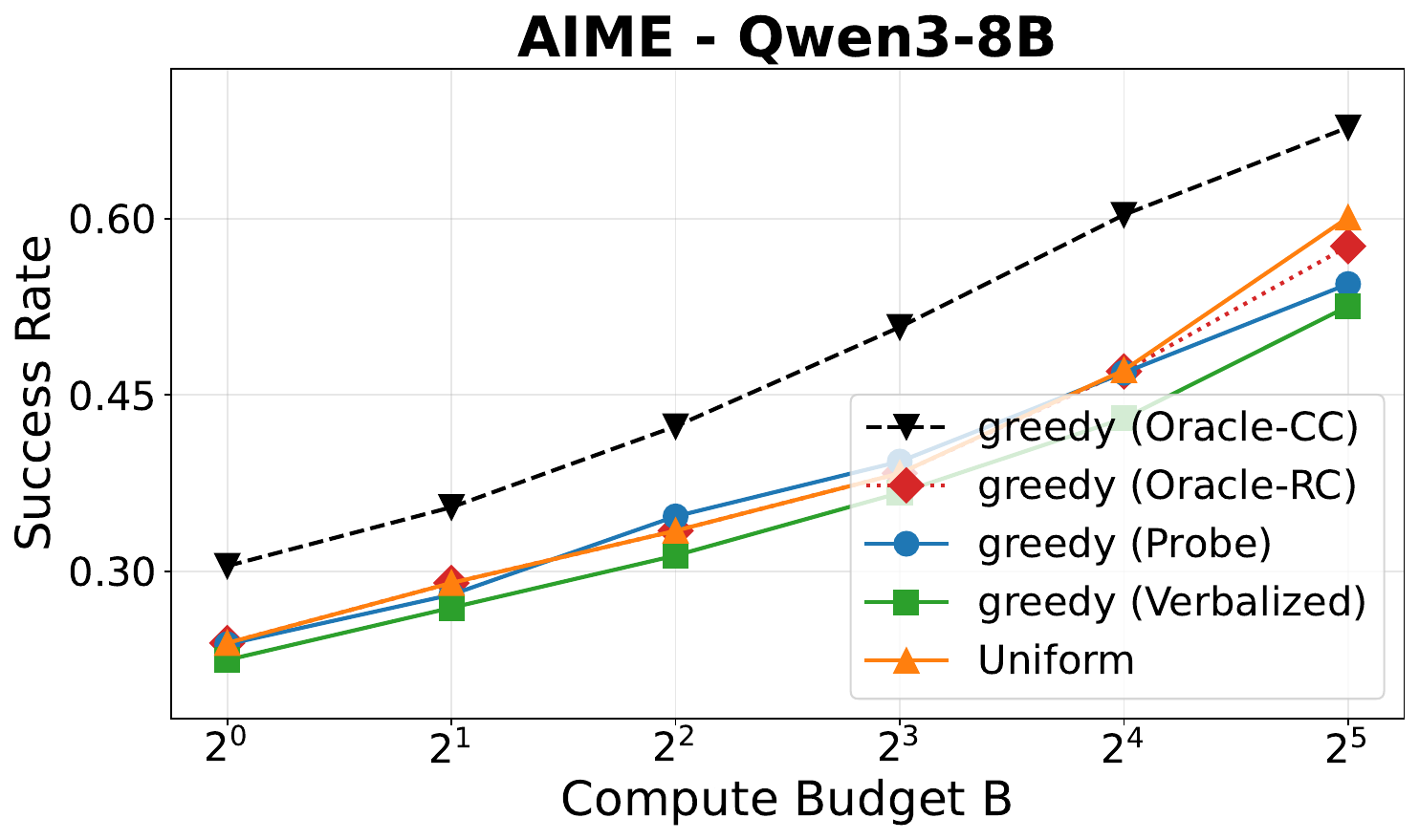}
    \end{subfigure}

    \vspace{0.3cm}

    \begin{subfigure}{\figw}
        \includegraphics[width=\linewidth]{figures/budget_alloc/math-500_gpt-oss-20b_budget_sweep.pdf}
    \end{subfigure}\hfill
    \begin{subfigure}{\figw}
        \includegraphics[width=\linewidth]{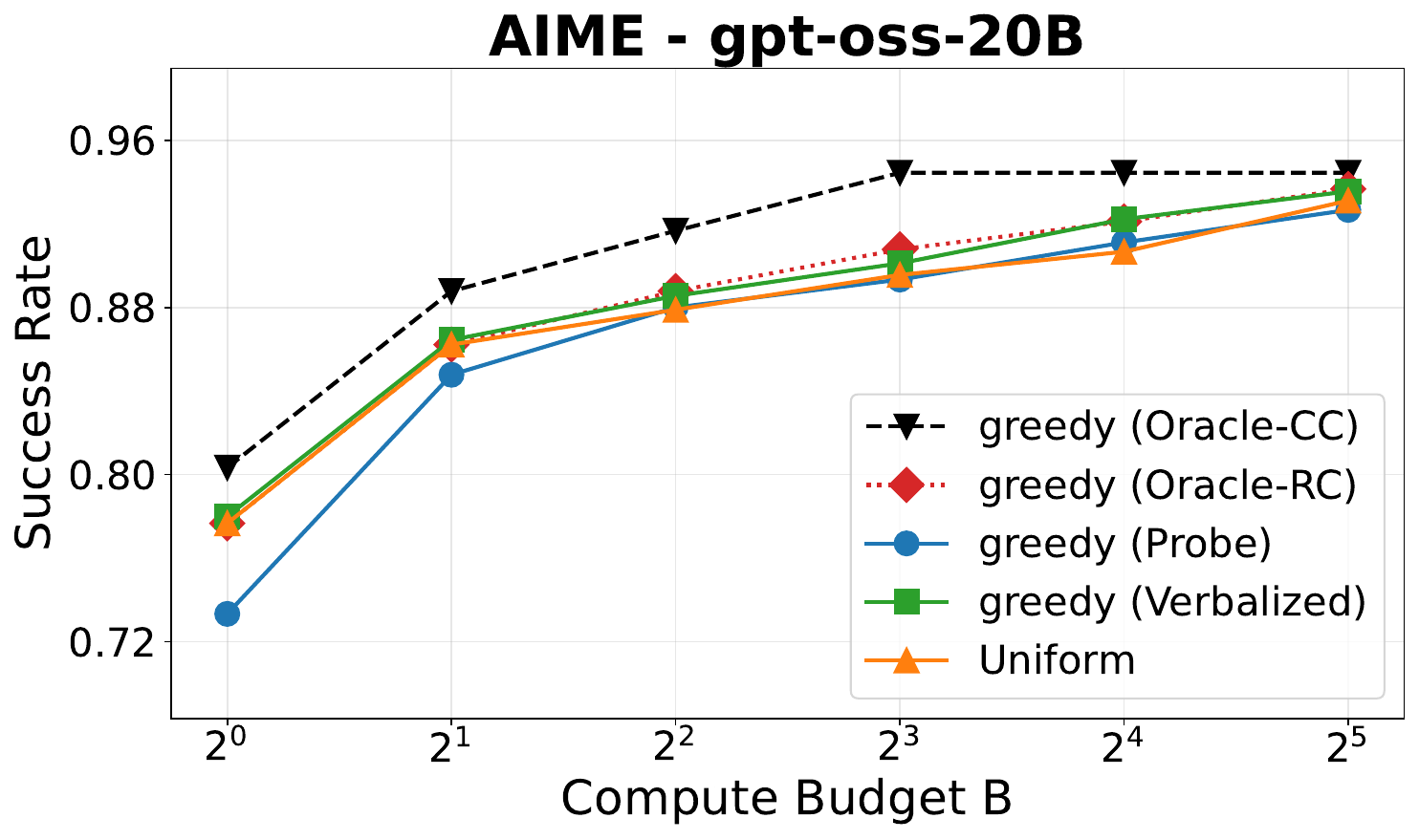}
    \end{subfigure}

    \caption{\textbf{Budget allocation results across 3 models  on MATH-500 and AIME datasets.} Oracle confidence consistency reaches the best performance in all scenarios. Verbalized confidence and Probe often outperform the uniform allocation. If the Brier score is good (see \cref{tab:main-results} for the Brier score), the estimated confidence will have better performance in inference budget allocation.}
    \label{fig:budget_full}
\end{figure}

\clearpage
\subsection{Estimator selection and downstream impact under RC and CC}\label{sec: estimator selection}

In this section, we ask:
\[
\textit{As two evaluation frameworks, do RC and CC select different confidence estimators?}
\]

We consider the 21 model--dataset settings in \cref{tab:main-results} and focus on query-level confidence estimation, where estimators receive only the query. At the dataset level, we test whether using expected accuracy as CC's target and single-response correctness as RC's target leads to different estimator selections. We use Brier score as the evaluation metric.

For CC, we sample $k_{eval}=100$ responses and use \cref{tab:main-results} to identify the CC-selected estimator in each setting, whose ranking is stable. For RC, the selected estimator depends on the sampled response correctness. We therefore conduct 100,000 simulations to estimate each estimator's selection probability.

Among Verbalized confidence, P(True), and Probe, RC can select a different estimator from CC in six model--dataset settings. These results show that RC can select a different estimator from CC with non-negligible probability. When estimators' performances are close, RC is more likely to select a different method because it evaluates against a noisy target induced by stochastic decoding. When one estimator clearly dominates, as Probe does in several settings, RC and CC select the same estimator. \cref{tab:rc-cc-selection} further reports the gpt-oss-20b settings where an alternative RC selection occurs in at least 5\% of trials. A low disagreement probability reflects estimator dominance, not framework equivalence.

\begin{table}[h]
\centering
\small
\caption{\textbf{CC selections and alternative single-response RC selections for gpt-oss-20b.}
We report an alternative RC selection (\textbf{RC alt.}) only when it is selected in at least 5\% of trials. P and V denote Probe and Verbalized confidence, respectively.}
\label{tab:rc-cc-selection}
\resizebox{\textwidth}{!}{
\begin{tabular}{l c c c c c c c}
\toprule
\textbf{CC choice; RC alt. (\%)} &
\textbf{TriviaQA} &
\textbf{SimpleQA} &
\textbf{GSM8K} &
\textbf{MATH} &
\textbf{AIME} &
\textbf{MMLU} &
\textbf{GPQA} \\
\midrule
\textbf{gpt-oss-20b}
& P; P
& P; P
& V; V
& \textbf{V; P (12.3\%)}
& \textbf{V; P (6.0\%)}
& V; V
& V; V \\
\bottomrule
\end{tabular}
}
\end{table}

\subsubsection{Impact on Downstream Applications}

We next ask:
\[
\textit{Which evaluation framework selects better estimators for downstream applications, RC or CC?}
\]

We focus on gpt-oss-20b on MATH and AIME, the two settings in \cref{tab:rc-cc-selection} with disagreement probabilities above 5\%. For Pass@\(k\) prediction, \cref{tab:rc-cc-passk} shows that the CC-selected estimator can perform better. For inference-budget allocation, \cref{fig:budget_full} shows that on AIME, the CC-selected Verbalized confidence outperforms the alternative RC selection, Probe.

\begin{table}[h]
\centering
\small
\caption{\textbf{Pass@\(k\) prediction MSE for the CC-selected estimator and the alternative RC selection.}
Lower is better. RC selection probabilities are shown in parentheses.}
\label{tab:rc-cc-passk}
\resizebox{\textwidth}{!}{
\begin{tabular}{l l l c c c c}
\toprule
\textbf{Data} &
\textbf{Selection} &
\textbf{Estimator} &
\textbf{Pass@1} &
\textbf{Pass@4} &
\textbf{Pass@16} &
\textbf{Pass@64} \\
\midrule

MATH
& CC choice
& Verbalized
& 0.0186
& 0.0102
& 0.0067
& 0.0028 \\

& RC alternative
& Probe (12.3\%)
& 0.0216
& 0.0102
& 0.0068
& 0.0028 \\

\midrule

AIME
& CC choice
& Verbalized
& 0.0581
& 0.0717
& 0.0701
& 0.0584 \\

& RC alternative
& Probe (6.0\%)
& 0.0966
& 0.0968
& 0.0744
& 0.0584 \\

\bottomrule
\end{tabular}
}
\end{table}

On MATH, the two estimators perform similarly because their confidence estimates are similar, not because RC and CC are equivalent. On AIME, the CC-selected Verbalized confidence achieves lower prediction error than the alternative RC selection. \cref{sec: other applications} provides further discussion of evaluation framework selection.

\subsection{Other potential applications} \label{sec: conceptual discussion with other appli.}
In this section, we first discuss how capability-calibrated confidence relates to other potential applications.\\
\textbf{Resource Routing and Estimation.}  As discussed in \citet{damani2024learning}, by estimating the ranking of a group of models' capability on answering a question, capability-calibrated confidence can be directly applied to LLM Routing \citep{maurya2025selectllm, jiang2023llm, chen2023frugalgpt, ong2024routellm}. Additionally, capability-calibrated confidence enables cost estimation. By estimating the expected accuracy $\mu$ (difficulty of the queries), we can predict the expected sampling budget $\frac{1}{\mu}$ required to generate a correct response \citep{wu2024inference}.

\textbf{Inference-Time Reliability and Active Refinement.}
By defining a confidence threshold, one can implement reliable systems that perform Selective Prediction ~\citep{mao2025calibrating, duan2024shifting, chen2023adaptation,kamath2020selective}, where systems could abstain or seek human assistance when the model is uncertain~\citep{wu2024need,chen2025query} or perform query rewriting when the confidence is low.

\textbf{Enhanced Learning and Evaluation.} Accurately estimating expected accuracy serves as a proxy for query difficulty. It allows Curriculum Learning \citep{zhang2025beyond, zhang2025clpo} that sorts training data by difficulty, or identifies the effects of easy and hard queries. Meanwhile, capability-calibrated confidence enables the evaluation of models on unlabeled test sets, which is called Label-free Benchmarking \citep{guha2024smoothie, zhang2025reasonerrank}. We can derive model rankings that align with ground-truth evaluations.

\subsubsection{The Latency Issue of RC}

Beyond accuracy, RC creates a systemic latency issue. Because RC evaluates a specific generated response, the system must wait for the LLM to autoregressively decode the full answer before estimating confidence. We now discuss more details of the scenarios we mentioned in \cref{sec: other applications}:

\textbf{Model Routing}: In model routing \citep{chen2023frugalgpt,damani2024learning,maurya2025selectllm}, systems need to decide beforehand whether to invoke a cheap model or a powerful, expensive one. If the system must wait for the cheap model to fully decode a response just to extract a response-calibrated confidence score, the latency and compute have already been wasted. Therefore, \citet{damani2024learning} uses lightweight probes to predict "expected accuracy" for model routing and budget allocation rather than predicting single-response correctness. Similarly, \citet{maurya2025selectllm} trains a lightweight router to estimate the likelihood of a correct response by using training labels derived from the majority vote of 10 responses, which is a direct approximation of capability calibration. While neither work formally names the concept or links it to the broader confidence calibration or uncertainty quantification literature, both implicitly perform “query-level capability estimation” - a practically necessary step for model routing and budget allocation.

\textbf{Active Query Refinement}~\citep{wu2024need}: Systems that rewrite vague prompts need to know if the model is capable of answering the query. Waiting for the model to generate a flawed answer before triggering refinement is highly inefficient. In contrast, the benchmarked CC method (linear probe) estimates the model's capability before generation, using only a single forward pass over the query.

\clearpage
\section{Targets difference} \label{app: full metric diff}

In this section, we compare the targets of response calibration (single-response correctness) and capability calibration (the expected accuracy of the output distribution) across three models and seven datasets. As illustrated in \cref{fig:full_metric_diff}, a clear divergence between these two targets persists across all evaluated scenarios.

\begin{figure}[ht!]
    \centering
    \newcommand{\figw}{0.32\textwidth}

    \begin{subfigure}{\figw}
        \includegraphics[width=\linewidth]{figures/metric_diff/triviaqa-validation_Olmo-3-7B-Instruct_comparison.pdf}
    \end{subfigure}\hfill
    \begin{subfigure}{\figw}
        \includegraphics[width=\linewidth]{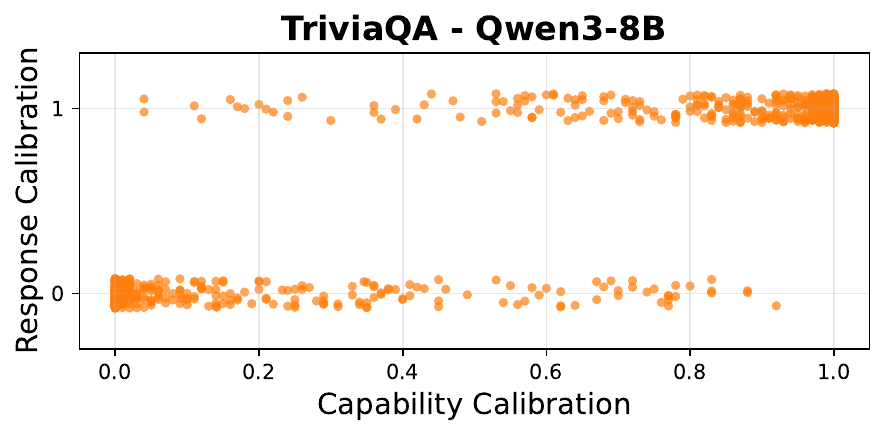}
    \end{subfigure}\hfill
    \begin{subfigure}{\figw}
        \includegraphics[width=\linewidth]{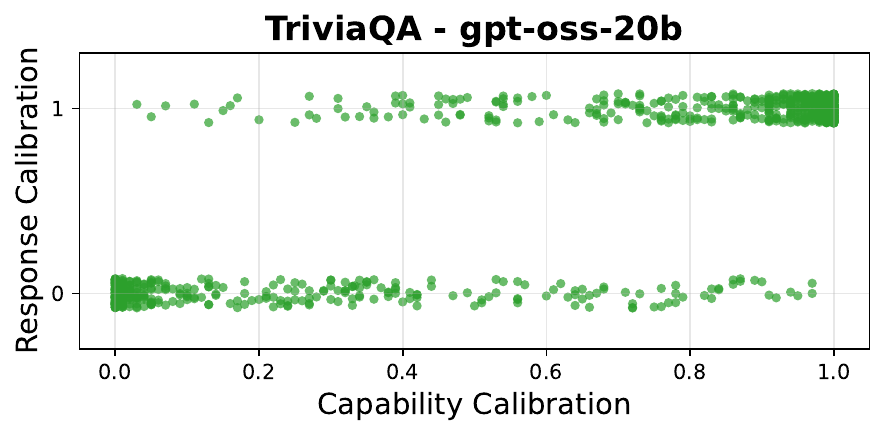}
    \end{subfigure}
    
    \vspace{0.2cm} 

    \begin{subfigure}{\figw}
        \includegraphics[width=\linewidth]{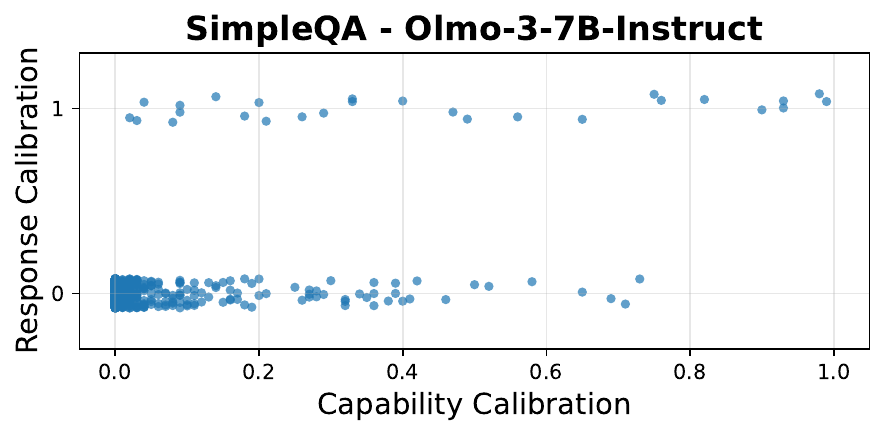}
    \end{subfigure}\hfill
    \begin{subfigure}{\figw}
        \includegraphics[width=\linewidth]{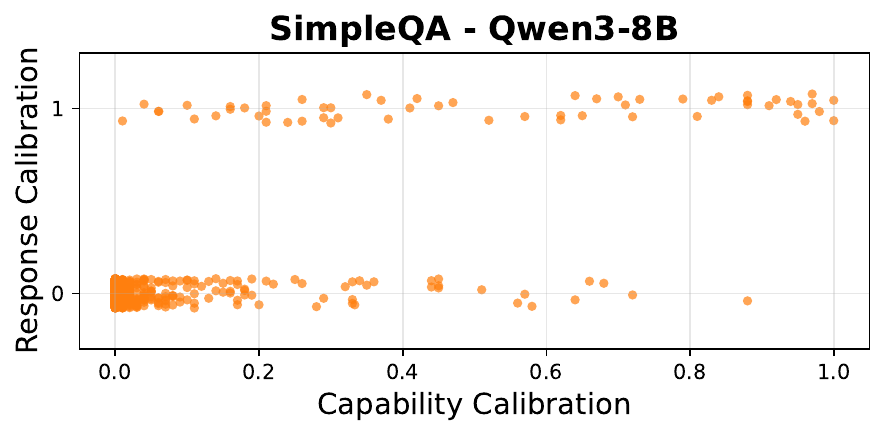}
    \end{subfigure}\hfill
    \begin{subfigure}{\figw}
        \includegraphics[width=\linewidth]{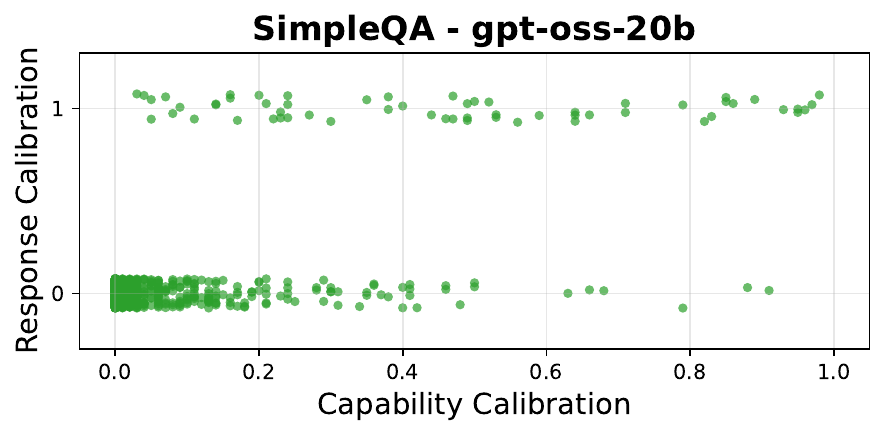}
    \end{subfigure}

    \vspace{0.2cm}

    \begin{subfigure}{\figw}
        \includegraphics[width=\linewidth]{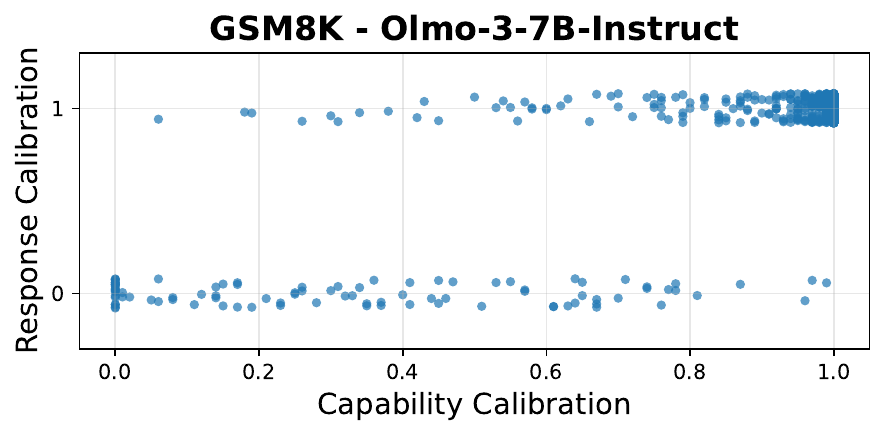}
    \end{subfigure}\hfill
    \begin{subfigure}{\figw}
        \includegraphics[width=\linewidth]{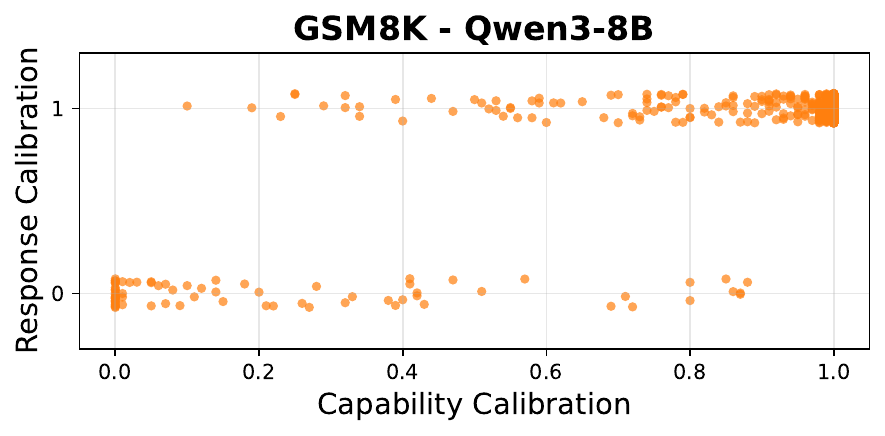}
    \end{subfigure}\hfill
    \begin{subfigure}{\figw}
        \includegraphics[width=\linewidth]{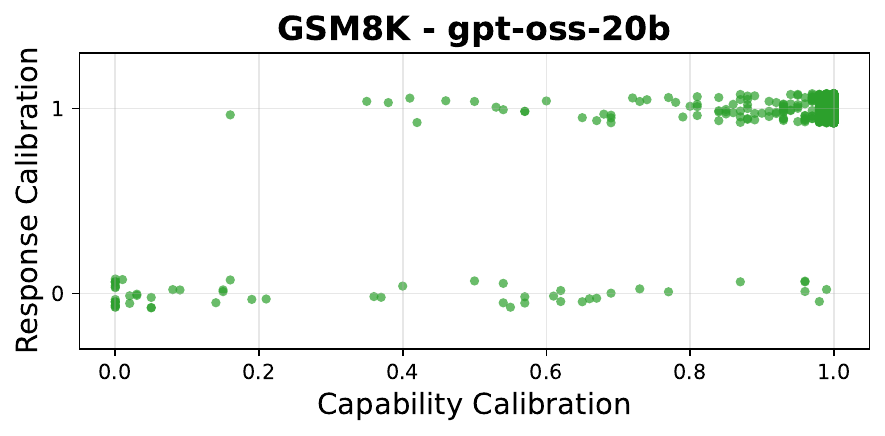}
    \end{subfigure}

    \vspace{0.2cm}

    \begin{subfigure}{\figw}
        \includegraphics[width=\linewidth]{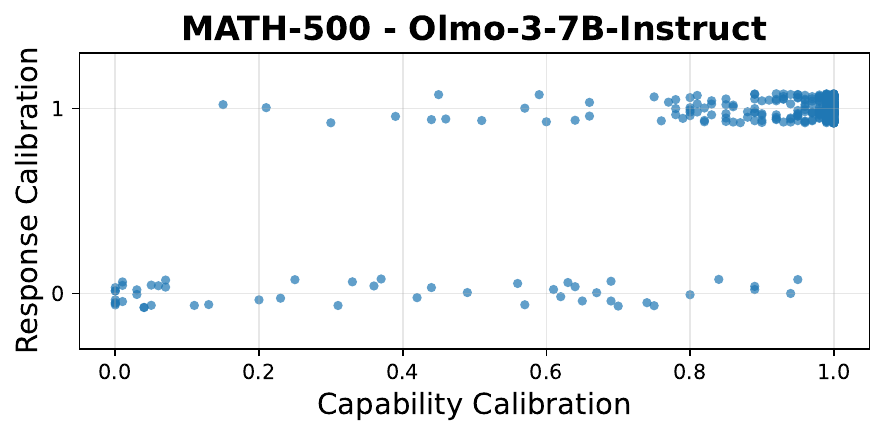}
    \end{subfigure}\hfill
    \begin{subfigure}{\figw}
        \includegraphics[width=\linewidth]{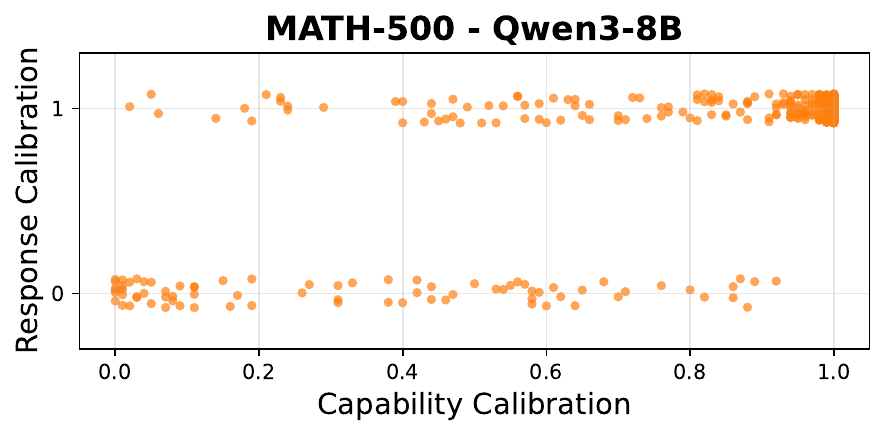}
    \end{subfigure}\hfill
    \begin{subfigure}{\figw}
        \includegraphics[width=\linewidth]{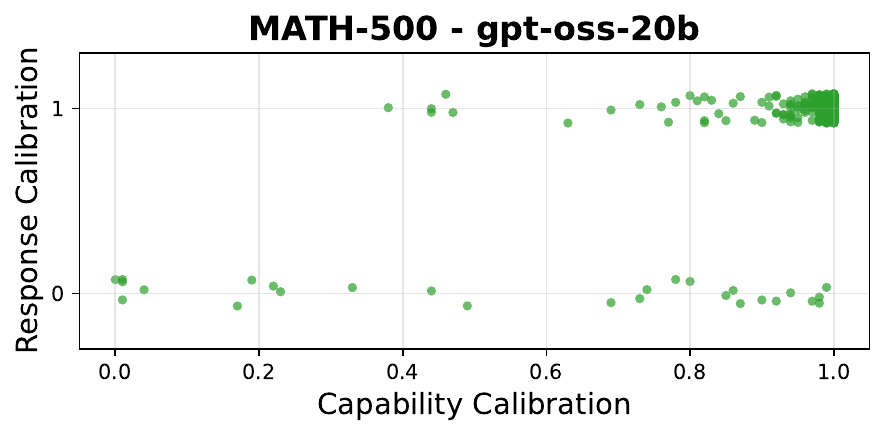}
    \end{subfigure}

    \vspace{0.2cm}

    \begin{subfigure}{\figw}
        \includegraphics[width=\linewidth]{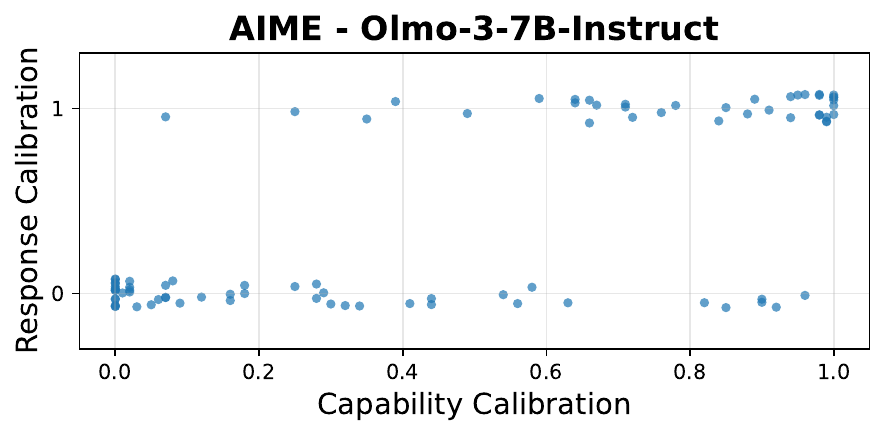}
    \end{subfigure}\hfill
    \begin{subfigure}{\figw}
        \includegraphics[width=\linewidth]{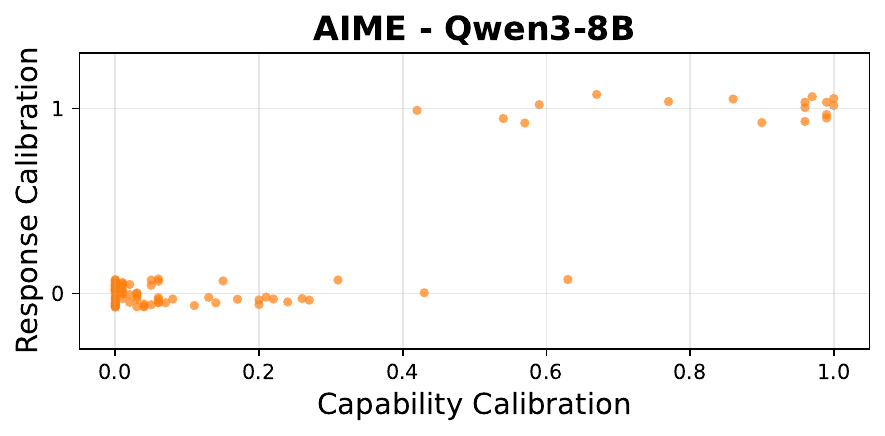}
    \end{subfigure}\hfill
    \begin{subfigure}{\figw}
        \includegraphics[width=\linewidth]{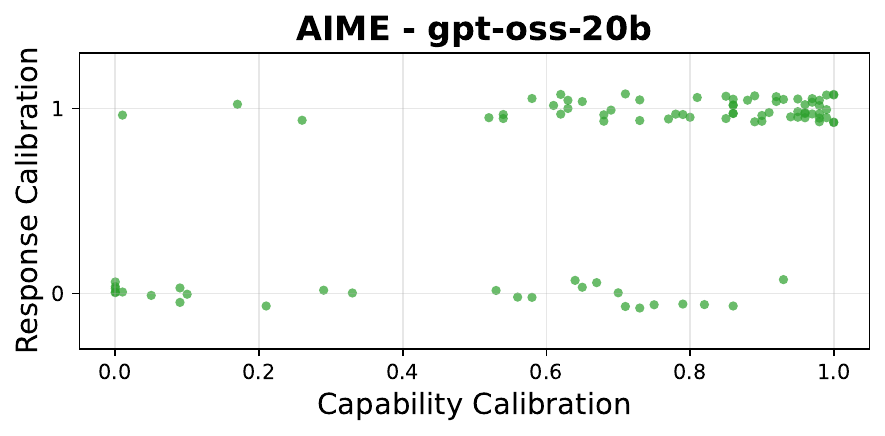}
    \end{subfigure}

        \vspace{0.2cm}

    \begin{subfigure}{\figw}
        \includegraphics[width=\linewidth]{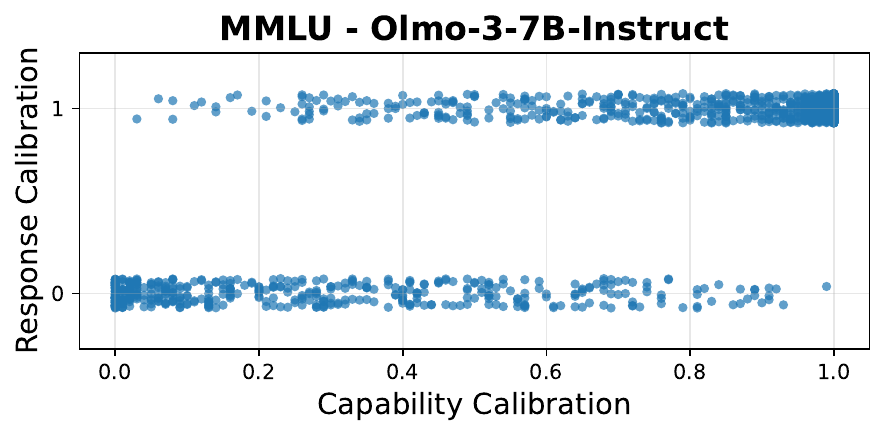}
    \end{subfigure}\hfill
    \begin{subfigure}{\figw}
        \includegraphics[width=\linewidth]{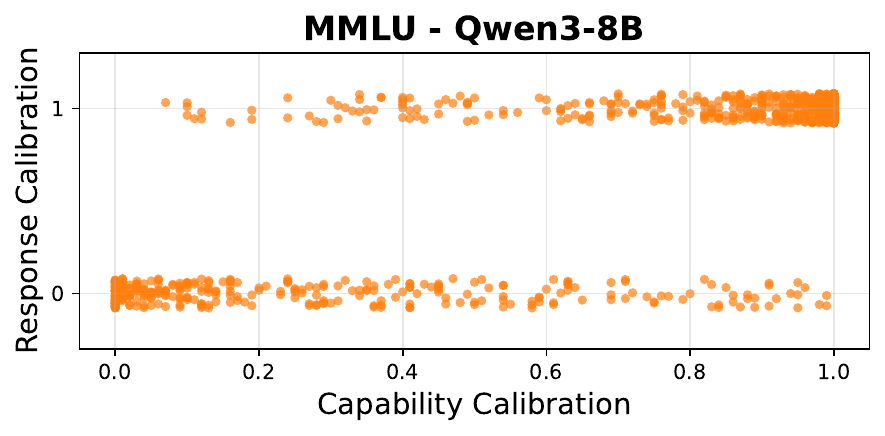}
    \end{subfigure}\hfill
    \begin{subfigure}{\figw}
        \includegraphics[width=\linewidth]{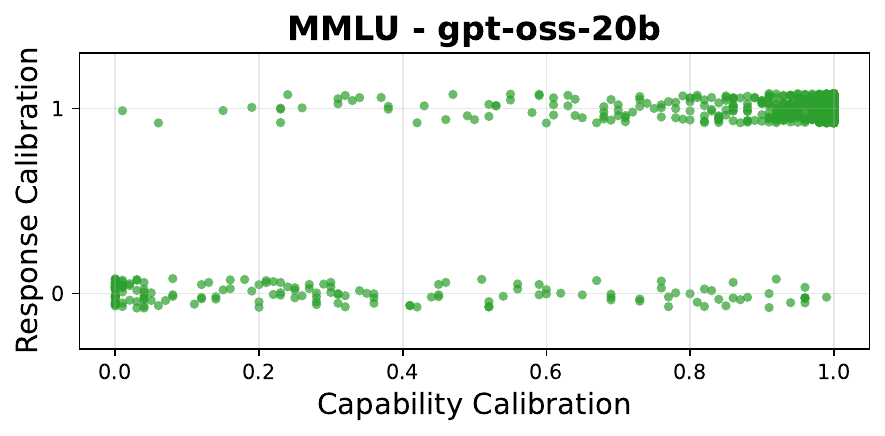}
    \end{subfigure}

    \vspace{0.2cm}

    \begin{subfigure}{\figw}
        \includegraphics[width=\linewidth]{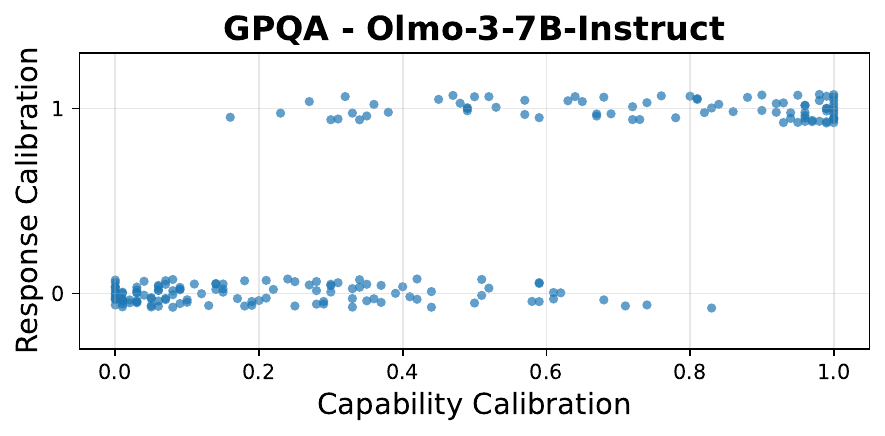}
    \end{subfigure}\hfill
    \begin{subfigure}{\figw}
        \includegraphics[width=\linewidth]{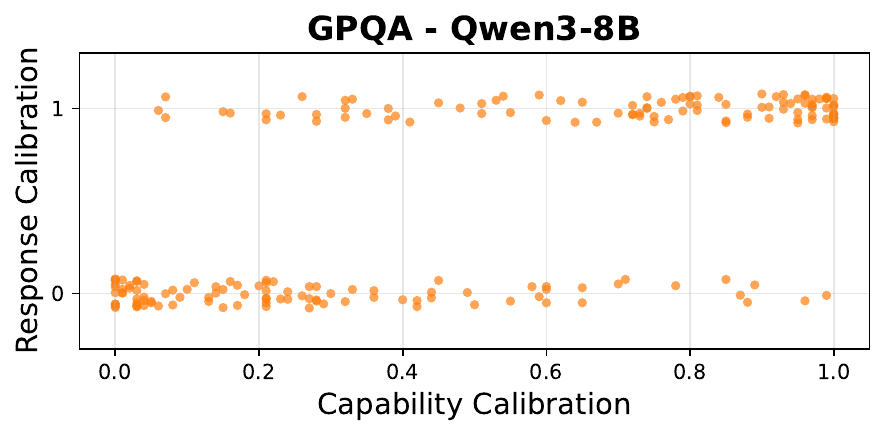}
    \end{subfigure}\hfill
    \begin{subfigure}{\figw}
        \includegraphics[width=\linewidth]{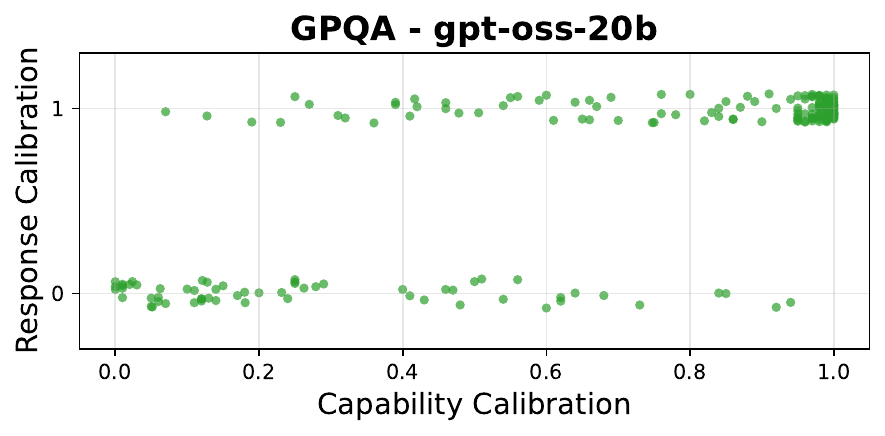}
    \end{subfigure}
    
    \caption{\textbf{Divergence of calibration targets across three models and seven datasets.} The data show that the Response Calibration (RC) target $\mathcal{C}(x, \hat{y})$ and Capability Calibration (CC) targets $\mathbb{E}_{\hat{y} \sim f_\theta(\cdot \mid x)}[\mathcal{C}(x,\hat{y})]$ differs at each model-dataset pair.}
    \label{fig:full_metric_diff}
\end{figure}

\end{document}